%% file: neurips_2026.tex
\documentclass{article}

\usepackage[utf8]{inputenc} 
\usepackage[T1]{fontenc} 
\usepackage{hyperref} 
\usepackage{url} 
\usepackage{booktabs} 
\usepackage{amsfonts} 
\usepackage{nicefrac} 
\usepackage{microtype} 
\usepackage{xcolor} 

\usepackage{graphicx}
\usepackage{subcaption}
\usepackage{booktabs}
\usepackage{hyperref}

\usepackage[main, final]{neurips_2026} 
\usepackage{amsmath}
\usepackage{bm}
\usepackage{amssymb}
\usepackage{amsfonts}
\usepackage{mathtools}
\usepackage{enumerate}
\usepackage{amsthm}
\usepackage{multirow}
\usepackage{makecell}
\usepackage{algorithm}

\usepackage{algpseudocode}
\usepackage[capitalize, noabbrev]{cleveref}
\usepackage[shortlabels]{enumitem}
\usepackage{wrapfig}

\usepackage{bm}
\newcommand{\best}[1]{\bm{#1}}

\usepackage[normalem]{ulem}
\newif\ifshowchanges \showchangestrue
\DeclareRobustCommand{\add}[1]{\ifshowchanges{\color{red}#1}\else#1\fi}

\newenvironment{addblock}{\ifshowchanges\color{red}\fi}{}
\showchangesfalse

\theoremstyle{plain}
\newtheorem{theorem}{Theorem}[section]

\theoremstyle{definition}

\theoremstyle{remark}
\newtheorem{remark}[theorem]{Remark}

\title{\texorpdfstring{S$^{2}$-PINN}{S2-PINN}: Stochastic Separable Physics-Informed Neural Networks}

\author{%
Zhendong Li \\
Industrial and Systems Engineering Dept.\\
Lehigh University\\
Bethlehem, PA 18015, USA \\
\texttt{zhl923@lehigh.edu} \\
\And
Akwum Onwunta\\
Industrial and Systems Engineering Dept.\\
Lehigh University\\
Bethlehem, PA 18015, USA \\
\texttt{ako221@lehigh.edu} \\
}

\begin{document}
	\maketitle

	\begin{abstract}
		Uncertainty quantification (UQ) for random partial differential equations (PDEs) is ubiquitous in computational science and engineering. However, classical spectral solvers for this class of problems face the curse of dimensionality, and existing neural solvers often ignore the stochastic structure that makes moments and calibration tractable. We introduce a stochastic separable physics-informed neural network, dubbed S$^{2}$-PINN, that represents the solution $u(t,\mathbf{x},\mathbf{Z})$ of a random PDE  with a learnable Gaussian spatial dictionary, Fourier temporal features, and a generalized polynomial chaos (gPC) stochastic basis, coupled by a low-rank Canonical Polyadic (CP) tensor decomposition core. The method is trained with a hybrid strong-form and gPC-projected residual loss. Our theoretical analysis establishes that the separable class is dense in $L^2$ under mild conditions, and the projected residual corresponds exactly to a stochastic Galerkin constraint. Furthermore, we show that mini-batch projection coefficients are logarithmically dependent on the number of gPC modes, and that the orthogonality penalty controls the conditioning of the learned spatial dictionary. Using four manufactured random PDE benchmarks, we show that S$^{2}$-PINN outperforms nine baselines in terms of mean and variance accuracy, as well as calibration, while using significantly fewer parameters. Further evaluations on non-manufactured Poisson and Darcy problems, a stochastic Navier--Stokes problem, a diffusion scaling study of higher random dimensions, and two stochastic inverse problems reveal the generalization capabilities of the proposed structure. Together, these results support stochastic separability as an effective design principle for physics-informed neural UQ. The code for the experiments can be found in \url{https://github.com/DMax1314/s2pinn}.
	\end{abstract}

	\section{Introduction}
	\label{sec:intro}
	Uncertainty quantification (UQ) for partial differential equations (PDEs) is essential when coefficients, boundary conditions, or forcing terms depend on random variables $\mathbf{Z}\in\mathbb{R}^{R}$. In this setting, the solution $u(t,\mathbf{x},\mathbf{Z})$ is a random field, and the primary targets are often its moments and calibration properties. Classical stochastic solvers such as Monte Carlo~\cite{ghanem1991stochastic}, stochastic Galerkin~\cite{xiu2002wiener}, and stochastic collocation~\cite{babuvska2007stochastic} are reliable but costly because they require many repeated solves or large coupled systems.
    

	Physics-informed neural networks (PINNs)~\cite{raissi2019physics,karniadakis2021physics} avoid labeled data by enforcing PDE residual constraints, but standard PINNs typically encode $\mathbf{Z}$ as generic inputs and rely on dense sampling to learn stochastic structure implicitly. 
    Operator-learning methods such as DeepONet~\cite{lu2021learning} and FNO~\cite{li2021fourier} learn parameter-to-solution maps from data, but require supervised datasets and degrade in stiff nonlinear regimes. 
    Recent stochastic PINN and physics-informed polynomial-chaos work~\cite{zhang2019quantifying,yang2021bpinns,psaros2023uncertainty,novak2024physics,shi2026expanding} has improved neural UQ, yet current architectures do not jointly exploit separability across space, time, and randomness.

	Separable PINNs (SPINN)~\cite{cho2024separable} showed that explicit
	separability can improve efficiency for deterministic PDEs. 
    In parallel, low-rank tensor methods (CP, TT, Tucker)~\cite{kolda2009tensor,oseledets2011tensor,bachmayr2023low} have long addressed high-dimensional structure in PDE and optimal control solvers~\cite{nouy2010proper,chinesta2010recent,bachmayr2016tensor, benner2015low, BENNER201626, benner2016block, Antil2023, Antil2025, Antil2025b}.
	The missing ingredient is an integrated stochastic PINN that combines learnable spatial representations, explicit stochastic basis structure, and physics constraints enforcing both pointwise residual consistency and stochastic projection consistency.

	S$^{2}$-PINN fills this gap. The model represents the solution $u(t,\mathbf{x},\mathbf{Z})$ by a three-way separable ansatz with a learnable Gaussian spatial dictionary, Fourier temporal features, and gPC stochastic modes, coupled through a low-rank CP core. 
    Training combines strong-form residual minimization with a gPC-projected residual term. 
    Our theory supports four concrete claims: finite-rank separable sums are expressive in $L^{2}$, the projected residual equals a stochastic Galerkin constraint, its mini-batch estimator has uniform concentration scaling as $O(\log G)$, where $G$ is the number of gPC modes, and the spatial orthogonality penalty controls feature conditioning. We evaluate both forward stochastic prediction and inverse stochastic parameter identification within the same physics-informed framework.
	\paragraph{Contributions.}
	{In this work, we make the following three main contributions.
	\begin{enumerate}[(a),nosep]
		\item 
    We introduce a stochastic separable PINN that couples a learnable spatial
dictionary with $M$ atoms, $M_t$ temporal Fourier features, and $G$ gPC
stochastic modes via a rank-$q$ CP factorization, reducing coupling complexity
from dense $O(MM_tG)$ to $O((M{+}M_t{+}G)q)$.  The hybrid strong-form and
gPC-projected residual loss enforces local PDE consistency and stochastic
Galerkin consistency.
		\item 
        We prove expressivity of the separable ansatz, projection--Galerkin equivalence, mini-batch concentration for stochastic projection coefficients, orthogonality-conditioning control, CP complexity, and gauge invariance. These results justify the representation and explain why projection is useful when the solution is well aligned with the chosen stochastic basis.
        \item 
        Among manufactured and non-manufactured benchmarks, S$^{2}$-PINN improves moment accuracy and calibration. Additional scaling and inverse-problem studies extend the evaluation beyond standard forward prediction, covering higher-dimensional random inputs and joint recovery of solution fields and random-field parameters.
	\end{enumerate}
	Our central claim is that stochastic separability provides a compact and
	well-aligned inductive bias for neural UQ of random PDEs, with its strongest
	benefits in regimes where the solution has low-rank space--time--random
	structure and a useful low-order stochastic expansion.

	\section{Related Work}
	\label{sec:related}

	\paragraph{Physics-Informed and Operator Learning for UQ.}
	Physics-informed neural networks typically enforce PDEs via soft residual constraints
	\cite{raissi2019physics}, with recent efforts targeting Uncertainty
	Quantification (UQ) via Bayesian frameworks \cite{yang2021bpinns}, auxiliary
	networks \cite{zhang2019quantifying}, and polynomial chaos expansions
	\cite{novak2024physics,shi2026expanding}. Parallel to PINNs, neural
	operators like DeepONet \cite{lu2021learning} and FNO \cite{li2021fourier}
	learn parameter-to-solution maps entirely from data. Very recently, Neural
	Chaos \cite{bahmani2025neural} proposed replacing fixed polynomial bases
	with sequentially learned networks to break the independence assumption. However,
	a common limitation spans these methods: they either inherently mandate extensive
	labeled training data (data-driven operators), fail to exploit
	multidimensional structural separation, or rely on dense representations that
	scale poorly with stochastic dimensions. S$^{2}$-PINN addresses this by formulating
	a strictly physics-informed, uncoupled architecture that mathematically
	isolates spatial derivative costs from the stochastic basis size.
{Physics-informed polynomial chaos expansions (PC$^{2}$)~\cite{novak2024physics}
are particularly relevant to our setting because they impose PDE constraints
directly on a PCE surrogate via KKT-style constrained regression. We therefore
include a UQpy-based PC$^{2}$ comparison as an external polynomial-chaos baseline:
it tests whether a fixed global polynomial surrogate with physics constraints
can match the accuracy and calibration gains obtained from our learnable
space--time representation and low-rank stochastic coupling.}

	\paragraph{Low-Rank Tensor Methods in Neural Solvers.}
	Classical low-rank tensor decompositions (CP, TT, Tucker) have long accelerated
	stochastic Galerkin solvers \cite{doostan2013nonintrusive,bigoni2016spectral}.
	In the deep learning regime, recent state-of-the-art neural-tensor hybrids—such
	as TT-PINN \cite{liu2022ttpinn}, functional tensor decompositions
	\cite{vemuri2024functional}, and INN-TD \cite{guo2025interpolating}—target
	purely \emph{deterministic} spatial/parametric coordinates. They primarily
	compress neural weight matrices for edge deployment or factorize spatial
	axes without incorporating a statistical structural prior. Because none embed
	a stochastic polynomial basis or a UQ-specific Galerkin projection, they
	cannot address robust uncertainty quantification. S$^{2}$-PINN fundamentally
	diverges by fusing a learnable spatial dictionary with explicitly separated
	gPC stochastic modes, governed entirely by a hybrid strong-form and projection
	constraint.

	\paragraph{Separable PINNs for High-Dimensional PDEs.}
	SPINN \cite{cho2024separable} first demonstrated that explicit dimensional
	separability, coupled with a coordinate warp, drastically accelerates
	deterministic PDE training. While it has been extended to quantum and elastic
	systems \cite{eskin2024separable,zanotta2025quantum}, seamlessly scaling
	SPINN to stochastic parameters is highly non-trivial. Naively introducing random
	variables into SPINN yields a dense, non-factorized coefficient tensor that
	\add{grows multiplicatively with the spatial, temporal, and stochastic resolutions}. S$^{2}$-PINN
	\add{mitigates} this specific bottleneck by eschewing the dense multi-way block
	for a low-rank CP coupling kernel, tightly integrating optimal stochastic bases
	(gPC) with spatio-temporal atoms, a structural synthesis currently absent in the
	SPINN-related literature. \add{To isolate the effect of this design, we additionally compare against structurally closer baselines: SPINN extended with stochastic coordinates (SPINN+Z)~\cite{cho2024separable}, functional CP-PINNs~\cite{vemuri2024functional}, and Neural Chaos~\cite{bahmani2025neural}.}

	\section{Method}
	\label{sec:method}
    S$^{2}$-PINN represents random PDE solutions with an explicitly separated space--time--random structure. The method assumes that the uncertainty is encoded by a finite random vector $\mathbf{Z}$ with known marginals, which allows us to choose an orthonormal gPC basis in the stochastic coordinates. The architecture then couples spatial, temporal, and stochastic factors through a low-rank CP kernel, while the loss combines pointwise physics constraints with a stochastic projection constraint. Appendix~\ref{sec:theory} provides PDE-agnostic support for the main design choices: expressivity of separable sums, projection--Galerkin equivalence, mini-batch concentration, orthogonality penalty conditioning, CP complexity, and gauge invariance.

    \subsection{Three-way Separable Representation} 
    We approximate the solution by a rank-$q$ CP expansion, 
    \begin{equation} \label{eq:cp-rep} u(t,\mathbf{x}, \mathbf{Z}) \approx \sum_{r=1}^{q} w_{r} \phi^{(r)}(\mathbf{x}) \tau^{(r)}(t) \psi^{(r)}(\mathbf{Z}), 
    \end{equation} 
    where each factor is a linear combination of basis functions (atoms) specific to its coordinate,
    \begin{equation} 
    \phi^{(r)}(\mathbf{x}) = \sum_{i=1}^{M} A_{ir}\Phi_i(\mathbf{x}), \quad \tau^{(r)}(t) = \sum_{k=1}^{M_t} B_{kr}T_k(t), \quad \psi^{(r)}(\mathbf{Z}) = \sum_{g=1}^{G} D_{gr}\Psi_g(\mathbf{Z}). 
    \end{equation} 
    Here $A\in\mathbb{R}^{M\times q}$, $B\in\mathbb{R}^{M_t\times q}$, $D\in\mathbb{R}^{G\times q}$, and $w\in\mathbb{R}^{q}$ are learnable CP parameters. Spatial atoms are Gaussian dictionary elements with an optional residual warp,

	The three feature families are as follows. Spatial atoms are Gaussian dictionary elements with an optional residual warp,
	\begin{equation}
		\label{eq:phi-dict}\Phi_{i}(\mathbf{x})=\exp\!\Big(-\tfrac{\|\mathbf{W}(\mathbf{x})-\mathbf{c}_i\|_2^2}
		{2\sigma_i^2}\Big),\qquad \mathbf{W}(\mathbf{x})=\mathbf{x}+\mathbf{w}_{\theta}
		(\mathbf{x}).
	\end{equation}
	Temporal atoms are DC and sine/cosine features with learnable frequencies
	$\omega_{k}=\exp(\theta_{k})$, and stochastic atoms are tensorized gPC bases
	(chosen by the base distribution) on $\mathcal{J}_{R,p}$.
	Given feature blocks $\mathbf{F}_{\text{space}}\in\mathbb{R}^{B\times M}$, $\mathbf{F}
	_{\text{time}}\in\mathbb{R}^{B\times M_t}$, and
	$\mathbf{F}_{\text{stoc}}\in\mathbb{R}^{B\times G}$, the batched forward map
	is
	\begin{equation}
		\label{eq:forward-cp}u = \big((\mathbf{F}_{\text{space}}A)\odot(\mathbf{F}
		_{\text{time}}B)\odot(\mathbf{F}_{\text{stoc}}D)\big)w,
	\end{equation}
	where $\odot$ denotes a row-wise Hadamard product.
	Arithmetic scales as $O(Bq(M+M_{t}+G))$ instead of $O(BMM_{t}G)$ {(\Cref{prop:cp_complexity})}. Relative to
	SPINN, this design adds explicit stochastic modes, learns a Gaussian spatial dictionary, and uses a three-way CP coupling that keeps the stochastic basis separated from spatial derivative costs. Figure~\ref{fig:architecture} illustrates the architecture.

	{\begin{theorem}
    Under the density assumptions on $\{\Phi_i\}$, $\{T_k\}$, and $\{\Psi_g\}$ stated in Appendix~\ref{sec:theory}, the union $\bigcup_{q\ge 1}\mathcal{F}_q$ of rank-$q$ separable functions is dense in $L^2(\mathcal{T}\times\Omega\times\mathbf{Z})$; see Theorem~\ref{prop:density_agnostic}.     
	\end{theorem}}
	\begin{figure}[htbp]
		\centering		\includegraphics[width=\linewidth]{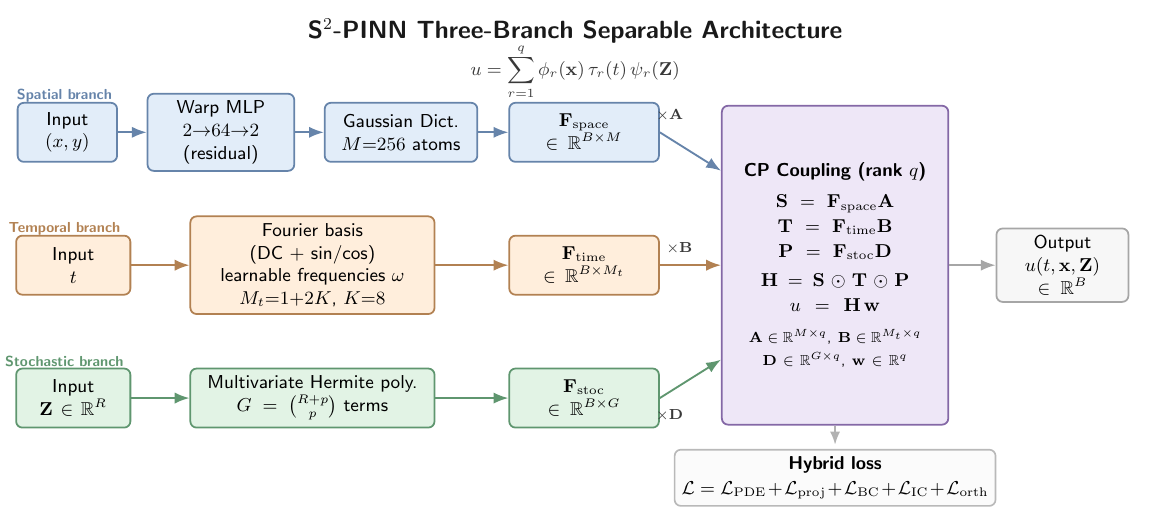}
		\caption{S$^{2}$-PINN architecture. Three specialized branches---spatial
		(Gaussian dictionary with warp), temporal (learnable Fourier),
		stochastic (Hermite gPC)---are coupled via a rank-$q$ CP core using
		Hadamard products. The hybrid loss combines PDE residual, gPC projection,
		boundary/initial conditions, and orthogonality regularization.}
		\label{fig:architecture}
	\end{figure}

	\subsection{Hybrid Physics Loss and Regularization}
	Let $r_\theta=\mathcal N[u_\theta;\vartheta](t,\mathbf x,\mathbf Z)$ be the strong-form residual, {where $\vartheta$ parameterizes the differential operator $\mathcal{N}$ and $u_\theta$ is the network output}. We minimize the objective
	\begin{equation}
		\label{eq:loss}
        \resizebox{\columnwidth}{!}{$
			\mathcal{L}  = w_{\mathrm{pde}}\mathbb{E}[r_\theta^{2}] + w_{\mathrm{proj}}\sum_{g=1}^{G}\Big(\mathbb{E}[r\Psi_{g}(\mathbf{Z})]\Big)^{2}+ w_{\mathrm{bc}}\mathbb{E}_{\partial\Omega}[(u-{u_{\partial}}
            )^{2}] + w_{\mathrm{ic}}\mathbb{E}[(u-{u_{0}}
            )^{2}|_{t=0}] + w_{\mathrm{orth}}\mathcal{R}_{\mathrm{orth}}.$}
	\end{equation}
    {Here $u_{\partial}$ and $u_{0}$ denote boundary and
	initial conditions. The orthogonal regularizer combines a spatial feature Gram penalty} 
    	\begin{equation}
		\label{eq:gram}\mathcal{R}_{\text{feat}}=\|\widehat{G}_{\Phi}-I\|_{F}^{2}
		,\qquad \widehat{G}_{\Phi}=\tfrac{1}{B}\mathbf{F}_{\text{space}}^{\top}\mathbf{F}
		_{\text{space}},
	\end{equation}
	and CP-factor orthogonality,
	\begin{equation}
		\label{eq:fac-orth}\mathcal{R}_{\text{fac}}=\sum_{M\in\{A,B,D\}}\Big\|\tfrac
		{1}{\|M\|_{\text{col}}^2}M^{\top}M-I\Big\|_{F}^{2},
	\end{equation}
	with $\mathcal{R}_{\mathrm{orth}}=\mathcal{R}_{\mathrm{feat}}+\lambda_{\mathrm{fac}}\mathcal{R}_{\mathrm{fac}}$.
	\add{Note that the CP format is defined by the forward map~\eqref{eq:forward-cp} itself, which is an explicit sum of rank-one products of spatial, temporal, and stochastic factors for any value of $\lambda_{\mathrm{fac}}$, including zero; the factor penalty (default $\lambda_{\mathrm{fac}}{=}0.1$) only encourages well-conditioned factors and neither defines the CP form nor enforces CP uniqueness.}
    \Cref{prop:galerkin_agnostic} shows that near-identity spatial Gram matrices control the condition number of the feature matrix.
    Regularization incorporates a spatial Gram penalty, which encourages feature near-orthogonality. Specifically, \Cref{prop:cond_agnostic} shows that $\|\widehat{G}_\Phi - I\|_2 \le \delta$ ensures $\kappa(F_{\text{space}}) \le \sqrt{(1+\delta)/(1-\delta)}$, with \Cref{prop:gram_conc} providing the corresponding sample complexity.

	The projected residual term is theoretically Galerkin-consistent and statistically stable, but it is not uniformly beneficial in every stochastic regime. \Cref{prop:galerkin_agnostic} identifies $\sum_g(\mathbb{E}[r_\theta\Psi_g])^2$ with $\|\mathsf{P}_{V_p}r_\theta\|_{L^2(\mathbf{Z})}^2$, the squared residual projection onto the gPC subspace. The mini-batch coefficient estimator is unbiased with $O(1/B)$ variance (\Cref{prop:mc_agnostic}), and \Cref{prop:conc_agnostic} gives uniform $\epsilon$-accuracy across all $G$ modes with batch size scaling as $O(\log G)$. Empirically, this projection acts as a useful spectral regularizer when the solution is well approximated by the chosen polynomial space, but it can introduce bias when important stochastic energy lies outside that space. 
	\add{In practice, the projection is computed pointwise in space and time: for each sampled $(t,\mathbf{x})$ we draw several stochastic samples $\{\mathbf{Z}_b\}$, estimate $\mathbb{E}_{\mathbf{Z}}[r_\theta(t,\mathbf{x},\mathbf{Z})\Psi_g(\mathbf{Z})]$ from them, and then average the squared coefficients over the space--time collocation points, rather than taking one joint average over random $(t,\mathbf{x},\mathbf{Z})$ tuples.}

    }

	\subsection{Training}
    At each step we sample PDE, boundary, initial, and Gram mini-batches and optimize Eq.~\eqref{eq:loss} with Adam. 
    Time derivatives are analytic from Fourier features, while spatial derivatives, including $\Delta u$, use autodiff through the warped Gaussian dictionary. 
    We periodically resample collocation points and apply column-wise normalization to CP factors $(A,B,D)$ after updates to fix the continuous gauge symmetry of the CP format (\Cref{prop:cp_gauge}), which also improves the Lipschitz conditioning of the forward map.
    The column normalization leaves the forward map unchanged (\Cref{prop:cp_gauge}) but rescales Adam's moment buffers; at the default learning rate $\eta=10^{-3}$, we observed no instability from this engineering choice. Algorithm~\ref{alg:training} and default hyperparameters are reported in Appendix~\ref{sec:model-details}.
%


	\section{Experiments}
	\label{sec:num}
	\subsection{Benchmark Problems}
	We evaluate S$^{2}$-PINN on four controlled random PDEs: diffusion, Allen--Cahn, Burgers, and Darcy. The first three are time-dependent problems that share the same manufactured target and differ in the physical operator, so their performance gaps isolate the effect of linear diffusion, cubic reaction, and nonlinear advection. Darcy is steady-state and uses exponential $\mathbf{Z}$-dependence in the coefficient, making it a deliberately harder case for a low-order polynomial stochastic basis. 

	Unless otherwise stated, all experiments use $R{=}5$ random variables, total gPC degree $p{=}3$ ($G{=}\binom{8}{3}{=}56$ modes), $M{=}256$ spatial atoms, $M_t{=}17$ temporal atoms, and CP rank $q{=}16$. We keep this single default across the four PDEs, the scaling sweep (\Cref{sec:scaling-sweep}), and the collocation-density sweep (\Cref{sec:colloc-sweep}) to avoid per-problem tuning. Manufactured solutions enable ground-truth evaluation of both mean and variance; full formulas are given in \Cref{sec:manufactured-solutions}.

	\paragraph{Baselines and fairness.}
	We compare against \add{nine} baselines: a standard MLP PINN, the same PINN with the gPC projection loss, PI-DeepONet~\cite{goswami2023physics}, MC-Dropout PINN, a $K{=}5$ deep ensemble~\cite{DBLP:conf/nips/Lakshminarayanan17}, stochastic collocation (SC) with $n_q{=}2$ Gauss--Hermite nodes per dimension ($2^5{=}32$ node-PINNs)\add{, and three structurally closer methods: SPINN with the random variables appended as additional separable coordinates (SPINN+Z)~\cite{cho2024separable}, functional CP-PINN~\cite{vemuri2024functional}, and Neural Chaos~\cite{bahmani2025neural}. Neural Chaos is supervised by reference solutions and therefore serves as a data-driven oracle rather than a purely physics-informed competitor}. All neural methods use the same optimizer family, collocation strategy, and training budget unless stated otherwise. The vanilla PINN has about $28\times$ more parameters than S$^{2}$-PINN (roughly 200K versus 7K), so parameter count does not favor the proposed architecture. 
    Additional architecture and optimization details are in Appendix~\ref{sec:baseline-details}.

	\subsection{Metrics}
	We report two moment-accuracy metrics, evaluated on a held-out stochastic
	test set:
	\begin{itemize}[nosep]
		\item \textbf{Mean relative $L^{2}$} :
			$\|\mathbb{E}_{\mathbf{Z}}[u] - \mathbb{E}_{\mathbf{Z}}[u^{*}]\|_{L^2(\Omega)}
			/ \|\mathbb{E}_{\mathbf{Z}}[u^{*}]\|_{L^2(\Omega)}$
		\item \textbf{Variance relative $L^{2}$} :
			$\|\mathrm{Var}_{\mathbf{Z}}[u] - \mathrm{Var}_{\mathbf{Z}}[u^{*}
			]\|_{L^2(\Omega)}/ \|\mathrm{Var}_{\mathbf{Z}}[u^{*}]\|_{L^2(\Omega)}$
	\end{itemize}
	Moments are estimated by Monte Carlo over 1024 i.i.d.\ samples
	$\mathbf{Z}\sim \mathcal{N}(0,I_{R})$ at $B{=}2048$ uniformly sampled probe
	points $(t,\mathbf{x})\in\mathcal{T}\times\Omega$. All results are averaged over
	5 random seeds (42, 123, 456, 789, 1024), and all results reported as mean $\pm$ std.
	{
	The lowest mean per column is bolded and when the runner-up lies within
	$\pm 1\sigma$ of the best, both entries are bolded to indicate a
	statistical tie.}

\paragraph{Calibration protocol.}
To assess whether the predicted uncertainty faithfully reflects the true variability of the solution under random inputs, we construct predictive intervals and evaluate their empirical coverage. At each probe $(t_{j},\mathbf{x}_{j})$, we draw $N_{Z}{=}1{,}024$ i.i.d.\ samples $\{\mathbf{Z}^{(i)}\}_{i=1}^{N_Z}$ from $\mathcal{N}(0,I_R)$ and form the nominal $(1{-}\alpha)$ predictive interval $[\hat{Q}_{\alpha/2},\,\hat{Q}_{1-\alpha/2}]$, where $\hat{Q}_{\tau}$ denotes the $\tau$-th empirical quantile of $\{u_{\theta}(t_{j},\mathbf{x}_{j},\mathbf{Z}^{(i)})\}$. Coverage is then evaluated on a {separate} hold-out set of $N_{Z}'{=}1{,}024$ samples $\{\mathbf{Z}'^{(k)}\}_{k=1}^{N_Z'}$: empirical coverage at nominal level $(1{-}\alpha)$ is the fraction of $(j,k)$ pairs for which $u^{*}(t_{j},\mathbf{x}_{j},\mathbf{Z}'^{(k)}) \in [\hat{Q}_{\alpha/2}^{(j)},\,\hat{Q}_{1-\alpha/2}^{(j)}]$. We report empirical coverage at four nominal levels ($50\%$, $80\%$, $90\%$, $95\%$), sharpness (mean interval width at $90\%$), and mean calibration error $\frac{1}{4}\sum_{\ell}|(1{-}\alpha_{\ell}) - \widehat{\mathrm{Cov}}_{\ell}|$. Results are in \Cref{tab:calibration}.

\subsection{Main Results}
Tables~\ref{tab:mean-accuracy}--\ref{tab:var-accuracy} summarize the comparison. These estimates provide tighter uncertainty quantification of performance trends on these four controlled benchmarks. 
	\begin{table}[htbp]
		\centering
		\caption{Mean solution accuracy.}
		\label{tab:mean-accuracy} \small \resizebox{\columnwidth}{!}{%
		\begin{tabular}{lcccc}
			\hline
			Method                       & Diffusion                                & Allen--Cahn                              & Burgers                                  & Darcy                                    \\
			\hline
			\textbf{S$^{2}$-PINN (proj)} & $\mathbf{(3.71 \pm 1.01)\times 10^{-3}}$ & \add{$\mathbf{(4.40 \pm 1.32)\times 10^{-3}}$} & $\mathbf{(3.44 \pm 1.47)\times 10^{-3}}$ & {$\mathbf{(1.05 \pm 0.30)\times 10^{-3}}$}          \\
			S$^{2}$-PINN (no-proj)       & $(4.85 \pm 2.80)\times 10^{-3}$          & {$(4.51 \pm 4.15)\times 10^{-3}$}& {$\mathbf{(3.86 \pm 1.84)\times 10^{-3}}$}& $\mathbf{(9.51 \pm 2.13)\times 10^{-4}}$ \\
			Vanilla PINN                 & $(1.41 \pm 0.42)\times 10^{-2}$          & $(1.81 \pm 0.68)\times 10^{-2}$          & $(1.78 \pm 0.72)\times 10^{-2}$          & $(1.09 \pm 0.28)\times 10^{-2}$          \\
			Vanilla PINN + proj          & $(1.82 \pm 0.69)\times 10^{-2}$          & $(1.49 \pm 0.40)\times 10^{-2}$          & $(2.13 \pm 0.75)\times 10^{-2}$          & $(1.57 \pm 0.60)\times 10^{-2}$          \\
			PI-DeepONet                  & \add{$(5.07 \pm 2.17)\times 10^{-2}$} & \add{$(5.74 \pm 2.28)\times 10^{-2}$} & \add{$(4.32 \pm 1.96)\times 10^{-2}$} & \add{$(2.56 \pm 0.82)\times 10^{-2}$} \\
			MC-Dropout PINN              & $(1.08 \pm 0.11)\times 10^{-1}$          & $(9.99 \pm 0.77)\times 10^{-2}$          & $(1.19 \pm 0.10)\times 10^{-1}$          & $(7.45 \pm 0.47)\times 10^{-2}$          \\
			Deep Ensemble ($K{=}5$)      & $(1.31\pm0.13)\times10^{-2}$& $(1.84\pm0.44)\times10^{-2}$ &$(1.01\pm0.33)\times10^{-2}$ &$(2.05\pm0.45)\times10^{-2}$ \\
			SC ($n_{q}{=}2$, 32 nodes)   &$(3.52\pm0.11)\times10^{-2}$ & $(3.45\pm0.33)\times10^{-2}$& $(3.57\pm0.10)\times10^{-2}$& $(8.74\pm0.27)\times10^{-2}$ \\
			\add{SPINN+Z} & \add{$(1.57\pm5.10)\times10^{-2}$} & \add{$(6.70\pm3.42)\times10^{-2}$} & \add{$(3.59\pm3.43)\times10^{-2}$} & \add{$(3.48\pm2.18)\times10^{-2}$} \\
			\add{Functional CP-PINN} & \add{$(9.37\pm4.22)\times10^{-2}$} & \add{$(10.87\pm4.86)\times10^{-2}$} & \add{$(10.87\pm4.86)\times10^{-2}$} & \add{$(4.53\pm1.74)\times10^{-2}$} \\
			\add{Neural Chaos} & \add{$(2.31\pm1.35)\times10^{-2}$} & \add{--} & \add{--} & \add{$(2.68\pm1.53)\times10^{-2}$} \\
			\hline
		\end{tabular}%
		}
	\end{table}

	\begin{table}[htbp]
		\centering
		\caption{Variance solution accuracy.}
		\label{tab:var-accuracy} \small \resizebox{\columnwidth}{!}{%
		\begin{tabular}{lcccc}
			\hline
			Method& Diffusion& Allen--Cahn& Burgers& Darcy\\
			\hline
			\textbf{S$^{2}$-PINN (proj)} & $\mathbf{(4.40 \pm 1.52)\times 10^{-3}}$ & $\mathbf{(4.07 \pm 1.98)\times 10^{-3}}$ & $\mathbf{(4.65 \pm 2.00)\times 10^{-3}}$ & $(5.41 \pm 0.86)\times 10^{-3}$          \\
			S$^{2}$-PINN (no-proj)       & $(7.03 \pm 2.41)\times 10^{-3}$          & $(7.32 \pm 7.19)\times 10^{-3}$          & {$\mathbf{(5.39 \pm 3.03)\times 10^{-3}}$}& $\mathbf{(4.57 \pm 0.51)\times 10^{-3}}$ \\
			Vanilla PINN                 & $(3.75 \pm 0.80)\times 10^{-2}$          & $(3.79 \pm 1.08)\times 10^{-2}$          & $(5.12 \pm 1.81)\times 10^{-2}$          & $(1.92 \pm 0.18)\times 10^{-2}$          \\
			Vanilla PINN + proj          & $(4.75 \pm 1.30)\times 10^{-2}$          & $(4.24 \pm 0.67)\times 10^{-2}$          & $(4.86 \pm 0.71)\times 10^{-2}$          & $(2.06 \pm 0.23)\times 10^{-2}$          \\
			PI-DeepONet& \add{$(6.05 \pm 1.76)\times 10^{-2}$} & \add{$(6.58 \pm 2.25)\times 10^{-2}$} & \add{$(5.31 \pm 3.75)\times 10^{-2}$} & \add{$(7.87 \pm 1.98)\times 10^{-2}$} \\
			MC-Dropout PINN& $(1.54 \pm 0.25)\times 10^{-1}$          & $(1.16 \pm 0.15)\times 10^{-1}$          & $(1.63 \pm 0.22)\times 10^{-1}$          & $(2.10 \pm 0.07)\times 10^{-1}$          \\
			Deep Ensemble ($K{=}5$)&$(4.08 \pm 2.13)\times 10^{-2}$ & $(3.64 \pm 3.21)\times 10^{-2}$& $(4.22 \pm 1.89)\times 10^{-2}$& $(5.96 \pm 2.47)\times 10^{-2}$\\
			SC ($n_{q}{=}2$, 32 nodes)& $(7.80 \pm 3.41)\times 10^{-2}$& $(3.45\pm 2.27)\times 10^{-2}$& $(8.61 \pm 3.54)\times 10^{-2}$& $(7.13 \pm 3.97)\times 10^{-2}$\\
			\add{SPINN+Z} & \add{$(1.28\pm8.65)\times10^{-2}$} & \add{$(1.40\pm4.56)\times10^{-2}$} & \add{$(7.20\pm6.59)\times10^{-3}$} & \add{$(2.95\pm1.05)\times10^{-2}$} \\
			\add{Functional CP-PINN} & \add{$(9.88\pm3.55)\times10^{-2}$} & \add{$(9.62\pm5.83)\times10^{-2}$} & \add{$(1.01\pm2.39)\times10^{-1}$} & \add{$(7.21\pm5.18)\times10^{-1}$} \\
			\add{Neural Chaos} & \add{$(2.72\pm1.88)\times10^{-1}$} & \add{--} & \add{--} & \add{$(6.56\pm8.21)\times10^{-2}$} \\
			\hline
		\end{tabular}%
		}
	\end{table}
	\begin{table}[htbp]
		\centering
		\caption{UQ calibration: empirical coverage at nominal levels and calibration
		error. \add{Calibration results for PI-DeepONet, SPINN+Z, functional CP-PINN, and Neural Chaos are in \Cref{tab:calibration-extra}.}}
		\label{tab:calibration} \scriptsize \resizebox{\columnwidth}{!}{%
		\begin{tabular}{lcccccc}
			\hline
			Method / PDE                                    & Cov@50\%        & Cov@80\%        & Cov@90\%        & Cov@95\%        & Sharpness$_{90}$ & Cal.~Err.       \\
			\hline
\multicolumn{7}{l}{\textbf{S$^{2}$-PINN (proj)}} \\
\quad Diffusion & $\best{0.491\pm0.012}$ & $\best{0.797\pm0.007}$ & $\best{0.897\pm0.004}$ & $\best{0.950\pm0.003}$ & $\best{0.194\pm0.006}$ & $\best{0.006\pm0.003}$ \\
\quad Allen--Cahn & $\best{0.493\pm0.013}$ & $\best{0.800\pm0.008}$ & $\best{0.900\pm0.005}$ & $\best{0.952\pm0.004}$ & $\best{0.194\pm0.006}$ & $\best{0.006\pm0.001}$ \\
\quad Burgers & $\best{0.493\pm0.012}$ & $\best{0.799\pm0.007}$ & $\best{0.899\pm0.004}$ & $\best{0.951\pm0.002}$ & $\best{0.194\pm0.006}$ & $\best{0.006\pm0.002}$ \\
\quad Darcy & $\best{0.497\pm0.007}$ & $\best{0.800\pm0.006}$ & $\best{0.901\pm0.007}$ & $\best{0.952\pm0.005}$ & $\best{0.094\pm0.002}$ & $\best{0.005\pm0.002}$ \\
\multicolumn{7}{l}{\textbf{S$^{2}$-PINN (no proj)}} \\
\quad Diffusion & $\best{0.494\pm0.013}$ & $\best{0.796\pm0.015}$ & $\best{0.888\pm0.015}$ & $0.937\pm0.012$ & $\best{0.189\pm0.000}$ & $\best{0.014\pm0.008}$ \\
\quad Allen--Cahn & $\best{0.496\pm0.013}$ & $\best{0.798\pm0.015}$ & $\best{0.890\pm0.015}$ & $\best{0.939\pm0.012}$ & $\best{0.189\pm0.000}$ & $\best{0.013\pm0.008}$ \\
\quad Burgers & $\best{0.496\pm0.012}$ & $\best{0.798\pm0.014}$ & $\best{0.890\pm0.015}$ & $\best{0.938\pm0.011}$ & $\best{0.189\pm0.000}$ & $\best{0.013\pm0.008}$ \\
\quad Darcy & $\best{0.501\pm0.007}$ & $\best{0.796\pm0.008}$ & $\best{0.894\pm0.007}$ & $\best{0.945\pm0.005}$ & $0.116\pm0.000$ & $\best{0.007\pm0.003}$ \\
\multicolumn{7}{l}{Vanilla PINN} \\
\quad Diffusion & $\best{0.494\pm0.027}$ & $\best{0.799\pm0.027}$ & $\best{0.900\pm0.019}$ & $\best{0.949\pm0.010}$ & $0.200\pm0.007$ & $0.017\pm0.005$ \\
\quad Allen--Cahn & $\best{0.481\pm0.017}$ & $\best{0.783\pm0.029}$ & $\best{0.888\pm0.025}$ & $\best{0.940\pm0.023}$ & $0.199\pm0.005$ & $\best{0.018\pm0.020}$ \\
\quad Burgers & $\best{0.499\pm0.014}$ & $\best{0.806\pm0.012}$ & $\best{0.903\pm0.005}$ & $\best{0.951\pm0.003}$ & $0.200\pm0.005$ & $\best{0.006\pm0.005}$ \\
\quad Darcy & $0.450\pm0.004$ & $0.737\pm0.013$ & $0.848\pm0.011$ & $0.911\pm0.007$ & $0.093\pm0.002$ & $0.051\pm0.008$ \\
\multicolumn{7}{l}{MC-Dropout PINN} \\
\quad Diffusion & $0.420\pm0.031$ & $0.664\pm0.024$ & $0.750\pm0.022$ & $0.801\pm0.017$ & $0.211\pm0.006$ & $0.129\pm0.023$ \\
\quad Allen--Cahn & $0.402\pm0.017$ & $0.671\pm0.017$ & $0.767\pm0.017$ & $0.822\pm0.015$ & $0.201\pm0.006$ & $0.122\pm0.016$ \\
\quad Burgers & $0.390\pm0.014$ & $0.642\pm0.027$ & $0.736\pm0.028$ & $0.786\pm0.026$ & $0.212\pm0.006$ & $0.149\pm0.024$ \\
\quad Darcy & $0.310\pm0.020$ & $0.569\pm0.028$ & $0.708\pm0.036$ & $0.811\pm0.028$ & $0.098\pm0.003$ & $0.188\pm0.027$ \\
\multicolumn{7}{l}{Deep Ensemble ($K{=}5$)} \\
\quad Diffusion & $\best{0.507\pm0.012}$ & $\best{0.818\pm0.022}$ & $\best{0.910\pm0.016}$ & $\best{0.954\pm0.010}$ & $\best{0.195\pm0.009}$ & $0.015\pm0.008$ \\
\quad Allen--Cahn & $\best{0.495\pm0.011}$ & $\best{0.806\pm0.021}$ & $\best{0.904\pm0.017}$ & $\best{0.950\pm0.011}$ & $\best{0.195\pm0.009}$ & $0.013\pm0.004$ \\
\quad Burgers & $0.530\pm0.012$ & $0.835\pm0.021$ & $0.920\pm0.014$ & $\best{0.959\pm0.009}$ & $0.196\pm0.009$ & $0.024\pm0.013$ \\
\quad Darcy & $\best{0.498\pm0.008}$ & $\best{0.795\pm0.010}$ & $\best{0.894\pm0.009}$ & $\best{0.945\pm0.006}$ & $0.097\pm0.001$ & $\best{0.008\pm0.003}$ \\
			\hline
		\end{tabular}%
		}
	\end{table}
	\textbf{S$^{2}$-PINN achieves the lowest observed errors across benchmarks.}\quad
	Across diffusion, Allen--Cahn, and Burgers (polynomial stochastic dependence), we observe  lower mean errors and lower variance errors for S$^{2}$-PINN with projection relative to the best included baseline.  The no-projection variant remains
	competitive, which separates the architectural gain of the CP/gPC coupling
	from the additional regularization effect of the projected residual. Darcy is more nuanced: S$^{2}$-PINN attains the best mean accuracy, but the no-projection variant gives the best variance, already indicating that projection is not uniformly beneficial outside the polynomial regime.
    
	\textbf{Baselines.}\quad
	Stochastic collocation trains 32 independent node-PINNs and reaches only
	vanilla-PINN-level accuracy under our matched setting; it cannot share
	representations across random modes. 
	\add{The moment errors of PI-DeepONet remain about one order of magnitude above S$^{2}$-PINN.}
	MC-dropout has substantially larger moment errors and under-covers in the
	calibration table. Deep ensembles improve stability over a single PINN, but
	use five full networks and remain less accurate on the controlled benchmarks.
	\add{Among the structurally closer baselines, SPINN+Z is well calibrated but its moment errors are $1.5$--$33\times$ larger than those of S$^{2}$-PINN, functional CP-PINN is less accurate and less calibrated, and even the supervised Neural Chaos oracle does not reach S$^{2}$-PINN accuracy, indicating that the gains stem from the gPC basis combined with CP coupling rather than from separability alone.}

	\textbf{Computational cost.}\quad
	On stochastic diffusion, S$^{2}$-PINN trains in ${\approx} 90$\, mins with $1$\,GB peak memory, versus ${\approx} 80$\,mins /$4$\,GB for the vanilla PINN; SPINN+Z and functional CP-PINN use less memory but are remarkably less accurate (\Cref{tab:runtime}; cost analysis in Appendix~\ref{sec:cost-analysis}).

	\textbf{Scaling stress test.}\quad
	On the diffusion scaling benchmark, S$^{2}$-PINN remains stable up to
	$R{=}15$ random dimensions, corresponding to $G{=}816$ gPC modes
	(Appendix~\ref{sec:scaling-sweep}). Mean errors stay in the
	$(4$--$7)\times10^{-3}$ and Cov@90\% remains near $0.88$. \add{We further push the stochastic dimension to $R{=}50$ ($G{=}\binom{53}{3}{=}23{,}426$ modes), where S$^{2}$-PINN still attains $(1.09\pm0.14)\times10^{-2}$ mean and $(1.24\pm0.30)\times10^{-2}$ variance relative $L^2$ error (\Cref{tab:scaling-sweep}).} This \add{provides} evidence for scalability on this diffusion stress test, not as a
	universal high-dimensional guarantee. It is consistent with
	\Cref{prop:conc_agnostic}, where the projection-estimation batch size depends
	logarithmically on $G$ under boundedness assumptions.

    	\begin{figure}[htbp]
		\centering
		\includegraphics[width=0.8\linewidth]{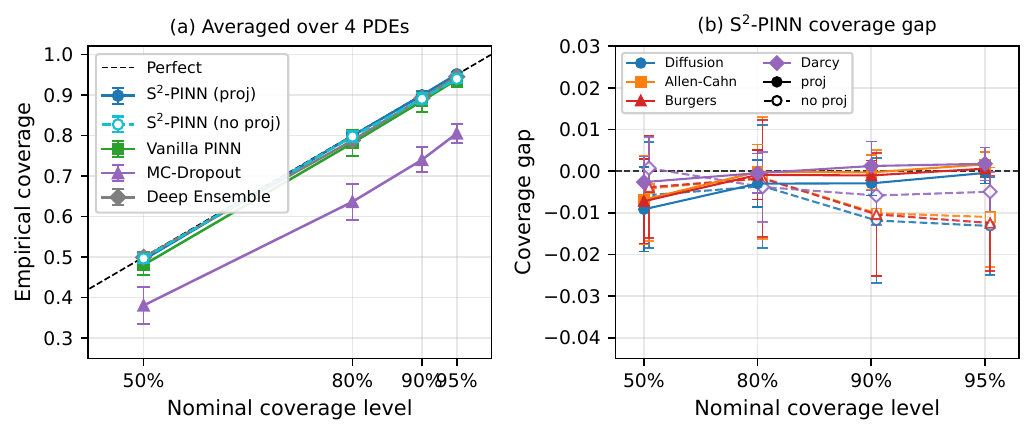}
		\caption{Reliability diagrams. (a) Empirical coverage vs.\ nominal level,
		averaged over 4 PDEs. S$^{2}$-PINN with projection and its no-projection ablation, Vanilla PINN; Deep Ensemble. The dashed diagonal indicates perfect calibration. (b) Per-PDE coverage gap, empirical minus nominal, for S$^{2}$-PINN with and without projection.}
		\label{fig:reliability}
	\end{figure}
\textbf{Calibration.}\quad
Figure~\ref{fig:reliability} evaluates predictive calibration through empirical
coverage at nominal levels $50\%$, $80\%$, $90\%$, and $95\%$. Averaged across
the four manufactured PDEs, S$^{2}$-PINN closely follows the ideal calibration
diagonal, while MC-dropout consistently under-covers. The no-projection
ablation remains reasonably calibrated at lower nominal levels but exhibits
mild under-coverage at $90\%$ and $95\%$, as made explicit by the coverage-gap
plot in Fig.~\ref{fig:reliability}(b). This suggests that the projected
residual is not merely improving mean and variance accuracy, but also helps
stabilize the pushforward uncertainty intervals. Vanilla PINN and the deep
ensemble can achieve near-nominal coverage on some controlled tasks, but their
weaker moment accuracy indicates that calibration alone is insufficient: a
useful stochastic surrogate must match both the predictive intervals and the
underlying solution statistics.

	\paragraph{External validity: non-manufactured benchmarks.}
	Manufactured solutions can provide artificially favorable optimization landscapes.
	To stress-test S$^{2}$-PINN beyond controlled settings, we evaluate on {two}
	non-manufactured forward problems with no closed-form solutions, using the
	{same} default hyperparameters as all manufactured benchmarks---no per-problem tuning.

	\textbf{(i) Stochastic Poisson source:}
    $-\Delta u = f(x,y,\mathbf{Z})$
	where the source $f$ is a KL-expanded random field ($\sigma{=}0.5$), with $u{=}
	0$ on $\partial\Omega$. \textbf{(ii) Stochastic Darcy flow:} $-\nabla\!\cdot\!
	(k(x,y,\mathbf{Z})\nabla u) = 1$ where the permeability $k$ is a log-normal field
	($\sigma{=}1.0$), with $u{=}0$ on $\partial\Omega$. Both are solved on $\Omega
	{=}[-1,1]^{2}$ and evaluated against finite-difference Monte Carlo
	references ($10{,}000$ samples, $64{\times}64$ grid). 

	\begin{table}[htbp]
		\centering
		\caption{Non-manufactured benchmark results.
		S$^{2}$-PINN uses the same default hyperparameters as all manufactured
		experiments; no per-problem tuning. Reference: FDM Monte Carlo ($N{=}10{,}
		000$, $64{\times}64$).}
		\label{tab:external-validity} \small
		\begin{tabular}{@{}llcccc@{}}
			\toprule                                                                                                                                      &                                   & \multicolumn{2}{c}{Poisson MC (source)} & \multicolumn{2}{c}{Darcy MC (coefficient)} \\
			\cmidrule(lr){3-4} \cmidrule(lr){5-6} Method                                                                                                  & \#Params                          & $\mathcal{E}_{\bar{u}}$                 & $\mathcal{E}_{\sigma^2}$                  & $\mathcal{E}_{\bar{u}}$  & $\mathcal{E}_{\sigma^2}$ \\
			\midrule S$^{2}$-PINN& 6\,114& \add{$\mathbf{0.335{\pm}0.122}$}& $\mathbf{1.39{\pm}1.10}$                  & $\mathbf{0.74{\pm}0.23}$ & $15.4{\pm}1.88$           \\
			Vanilla PINN                                                                                                                                  & 199\,937                          & $19.7{\pm}23.4$                         & $6.79{\pm}6.82$                           & $1.40{\pm}0.49$          & $5.40{\pm}1.92$          \\
			Deep Ens.\,($K{=}5$)                                                                                                                          & $5{\times}$199\,937               & $7.07\pm 1.11$& $77.8\pm0.89$& $1.04\pm0.19$& $\mathbf{4.92\pm1.46}$          \\
			\bottomrule \multicolumn{6}{@{}l@{}}{\scriptsize All errors $\times 10^{-2}$ (relative $L^{2}$).}                      \\
		\end{tabular}
	\end{table}
On the Poisson problem, S$^{2}$-PINN achieves $59\times$ lower mean error and $5\times$ lower variance error than the vanilla PINN while using $33\times$ fewer parameters. On the log-normal Darcy problem, S$^{2}$-PINN gives the best mean prediction ($0.74\%$), but the deep ensemble gives the best variance ($4.92\%$). This variance gap is a useful failure signal: low-rank gPC-aligned representations can under-resolve the nonlinear variance induced by exponential random coefficients.

\textbf{Diagnosing the failure regime before training.}\quad
The log-normal Darcy field violates the uniform boundedness (sub-Gaussian) assumption underlying \Cref{prop:conc_agnostic}: the exponential map pushes substantial stochastic energy into high-degree modes, which leaks out of a low-degree subspace $V_p$. This can be diagnosed offline, before training, with the \emph{subspace energy ratio}
\begin{equation}
\label{eq:ser}
\mathrm{SER}(p)=\frac{\sum_{|\boldsymbol{\alpha}|\le p}\|\hat a_{\boldsymbol{\alpha}}\|_{L^2}^2}{\mathbb{E}_{\mathbf{Z}}\big[\|a(\cdot,\mathbf{Z})\|_{L^2}^2\big]},
\end{equation}
i.e., the fraction of the energy of the input coefficient field $a(\mathbf{x},\mathbf{Z})$ captured by its gPC coefficients $\hat a_{\boldsymbol\alpha}$ of total degree at most $p$. A value $\mathrm{SER}(p)<95\%$ indicates that significant energy lies outside $V_p$, in which case the polynomial degree should be increased, a log-transformation applied, or the projection term disabled. 
On the log-normal Darcy problem, increasing the degree from $p{=}2$ to $p{=}4$ raises $\mathrm{SER}(p)$ above $95\%$ and reduces the S$^{2}$-PINN variance error from $15.4\%$ to $4.15\%$, below that of the deep ensemble ($4.92\%$).

\add{\textbf{Stochastic Navier--Stokes.}\quad To test a harder, nonlinear vector-valued problem, we further consider a 2D stochastic incompressible Navier--Stokes problem in vorticity form with five uniform random variables controlling phase shifts, viscosity, and two vortex amplitudes (Appendix~\ref{sec:navier-stokes}). S$^{2}$-PINN attains velocity mean errors of about $2\times10^{-3}$, $5$--$17\times$ lower than the baselines, with the lowest variance errors and the best calibration (\Cref{tab:ns-moments,tab:ns-calibration}).}
    	\begin{figure}[htbp]
		\centering
		\includegraphics[width=.9\linewidth]{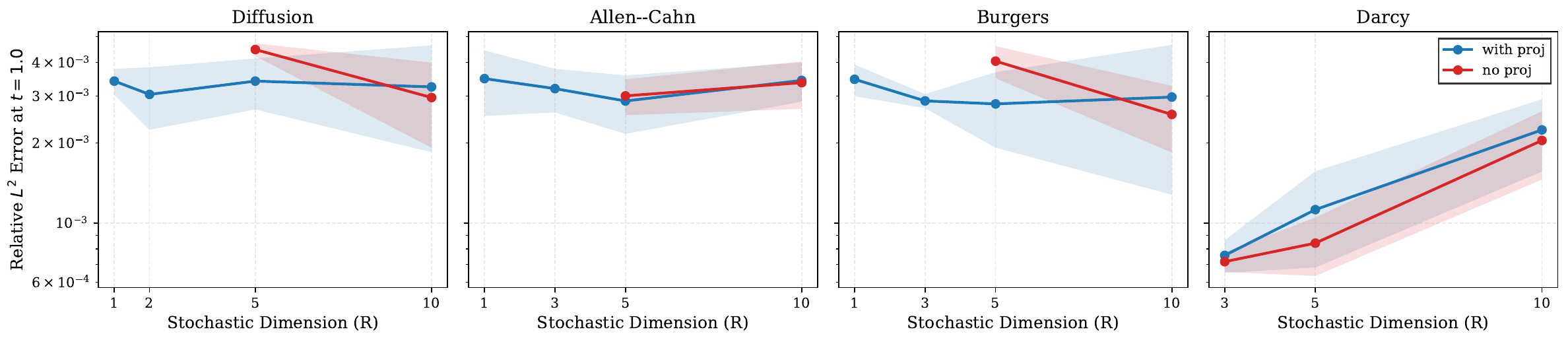}
		\caption{Regime map: relative $L^{2}$ at $t{=}1.0$ with/without gPC projection,
		across $R$ values and PDEs (diffusion, Allen--Cahn, Burgers, Darcy). Shaded
		bands show mean $\pm$ std over 5 seeds. Projection consistently helps
		for polynomial stochastic dependence (diffusion/Allen--Cahn/Burgers, $R{\le}
		5$) and becomes neutral or detrimental at higher $R$ or for non-polynomial/exponential
		$\mathbf{Z}$ (Darcy).}
		\label{fig:regime-map}
	\end{figure}

	\textbf{gPC projection is a regime-dependent regularizer.}\quad
	On diffusion, Allen--Cahn, and Burgers---all with polynomial dependence on $\mathbf{Z}$---the
	projection loss improves S$^{2}$-PINN's mean accuracy by \add{2--24\%}. On Darcy (exponential
	$\mathbf{Z}$-dependence), no-projection slightly outperforms projection on
	mean accuracy ($9.51 \times 10^{-4}$ vs.\ $1.05 \times 10^{-3}$). This is
	expected: the gPC projection penalizes residual components outside the polynomial
	subspace $V_{p}$, which is beneficial when the true solution is well-approximated
	by $V_{p}$ (polynomial $\mathbf{Z}$) but can introduce bias when it is not (exponential
	$\mathbf{Z}$). Figure~\ref{fig:regime-map} further visualizes this regime-dependent
	behavior across different $R$ values.
    
	\textbf{Projection in the vanilla PINN gives mixed results.}\quad
	Adding the gPC projection to a standard MLP-based PINN yields inconsistent behavior:
	it helps on Allen--Cahn (about 18\% mean-error reduction) but degrades
	diffusion, Burgers, and Darcy. Critically, the improvement is smaller than
	in S$^{2}$-PINN, suggesting that the \emph{separable tensor structure
	amplifies the benefit of Galerkin projection} by providing a representation that
	is structurally aligned with the gPC basis.

\textbf{PC$^{2}$ comparison.}\quad
We evaluate PC$^{2}$ on the
manufactured diffusion benchmark and the two non-manufactured forward
benchmarks using UQpy~\cite{olivier2020uqpy}, fix $p_Z{=}3$, and choose the
deterministic polynomial degree by held-out PDE/boundary residuals rather than
test moment errors.
\begin{table}[htbp]
	\centering
	\caption{{Physics-informed PCE (PC$^{2}$) comparison against S$^{2}$-PINN. 
	Errors are raw relative $L^2$ values.}}
	\label{tab:pc2-comparison}
	\scriptsize
	\resizebox{\columnwidth}{!}{%
	\begin{tabular}{lccccc}
		\toprule
		Problem & Solver & Selected $p_{\mathrm{det}}$ & \# Basis
		& $\mathcal{E}_{\bar{u}}$ & $\mathcal{E}_{\sigma^2}$ \\
		\midrule
        Poisson MC & LAR & $6,4,4$ & $15.67\pm8.02$
		& $0.068\pm0.078$
		& $0.957\pm0.068$ \\
		Poisson MC & OLS & $8,8,8$ & $2520\pm0$
		& $\best{(1.59\pm0.05)\times10^{-3}}$
		& $(1.44\pm0.39)\times10^{-2}$ \\
		Poisson MC & S$^{2}$-PINN (proj) & -- & --
		& $(3.35\pm1.22)\times$\add{$10^{-3}$}
		& $\best{(1.39\pm1.10)\times10^{-2}}$ \\
		\midrule
		Diffusion & LAR & $4,4,4$ & $3.67\pm1.15$
		& $1.103\pm0.181$
		& $9.452\pm14.640$ \\
		Diffusion & OLS & $3,3,3$ & $1120\pm0$
		& $25.798\pm0.569$
		& $262520.941\pm59229.735$ \\
		Diffusion & S$^{2}$-PINN (proj) & -- & --
		& $\best{(3.71\pm1.01)\times10^{-3}}$
		& $\best{(4.40\pm1.52)\times10^{-3}}$ \\
		\midrule
		Darcy MC & LAR & $4,4,6$ & $4.33\pm2.08$
		& $0.844\pm0.270$
		& $0.996\pm0.006$ \\
		Darcy MC & OLS & $8,8,6$ & $2202.67\pm549.64$
		& $1.483\pm0.349$
		& $1064.196\pm81.247$ \\
		Darcy MC & S$^{2}$-PINN (proj) & -- & --
		& $\best{(0.74\pm0.23)\times10^{-2}}$
		& $\best{(15.4\pm1.88)\times10^{-2}}$ \\
		\bottomrule
	\end{tabular}}
\end{table}
Table~\ref{tab:pc2-comparison} shows that physics-informed PCE is a strong but
regime-sensitive baseline. On the smooth stochastic Poisson source, anisotropic
PC$^{2}$ with OLS attains very low mean error and competitive variance error,
confirming that global polynomial surrogates are effective when the solution is
well captured by a low-order polynomial space. This strength does not transfer
to the harder regimes: PC$^{2}$ becomes unstable on the transient diffusion
benchmark and fails to recover log-normal Darcy variance, with sparse LAR
selecting a nearly constant basis and dense OLS producing large variance
errors. In contrast, S$^{2}$-PINN keeps the stochastic polynomial basis but
replaces the global deterministic polynomial surrogate with learnable
space--time features and low-rank coupling, yielding more stable
physics-informed UQ across these regimes.

\subsection{Stochastic inverse problems}
	Beyond forward UQ, the differentiable gPC structure of S$^{2}$-PINN naturally
	extends to \emph{stochastic inverse problems}, where unknown PDE parameters must
	be inferred jointly with the solution field. We demonstrate this on two
	distinct benchmarks: the stochastic diffusion equation (linear) and the stochastic
	Burgers equation (nonlinear). For both PDEs, we consider a log-normal random
	field $k(\mathbf{x},\mathbf{Z}) = \exp\bigl(\mu_{k}+ \sigma\,\sum_{r=1}^{R}\sqrt{\lambda_{r}}
	\,\varphi_{r}(\mathbf{x})\,Z_{r}\bigr)$ (representing the diffusion
	coefficient and the kinematic viscosity, respectively), treating the log-mean
	$\mu_{k}$ and field standard deviation $\sigma$ as unknown learnable parameters
	(true values: $\mu_{k}^{*}{=}0$, $\sigma^{*}{=}1$). Starting from
	deliberately poor initial guesses ($\hat{\mu}_{k}{=}0.5$,
	$\hat{\sigma}{=}0.5$), we augment the physics loss with a data-fidelity term
	$\mathcal{L}_{\mathrm{data}}= \frac{1}{N_{\mathrm{obs}}}\sum_{i=1}^{N_{\mathrm{obs}}}
	\|\hat{u}(t_{i},\mathbf{x}_{i},\mathbf{Z}_{i}) - u^{\mathrm{obs}}_{i}\|^{2}$
	computed over $N_{\mathrm{obs}}{=}1{,}000$ scattered observations, weighted by
	$w_{\mathrm{data}}{=}100$. The PDE parameters receive a $10\times$ higher
	learning rate than the model weights via separate Adam parameter groups, enabling
	the two scalars to converge at a rate commensurate with the $\sim$7\,k
	network weights. Table~\ref{tab:inverse-results} summaries the results. For diffusion, both parameters are recovered to within $5\times10^{-5}$
	of their true values, and the solution field achieves mean relative $L_{2}$
	error of $(1.29 \pm 0.36)\times10^{-4}$ and variance relative $L_{2}$ error
	of $(1.06 \pm 0.48)\times10^{-4}$. 
    For Burgers, the nonlinear advection term creates a more
	challenging optimization landscape: the log-mean is recovered to within $10^{-3}$
	and $\sigma$ to within $5\times10^{-4}$, with mean solution error
	$(2.03 \pm 0.62 )\times10^{-3}$ and variance error
	$(3.71 \pm 0.87)\times10^{-3}$---roughly one order of magnitude above
	diffusion, consistent with the harder forward-problem baseline. These
	results confirm that the gPC basis and CP factorization provide sufficient
	inductive bias for simultaneous parameter estimation and stochastic solution
	reconstruction across both linear and nonlinear regimes.
    \begin{table}[htbp]
		\centering
		\caption{Stochastic inverse problem results on the 2D diffusion and
		Burgers benchmarks. True parameters: $\mu_{k}^{*}{=}0.0$, $\sigma
		^{*}{=}1.0$. Initial guesses: $\hat{\mu}_{k}{=}0.5$,
		$\hat{\sigma}{=}0.5$.}
		\label{tab:inverse-results}
		\begin{tabular}{lcccc}
			\toprule PDE       & $|\hat{\mu}_{k}- \mu_{k}^{*}|$ & $|\hat{\sigma}- \sigma^{*}|$ & $\mathcal{E}_{\bar{u}}$ (Mean Rel.\ $L_{2}$) & $\mathcal{E}_{\sigma^2}$ (Var Rel.\ $L_{2}$) \\
			\midrule Diffusion & ${< 5\times10^{-5}}$        & ${< 5\times10^{-5}}$      & $(1.29 \pm 0.36)\times10^{-4}$& $(1.06 \pm 0.48)\times10^{-4}$               \\
			Burgers& $< 1\times10^{-3}$& $< 5\times10^{-4}$& $(2.03 \pm 0.62)\times10^{-3}$& $(3.71 \pm 0.87)\times10^{-3}$\\
			\bottomrule
		\end{tabular}
\end{table}

\section{Conclusion}
\label{sec:conclusion}
{In this work, we have presented \textbf{S$^2$-PINN} (Stochastic Separable Physics-Informed Neural Network), an
  efficient framework that extends separable coordinate architectures to forward and inverse uncertainty
  quantification for random PDEs.  By coupling learnable Gaussian spatial atoms, Fourier temporal features, and an
  orthonormal gPC stochastic basis via a low-rank CP tensor
  decomposition, S$^{2}$-PINN enables closed-form statistical moment evaluation and circumvents the curse of
  dimensionality inherent in dense space--time--stochastic grids.}
Our analysis also shows that S$^{2}$-PINN's 
gains come from these coupling three aligned
ingredients. For
polynomial-in-$\mathbf{Z}$ problems, the projection encourages residual
consistency in the stochastic modes used by the model, and
\Cref{prop:galerkin_agnostic} identifies the corresponding population objective
with a stochastic Galerkin constraint. The theory further supports the main
mechanism by \add{specializing standard tools (tensor-product density, orthogonal projection, and Monte Carlo concentration) to the proposed architecture and} establishing $L^{2}$ expressivity of finite-rank separable sums,
mini-batch concentration for projected coefficients, and conditioning control
for the learned spatial dictionary. 
Together, the theory and ablations support low-rank stochastic separability as
an effective modeling principle for physics-informed neural UQ.

\textbf{Limitations and future work.}
Although S$^{2}$-PINN shows great promise and outperforms across forward prediction,
calibration,  scaling, and inverse problems, there remain some
frontiers for making S$^{2}$-PINN a more general random PDE solver. First, S$^{2}$-PINN currently inherits the regular-domain
bias of separable coordinate architectures, and extending the method to complex
geometries will require geometry-aware dictionaries, coordinate maps, or mesh-based
variants. \add{Second, our theoretical results establish asymptotic expressivity and characterize the projection estimator under stated assumptions, but not provide a finite-$q$, finite-$G$ end-to-end solution-error guarantee: a finite rank $q$ introduces low-rank approximation error and a finite $G$ introduces stochastic truncation error.} 
{A promising way forward to overcome these limitations is to combine low-rank tensor representations with adaptive
  rank selection, learnable gPC truncation strategies to capture the distribution of random variables, and geometry-adaptive coordinate
  mappings.}

\textbf{Broader impact.}
Fast physics-informed UQ can support safer engineering design, climate and energy
modeling, and medical-device simulation by making uncertainty estimates more
accessible when repeated high-fidelity solves are expensive. The main risk is
misuse through overconfident deployment: inaccurate calibration or unresolved
variance in difficult stochastic regimes could lead to misleading confidence
estimates. For this reason, real-world use should pair S$^{2}$-PINN with
problem-specific validation, calibration checks, and conservative uncertainty
reporting.

    \section*{Acknowledgement}
The authors gratefully acknowledge support from the U.S. Department of Energy (DOE grant DE-SC0025801).

    \newpage
    \bibliographystyle{abbrvnat}

	\bibliography{ref,pinn, origin}
    \clearpage
	\appendix
	\section{Mathematical Formulation of Stochastic Elements and Tensor
	Factorization}
	\label{sec:math-formulation}

	This section details the mathematical construction of the stochastic fields,
	the chosen orthogonal polynomial basis, and the algebraic structure of the
	tensor decomposition that underpins S$^{2}$-PINN's efficiency.

	\subsection{Karhunen--Loève Expansion (KLE)}
	\label{sec:kle} 
    The stochastic coefficients (e.g.,
	the diffusion field $k(x,y,\mathbf{Z})$) are parameterized using the
	Karhunen--Loève Expansion. The KLE provides an optimal mean-square
	representation of a stochastic process by decomposing it into an infinite
	linear combination of orthogonal spatial eigenfunctions multiplied by
	uncorrelated random variables.

	For a random field defined by a specified covariance function
	$C(\mathbf{x}, \mathbf{x}')$, the KLE is given by
	\[
		\log k(\mathbf{x}, \mathbf{Z}) = \mu_{k}+ \sigma \sum_{i=1}^{\infty}\sqrt{\lambda_{i}}
		\phi_{i}(\mathbf{x}) Z_{i},
	\]
	where $\mu_{k}$ is the mean of the log-field, $\sigma$ controls the standard
	deviation (contrast), $\lambda_{i}$ and $\phi_{i}(\mathbf{x})$ are the
	eigenvalues and eigenfunctions of the covariance operator
	$\int C(\mathbf{x}, \mathbf{x}' )\phi_{i}(\mathbf{x}') d\mathbf{x}' = \lambda
	_{i}\phi_{i}(\mathbf{x})$, and $Z_{i}\sim \mathcal{N}(0,1)$ are independent standard
	normal random variables.

	In practical numerical implementations, this infinite series is truncated to
	the leading $R$ terms that capture the vast majority of the field's variance.
	By treating $\mathbf{Z}= (Z_{1}, \dots, Z_{R})^{\top}\in \mathbb{R}^{R}$ as explicit
	inputs, the stochastic PDE is transformed into a deterministic, high-dimensional
	parametric PDE over the domain $\Omega \times \mathcal{T}\times \mathbb{R}^{R}$.

	\subsection{Generalized Polynomial Chaos (gPC)}
	{The generalized Polynomial Chaos (gPC) framework provides a systematic
	approach to representing random processes through orthogonal polynomial
	expansions. For each probability distribution of the random input, there
	exists an optimal family of orthogonal polynomials that achieves
	\emph{spectral convergence} in the $L^2$ sense. This correspondence,
	known as the \emph{Wiener--Askey scheme}, is summarized in
	\Cref{tab:wiener-askey}.

	\begin{table}[htbp]
		\centering
		\caption{Wiener--Askey scheme: optimal polynomial basis for common
		distributions.}
		\label{tab:wiener-askey} \small
		\begin{tabular}{llll}
			\toprule
			Distribution & Support & Polynomial Family & Weight $w(z)$ \\
			\midrule
			Gaussian $\mathcal{N}(0,1)$ & $(-\infty, \infty)$ & Hermite & $e^{-z^2/2}/\sqrt{2\pi}$ \\
			Uniform $\mathcal{U}(-1,1)$ & $[-1, 1]$ & Legendre & $1/2$ \\
			Exponential $\mathrm{Exp}(1)$ & $[0, \infty)$ & Laguerre & $e^{-z}$ \\
			Beta $\mathrm{Beta}(\alpha,\beta)$ & $[0, 1]$ & Jacobi $P_n^{(\alpha,\beta)}$ & $z^{\beta-1}(1{-}z)^{\alpha-1}$ \\
			Gamma $\Gamma(\alpha,1)$ & $[0, \infty)$ & Gen.\ Laguerre $L_n^{(\alpha)}$ & $z^{\alpha-1}e^{-z}$ \\
			\bottomrule
		\end{tabular}
	\end{table}

	In our implementation, S$^{2}$-PINN automatically selects the polynomial
	family from this table based on the marginal distribution of each random
	input $Z_i$. For problems with mixed distributions (e.g., some $Z_i$
	Gaussian and others uniform), the multivariate basis is formed as a tensor
	product of the corresponding 1D families.}



	\paragraph{Multivariate basis and total-degree truncation.}
	For the $R$-dimensional vector $\mathbf{Z}$, the multidimensional gPC basis is
	constructed via the tensor product of 1D polynomials, indexed by a
	multi-index
	$\bm{\alpha}= (\alpha_{1}, \dots, \alpha_{R}) \in \mathbb{N}_{0}^{R}$
	\[
		\Psi_{\bm{\alpha}}(\mathbf{Z}) = \prod_{i=1}^{R}\psi_{\alpha_i}(Z_{i}),
	\]
	{where $\psi_{\alpha_i}$ denotes the degree-$\alpha_i$ polynomial from
	the family corresponding to the marginal distribution of $Z_i$
	(e.g., $P_{\alpha_i}$ for uniform).}
	To avoid an uncontrolled combinatorial explosion, we employ a \emph{total
	degree truncation} scheme, limiting the basis to all polynomials where the
	sum of degrees $|\bm{\alpha}| = \sum_{i=1}^{R}\alpha_{i}\le p$. This truncation
	drastically sparsifies the tensor product space. By a standard combinatorial
	``stars and bars'' argument, finding the number of non-negative integer tuples
	$(\alpha_{1}, \dots, \alpha_{R})$ satisfying $\sum_{i=1}^{R}\alpha_{i}\le p$
	is equivalent to distributing up to $p$ indistinguishable degrees among $R$ distinct
	variables. By introducing a ``slack'' variable to absorb the remaining
	$p - \sum_{i=1}^{R}\alpha_{i}$ degrees, this equals the number of ways to distribute
	exactly $p$ items among $R+1$ bins, yielding exactly $G = \binom{R+p}{p}$
	active gPC modes. {To illustrate how this basis is explicitly constructed, consider a 2D random space ($R{=}2$) with maximum degree $p{=}3$ and uniform inputs. The total number of valid modes is $G = \binom{2+3}{3} = 10$. The valid multi-indices $(\alpha_1, \alpha_2)$ and their corresponding 2D Legendre polynomials $\Psi_{(\alpha_1, \alpha_2)}(Z_1, Z_2) = H_{\alpha_1}(Z_1)H_{\alpha_2}(Z_2)$ are systematically enumerated by their total degree $|\bm{\alpha}|$:
	\begin{itemize}[itemsep=2pt, leftmargin=*]
	    \item \textbf{Degree 0:} $(0,0) \Rightarrow 1$
	    \item \textbf{Degree 1:} $(1,0) \Rightarrow Z_1$,\ \ $(0,1) \Rightarrow Z_2$
	    \item \textbf{Degree 2:} $(2,0) \Rightarrow \tfrac{1}{2}(3Z_1^{2}-1)$,\ \ $(1,1) \Rightarrow Z_1Z_2$,\ \ $(0,2) \Rightarrow \tfrac{1}{2}(3Z_2^{2}-1)$
	    \item \textbf{Degree 3:} $(3,0) \Rightarrow \tfrac{1}{2}(5Z_1^{3}-3Z_1)$,\ \ $(2,1) \Rightarrow \tfrac{1}{2}(3Z_1^{2}-1)Z_2$,\ \ $(1,2) \Rightarrow Z_1\tfrac{1}{2}(3Z_2^{2}-1)$,\ \ $(0,3) \Rightarrow \tfrac{1}{2}(5Z_2^{3}-3Z_2)$
	\end{itemize}
	This targeted truncation avoids the exponential blow-up of a full tensor grid. For higher dimensions such as $R{=}10$ and $p{=}3$, we only need $G{=}286$ modes instead of the naïve $4^{10}\approx 10^{6}$ grid points required by full tensor products.}

	\subsection{Canonical Polyadic  (CP) Tensor Decomposition vs.\ Tensor Train (TT)}
	The core architecture of S$^{2}$-PINN decomposes the multidimensional
	solution $u(\mathbf{x}, t, \mathbf{Z})$ into a rank-$q$ Canonical Polyadic  (CP) tensor product.

	If we represent the network's internal features as an abstract 3-way tensor
	$\mathcal{W}\in \mathbb{R}^{M \times M_t \times G}$, computing the full dense
	tensor would require $\mathcal{O}(M M_{t}G)$ parameters. The CP
	decomposition factors this tensor into a sum of $q$ rank-1 components
	\[
		u(\mathbf{x}, t, \mathbf{Z}) = \sum_{r=1}^{q}w_{r}\left( \mathbf{A}_{r}^{\top}
		\varphi(\mathbf{x}) \right) \left( \mathbf{B}_{r}^{\top}\psi(t) \right) \left
		( \mathbf{D}_{r}^{\top}\Phi(\mathbf{Z}) \right),
	\]
	where $\mathbf{A}\in \mathbb{R}^{M \times q}$,
	$\mathbf{B}\in \mathbb{R}^{M_t \times q}$, and $\mathbf{D}\in \mathbb{R}^{G
	\times q}$ are factor matrices, and $\mathbf{w}\in \mathbb{R}^{q}$ is a
	learnable scaling vector. This reduces the parameter complexity from multiplicative
	to additive: $\mathcal{O}(q(M + M_{t}+ G))$.

	\textbf{Why CP instead of Tensor Train (TT)?} While Tensor Train is a powerful
	decomposition for ultra-high dimensional spaces ($D \gg 10$), our functional
	representation operates fundamentally on a 3-way tensor (Space $\times$ Time
	$\times$ Stochastic), because the spatial coordinates $(x,y)$ are encoded together
	by the spatial dictionary and the $R$-dimensional $\mathbf{Z}$ is encoded together
	by the gPC basis. For a 3-way tensor, CP provides several decisive advantages
	over TT for PINN optimization:
	\begin{enumerate}
		\item \textbf{Decoupled Autodiff:} In CP, the spatial derivative operators
			(e.g., the Laplacian $\Delta$) distribute purely across the sum and act
			\emph{only} on the spatial terms
			$\mathbf{A}_{r}^{\top}\varphi(\mathbf{x})$. The time and stochastic components
			act as constants. This means the computational cost of automatic
			differentiation does not scale with $G$. In a TT format, computing
			spatial derivatives requires backpropagating through a chain of
			matrix multiplications, which introduces significant memory overhead
			and stiffness.

		\item \textbf{Explicit Galerkin Coefficients:} Our hybrid loss requires
			computing the exact Galerkin projection of the PDE residual onto the
			gPC modes $\Phi_{g}(\mathbf{Z})$. In the CP formulation, the
			macroscopic scalar coefficient attached to any gPC mode $\Phi_{g}$
			is explicitly given by the scalar sum
			$\sum_{r=1}^{q}w_{r}(\mathbf{A}_{r}^{\top}\varphi(\mathbf{x})) (\mathbf{B}
			_{r}^{\top}\psi(t)) D_{g,r}$. In a TT format, the stochastic
			subspace is folded inside a core matrix, making it algebraically opaque
			and computationally expensive to extract exact projection coefficients.
	\end{enumerate}

	\section{Theoretical Analysis}
	\label{sec:theory}

	{In this section, we collect PDE-agnostic results that support the main
	design choices in S$^{2}$-PINN: finite-rank separable expressivity,
	projection--Galerkin consistency, mini-batch stability of projected residual
	coefficients, orthogonality conditioning regularization, and the algebraic
	efficiency of the CP coupling. 
Instead of proving optimization convergence or finite-rank rates for a specific PDE, these results formalize the mechanism used by the architecture and loss.}
	\add{We emphasize that this section does not introduce new abstract functional-analytic tools: the arguments rely on standard results (density of algebraic tensor products, orthogonal projection and Parseval identities, and Monte Carlo concentration). Their method-specific content is twofold. (i) \emph{Expressive completeness}: \Cref{prop:density_agnostic} shows that a low-rank CP coupling of spatio-temporal features with a gPC basis preserves universal approximation in $L^2$, so the factorization introduces no representation bottleneck in the limit. (ii) \emph{Error decomposition}: Theorems~\ref{prop:density_agnostic}--\ref{prop:mc_agnostic} separate the total error into CP-rank approximation error (controlled by $q$), gPC truncation error (controlled by $p$), and Monte Carlo sampling error (controlled by the batch size $B$), which guides the balance between network capacity, spectral degree, and sampling budget.}
	\begin{addblock}
	\begin{remark}[Scope of the guarantees]
	The results below are asymptotic or estimator-level statements. They do not provide a finite-$q$, finite-$G$ end-to-end bound on the solution error of a trained model: a finite rank $q$ incurs low-rank approximation error, a finite $G$ incurs stochastic truncation error, and non-convex optimization under a finite budget is not covered.
	\end{remark}
	\end{addblock}

	\subsection{Setup and notation}
	Let $(\mathcal{T},\mu_{t})$, $(\Omega,\mu_{x})$, and $(\mathbf{Z},\mathbb{P}
	_{\mathbf{Z}})$ be measure spaces for time, spatial, and stochastic,
	respectively, and $L^{2}(\mathcal{T})$, $L^{2}(\Omega)$, $L^{2}(\mathbf{Z}
	)$ are corresponding $L^{2}$ spaces. Define the product space
	\[
    \mathcal{H}:=L^{2}(\mathcal{T})\otimes L^{2}(\Omega)\otimes L^{2}(\mathbf{Z}
	)\]
	with norm $\|\cdot\|_{L^2(\mathcal{T}\times\Omega\times\mathbf{Z})}$. Let $\{
	T_{k}\}_{k=1}^{M_t}\subset L^{2}(\mathcal{T})$,
	$\{\Phi_{i}\}_{i=1}^{M}\subset L^{2}(\Omega)$, $\{\Psi_{g}\}_{g=1}^{G}\subset
	L^{2}(\mathbf{Z})$ be (not necessarily orthonormal) feature families. {Throughout, we make explicit the \emph{mild density conditions} invoked in the main body: the feature families are total in their respective $L^{2}$ spaces, i.e.\ $\overline{\mathrm{span}}\{T_{k}\}_{k\ge 1}=L^{2}(\mathcal{T})$, $\overline{\mathrm{span}}\{\Phi_{i}\}_{i\ge 1}=L^{2}(\Omega)$, and $\overline{\mathrm{span}}\{\Psi_{g}\}_{g\ge 1}=L^{2}(\mathbf{Z})$, so that the closures $\mathcal{H}_{\mathrm{time}}$, $\mathcal{H}_{\mathrm{space}}$, and $\mathcal{H}_{\mathrm{stoc}}$ introduced in \Cref{prop:density_agnostic} below coincide with the full $L^{2}$ spaces. This assumption is harmless for the bases used in our experiments (Hermite/Legendre polynomials in $\mathbf{Z}$ are total in $L^{2}(\mathbf{Z})$, and our spatial Gaussian dictionary becomes total in $L^{2}(\Omega)$ in the limit $M\to\infty$).} For a
	generic signal $f:\mathcal{T}\times\Omega\times\mathbf{Z}\to\mathbb{R}$, define
	its (strong-form) "residual" symbolically as $r:=\mathcal{R}[f]$ where $\mathcal{R}$
	is any measurable operator. None of the results below require a PDE
	structure of $\mathcal{R}$.
    	\subsection{Expressivity of separable sums}
	Let
	$\mathcal{H}_{\mathrm{space}}=\overline{\mathrm{span}}\{\Phi_{i}\}_{i\ge 1}\subset
	L^{2}(\Omega)$, $\mathcal{H}_{\mathrm{time}}=\overline{\mathrm{span}}\{T_{k}\}
	_{k\ge 1}\subset L^{2}(\mathcal{T})$, and $\mathcal{H}_{\mathrm{stoc}}=\overline{\mathrm{span}}\{\Psi_{g}\}_{g\ge 1}\subset
	L^{2}(\mathbf{Z})$
	be dense subspaces (closures in $L^{2}$). Define the rank-$q$ separable class
	\[
        \mathcal{F}_{q}:=\Big\{\sum_{r=1}^{q}
        {\tau^{(r)}(t)\,\phi^{(r)}(\mathbf{x})\,\psi^{(r)}(\mathbf{z})}:\
 {\tau^{(r)}\in\mathcal{H}_{\mathrm{time}},\ }\phi^{(r)}\in\mathcal{H}
		_{\mathrm{space}},\ 
        \psi^{(r)}
		\in\mathcal{H}_{\mathrm{stoc}}\Big\}.
	\]

	\begin{theorem}[Density of finite-rank separable sums]
		\label{prop:density_agnostic} The union $\bigcup_{q\ge 1}\mathcal{F}_{q}$
		is dense in the Hilbert tensor product $\mathcal{H}_{\mathrm{time}}\otimes
		\mathcal{H}_{\mathrm{space}}\otimes\mathcal{H}_{\mathrm{stoc}}\subset L^{2}
		(\mathcal{T}\times\Omega\times\mathbf{Z})$.
	\end{theorem}
\begin{proof}
Define the algebraic tensor product of these three spaces, denoted as 
\[
\mathcal{V} = \mathcal{H}_{\mathrm{time}} \otimes_{\text{alg}} \mathcal{H}_{\mathrm{space}} \otimes_{\text{alg}} \mathcal{H}_{\mathrm{stoc}}.
\]
The elements of the algebraic tensor product $\mathcal{V}$ are strictly defined as formal finite linear combinations. That is, any element $f \in \mathcal{V}$ can be expressed as
\[
    f(t, \mathbf{x}, \mathbf{z}) = \sum_{r=1}^{q} \tau^{(r)}(t) \, \phi^{(r)}(\mathbf{x}) \, \psi^{(r)}(\mathbf{z}),
\]
where $q \in \mathbb{N}$ is a finite positive integer, and $\tau^{(r)} \in \mathcal{H}_{\mathrm{time}}$, $\phi^{(r)} \in \mathcal{H}_{\mathrm{space}}$, $\psi^{(r)} \in \mathcal{H}_{\mathrm{stoc}}$. 

By definition, the function class $\mathcal{F}_q$ of separable functions of rank $q$ precisely describes the set of functions formed by the addition of at most $q$ pure tensors. Therefore, taking the union over all possible finite ranks $q \ge 1$ exactly yields the entire algebraic tensor product space $\bigcup_{q \ge 1} \mathcal{F}_q = \mathcal{V}.$

For two tensors $u_1 \otimes v_1 \otimes w_1$ and $u_2 \otimes v_2 \otimes w_2$, their inner product is defined as the product of the inner products within their respective spaces
\[
    \langle u_1 \otimes v_1 \otimes w_1, u_2 \otimes v_2 \otimes w_2 \rangle = \langle u_1, u_2 \rangle_{\mathcal{H}_{\mathrm{time}}} \langle v_1, v_2 \rangle_{\mathcal{H}_{\mathrm{space}}} \langle w_1, w_2 \rangle_{\mathcal{H}_{\mathrm{stoc}}}.
\]
Since all three spaces are subspaces of $L^2$ spaces, the aforementioned inner product is equivalent to the standard $L^2$ inner product under the product measure $\mu_t \otimes \mu_x \otimes \mathbb{P}_{\mathbf{Z}}$. 

Extending this inner product bilinearly to the entire algebraic tensor product $\mathcal{V}$ induces a norm $\|\cdot\|_{L^2(\mathcal{T} \times \Omega \times {\mathbf{Z}})}$. We denote the Hilbert tensor product as $\mathcal{H}_{\mathrm{time}} \otimes \mathcal{H}_{\mathrm{space}} \otimes \mathcal{H}_{\mathrm{stoc}}$. By definition, it is the completion (or closure) of the algebraic tensor product $\mathcal{V}$ with respect to this induced $L^2$ norm
\[
    \mathcal{H}_{\mathrm{time}} \otimes \mathcal{H}_{\mathrm{space}} \otimes \mathcal{H}_{\mathrm{stoc}} := \overline{\mathcal{V}}^{\|\cdot\|_{L^2}}.
\]

According to the fundamental principles of topology and metric spaces, the closure of a set is the smallest closed set containing it. Furthermore, a subset $\mathcal{V}$ is dense in a space $\mathcal{H}$ if and only if the closure of $\mathcal{V}$ is exactly $\mathcal{H}$. 

Given that the Hilbert tensor product is intrinsically the closure of the algebraic tensor product, we have
\[
    \overline{\bigcup_{q \ge 1} \mathcal{F}_q}^{\|\cdot\|_{L^2}} = \overline{\mathcal{V}}^{\|\cdot\|_{L^2}} = \mathcal{H}_{\mathrm{time}} \otimes \mathcal{H}_{\mathrm{space}} \otimes \mathcal{H}_{\mathrm{stoc}}.
\]
This establishes that for any arbitrary function $f$ in the Hilbert tensor product space and any given error tolerance $\epsilon > 0$, there necessarily exists a sufficiently large finite integer $q$ and an approximant $\tilde{f} = \sum_{r=1}^{q}\phi^{(r)}\tau^{(r)}\psi^{(r)} \in \mathcal{F}_q$ such that:
\[
    \|f - \tilde{f}\|_{L^2(\mathcal{T} \times \Omega \times {\mathbf{Z}})} < \epsilon.
\]
Consequently, the union of finite-rank separable functions, $\bigcup_{q \ge 1} \mathcal{F}_q$, is dense in the Hilbert tensor product space $\mathcal{H}_{\mathrm{time}} \otimes \mathcal{H}_{\mathrm{space}} \otimes \mathcal{H}_{\mathrm{stoc}}$.

Finally, under the mild density conditions stated in \Cref{sec:theory}, the closures $\mathcal{H}_{\mathrm{time}}$, $\mathcal{H}_{\mathrm{space}}$, $\mathcal{H}_{\mathrm{stoc}}$ coincide with the full $L^{2}$ spaces, so the Hilbert tensor product satisfies
\[
    \mathcal{H}_{\mathrm{time}} \otimes \mathcal{H}_{\mathrm{space}} \otimes \mathcal{H}_{\mathrm{stoc}}
    = L^{2}(\mathcal{T}) \otimes L^{2}(\Omega) \otimes L^{2}(\mathbf{Z})
    \cong L^{2}(\mathcal{T} \times \Omega \times \mathbf{Z}),
\]
where the final isomorphism is the standard one for $\sigma$-finite product measures (a consequence of Fubini's theorem). Combining this identification with the density just established yields
\[
    \overline{\bigcup_{q \ge 1} \mathcal{F}_q}^{\|\cdot\|_{L^{2}}} = L^{2}(\mathcal{T} \times \Omega \times \mathbf{Z}).
\]
\end{proof}

	\subsection{Projection in the random space}
	Assume $\{\Psi_{g}\}_{g=1}^{G}$ are orthonormal in $L^{2}(\mathbf{Z})$ and
	let $V_{p}=\mathrm{span}\{\Psi_{g}\}_{g=1}^{G}$.
	\begin{theorem}[Projection–Galerkin equivalence in $L^{2}(\mathbf{Z})$]
		\label{prop:galerkin_agnostic} For any measurable $h:\mathbf{Z}\to\mathbb{R}$,
		\[
			\begin{aligned}
				\sum_{g=1}^{G}\big(\mathbb{E}[\,h(\mathbf{Z})\,\Psi_{g}(\mathbf{Z})\,]\big)^{2} & =\big\|\mathsf{P}_{V_p}h\big\|_{L^2(\mathbf{Z})}^{2}, \\
				\mathbb{E}[\,h\Psi_{g}\,]=0\ \forall g                                          & \iff \mathsf{P}_{V_p}h=0.
			\end{aligned}
		\]
	\end{theorem}
	\noindent
    \begin{proof}
According to the Orthogonal Projection Theorem in Hilbert spaces, for any function $h \in L^2(\mathbf{Z})$ can be uniquely decomposed into an orthogonal sum:
\[
    h = h_{\parallel} + h_{\perp},
\]
where $h_{\parallel} = \mathsf{P}_{V_p}h \in V_p$ is the orthogonal projection of $h$ onto the subspace $V_p$, and $h_{\perp} \in V_p^{\perp}$ is the residual component such that $\langle h_{\perp}, v \rangle_{L^2(\mathbf{Z})} = 0$ for all $v \in V_p$. Given that $V_p = \mathrm{span}\{\Psi_{g}\}_{g=1}^{G}$, the projection $h_{\parallel}$ can be expressed as a linear combination of the basis functions
\[
    h_{\parallel} = \sum_{g=1}^{G} \alpha_{g} \Psi_{g},
\]
where $\alpha_{g} \in \mathbb{R}$ denote the expansion coefficients.

\paragraph{Part I: Proof of the First Equality.}
In the $L^2(\mathbf{Z})$ space, the inner product is equivalent to the expectation under the probability measure $\mathbb{P}_{\mathbf{Z}}$, i.e., $\langle u, v \rangle_{L^2} = \mathbb{E}[uv]$. Taking the inner product of $h$ with the $g'$-th basis function $\Psi_{g'}$, we have
\[
    \mathbb{E}[\,h\,\Psi_{g'}\,] = \langle h, \Psi_{g'} \rangle_{L^2(\mathbf{Z})}.
\]
Substituting the orthogonal decomposition $h = \sum_{g=1}^{G} \alpha_{g} \Psi_{g} + h_{\perp}$ into the expression above yields
\[
    \mathbb{E}[\,h\,\Psi_{g'}\,] = \left\langle \sum_{g=1}^{G} \alpha_{g} \Psi_{g} + h_{\perp}, \Psi_{g'} \right\rangle_{L^2(\mathbf{Z})}=\sum_{g=1}^{G} \alpha_{g} \langle \Psi_{g}, \Psi_{g'} \rangle_{L^2(\mathbf{Z})} + \langle h_{\perp}, \Psi_{g'} \rangle_{L^2(\mathbf{Z})}.
\]
We now invoke the following properties:
\begin{enumerate}
    \item Since $h_{\perp} \perp V_p$ and $\Psi_{g'} \in V_p$, it follows that $\langle h_{\perp}, \Psi_{g'} \rangle_{L^2(\mathbf{Z})} = 0$.
    \item The basis $\{\Psi_{g}\}$ is orthonormal in $L^2(\mathbf{Z})$, meaning $\langle \Psi_{g}, \Psi_{g'} \rangle_{L^2(\mathbf{Z})} = \delta_{g,g'}$, where $\delta_{g,g'}$ is the Kronecker delta.
\end{enumerate}
Substituting these properties, all terms in the summation vanish except for $g = g'$, leading to $\mathbb{E}[\,h\,\Psi_{g'}\,] = \alpha_{g'}$.
This confirms that the projection coefficients are precisely the expectations (inner products) of $h$ with the basis functions. 

Next, we compute the squared $L^2$-norm of the projection $h_{\parallel} = \mathsf{P}_{V_p}h$,
\[
    \big\|\mathsf{P}_{V_p}h\big\|_{L^2(\mathbf{Z})}^{2} = \left\| \sum_{g=1}^{G} \alpha_{g} \Psi_{g} \right\|_{L^2(\mathbf{Z})}^{2}.
\]
Applying the orthonormality of $\{\Psi_{g}\}$ (a generalized form of the Pythagorean theorem or Parseval’s identity), we obtain
\[
    \big\|\mathsf{P}_{V_p}h\big\|_{L^2(\mathbf{Z})}^{2} = \sum_{g=1}^{G} \alpha_{g}^{2}.
\]
Substituting $\alpha_{g} = \mathbb{E}[\,h(\mathbf{Z})\,\Psi_{g}(\mathbf{Z})\,]$ into the equation above directly yields the first desired identity
\[
    \sum_{g=1}^{G}\big(\mathbb{E}[\,h(\mathbf{Z})\,\Psi_{g}(\mathbf{Z})\,]\big)^{2} = \big\|\mathsf{P}_{V_p}h\big\|_{L^2(\mathbf{Z})}^{2}.
\]

\paragraph{Part II: Proof of the Equivalence.}
We aim to prove that $\mathbb{E}[\,h\Psi_{g}\,]=0$ for all $g \in \{1, \dots, G\}$ if and only if $\mathsf{P}_{V_p}h=0$. From Part I, we have established
\[
    \big\|\mathsf{P}_{V_p}h\big\|_{L^2(\mathbf{Z})}^{2} = \sum_{g=1}^{G}\big(\mathbb{E}[\,h\Psi_{g}\,]\big)^{2}.
\]

\textit{Sufficiency ($\implies$):} If $\mathbb{E}[\,h\Psi_{g}\,] = 0$ for all $g \in \{1, \dots, G\}$, then every term in the summation on the right-hand side is zero. Consequently, $\big\|\mathsf{P}_{V_p}h\big\|_{L^2(\mathbf{Z})}^{2} = 0$. By the definiteness property of the norm, it follows that $\mathsf{P}_{V_p}h = 0$.

\textit{Necessity ($\impliedby$):} Conversely, if $\mathsf{P}_{V_p}h = 0$, then its squared norm must be zero, $\big\|\mathsf{P}_{V_p}h\big\|_{L^2(\mathbf{Z})}^{2} = 0$. This implies $    \sum_{g=1}^{G}\big(\mathbb{E}[\,h\Psi_{g}\,]\big)^{2} = 0.
$
Since this is a sum of non-negative real terms (squares), the sum equals zero if and only if each individual term is zero. Therefore, it must hold that $\mathbb{E}[\,h\Psi_{g}\,] = 0$ for all $g \in \{1, \dots, G\}$.
\end{proof}

{\subsection{Mini-batch estimator and uniform concentration}
The two theorems in this subsection control the Monte-Carlo error incurred when the projection coefficients $\mathbb{E}[h\Psi_{g}]$ are estimated from a finite collocation batch. The first establishes unbiasedness and the standard $1/B$ variance scaling and the second upgrades this to a high-probability uniform bound across all $G$ basis functions, which is the form used in our training analysis.}

\begin{theorem}[Mini-batch estimator: unbiasedness and variance]
		\label{prop:mc_agnostic} 
        Let $\widehat{m}_{g}=\frac{1}{B}\sum_{b=1}^{B}h(
		\mathbf{Z}_{b})\Psi_{g}(\mathbf{Z}_{b})$ with i.i.d.\ $\mathbf{Z}_{b}\sim
		\mathbb{P}_{\mathbf{Z}}$. Then $\mathbb{E}\widehat{m}_{g}=\mathbb{E}[h\Psi
		_{g}]$ and $\mathrm{Var}(\widehat{m}_{g})=\mathrm{Var}(h\Psi_{g})/B$.
		Moreover, for $\widehat{\mathbf{m}}=(\widehat{m}_{g})_{g=1}^{G}$ and $\mathbf{m}
		=(\mathbb{E}[h\Psi_{g}])_{g=1}^{G}$,
		$\mathbb{E}\|\widehat{\mathbf{m}}\|_{2}^{2}=\|\mathbf{m}\|_{2}^{2}+\frac{1}{B}
		\sum_{g=1}^{G}\mathrm{Var}(h\Psi_{g})$.
	\end{theorem}
\begin{proof}
\textbf{Part I: Unbiasedness of the Estimator $\mathbb{E}[\widehat{m}_{g}] = \mathbb{E}[\,h\Psi_{g}\,]$.}
By definition, the estimator is given by
\[
    \widehat{m}_{g} = \frac{1}{B}\sum_{b=1}^{B}h(\mathbf{Z}_{b})\Psi_{g}(\mathbf{Z}_{b}).
\]
Taking the expectation of $\widehat{m}_{g}$ and invoking the linearity of expectation, we have
\[
    \mathbb{E}[\widehat{m}_{g}] = \mathbb{E}\left[ \frac{1}{B}\sum_{b=1}^{B}h(\mathbf{Z}_{b})\Psi_{g}(\mathbf{Z}_{b}) \right] = \frac{1}{B}\sum_{b=1}^{B}\mathbb{E}[\,h(\mathbf{Z}_{b})\Psi_{g}(\mathbf{Z}_{b})\,].
\]
Since the samples $\{\mathbf{Z}_{b}\}_{b=1}^{B}$ are independent and identically distributed (i.i.d.) following the probability measure $\mathbb{P}_{\mathbf{Z}}$, the expectation for each $b$ equals the population expectation
\[
    \mathbb{E}[\,h(\mathbf{Z}_{b})\Psi_{g}(\mathbf{Z}_{b})\,] = \mathbb{E}[\,h(\mathbf{Z})\Psi_{g}(\mathbf{Z})\,] = m_{g}.
\]
Substituting this back into the expression yields
\[
    \mathbb{E}[\widehat{m}_{g}] = \frac{1}{B}\sum_{b=1}^{B} m_{g} = \frac{1}{B} (B \cdot m_{g}) = m_{g} = \mathbb{E}[\,h\Psi_{g}\,].
\]
This confirms that $\widehat{m}_{g}$ is an unbiased estimator of $m_{g}$.

\textbf{Part II: Variance of the Estimator $\mathrm{Var}(\widehat{m}_{g}) = \mathrm{Var}(h\Psi_{g})/B$.}
We now consider the variance of $\widehat{m}_{g}$. Recall that for independent random variables $\{X_i\}$, $\mathrm{Var}(\sum X_i) = \sum \mathrm{Var}(X_i)$, and for a constant $C$, $\mathrm{Var}(CX) = C^2 \mathrm{Var}(X)$. Applying these properties, we have
\[
    \mathrm{Var}(\widehat{m}_{g}) = \mathrm{Var}\left( \frac{1}{B}\sum_{b=1}^{B}h(\mathbf{Z}_{b})\Psi_{g}(\mathbf{Z}_{b}) \right) = \frac{1}{B^2} \sum_{b=1}^{B} \mathrm{Var}\big(h(\mathbf{Z}_{b})\Psi_{g}(\mathbf{Z}_{b})\big).
\]
Because the samples $\mathbf{Z}_{b}$ are identically distributed, the variance is constant for all $b$,
\[
    \mathrm{Var}\big(h(\mathbf{Z}_{b})\Psi_{g}(\mathbf{Z}_{b})\big) = \mathrm{Var}(h\Psi_{g}).
\]
Substituting this back, we obtain
\[
    \mathrm{Var}(\widehat{m}_{g}) = \frac{1}{B^2} \big(B \cdot \mathrm{Var}(h\Psi_{g})\big) = \frac{\mathrm{Var}(h\Psi_{g})}{B}.
\]
This demonstrates that the variance of the estimator decays at the standard Monte Carlo rate of $O(1/B)$ relative to the batch size $B$.

\textbf{Part III: Expected Squared $L_2$ Norm $\mathbb{E}\|\widehat{\mathbf{m}}\|_{2}^{2} = \|\mathbf{m}\|_{2}^{2} + \frac{1}{B}\sum_{g=1}^{G}\mathrm{Var}(h\Psi_{g})$.}
Using the definition of the squared $L_2$ norm and applying the linearity of expectation, we obtain
\[
    \mathbb{E}\|\widehat{\mathbf{m}}\|_{2}^{2} = \mathbb{E} \bigg[ \sum_{g=1}^{G} \widehat{m}_{g}^{2} \bigg] = \sum_{g=1}^{G} \mathbb{E}[\widehat{m}_{g}^{2}].
\]
Utilizing the fundamental relation $\mathbb{E}[X^2] = (\mathbb{E}[X])^2 + \mathrm{Var}(X)$, we can decompose the second moment of each coefficient and get
\[
    \mathbb{E}[\widehat{m}_{g}^{2}] = (\mathbb{E}[\widehat{m}_{g}])^2 + \mathrm{Var}(\widehat{m}_{g}).
\]
Substituting the unbiasedness result ($\mathbb{E}[\widehat{m}_{g}] = m_{g}$) and the variance result ($\mathrm{Var}(\widehat{m}_{g}) = \frac{1}{B}\mathrm{Var}(h\Psi_{g})$) derived previously, we have
\[
    \mathbb{E}[\widehat{m}_{g}^{2}] = m_{g}^{2} + \frac{1}{B}\mathrm{Var}(h\Psi_{g}).
\]
Summing over all $g \in \{1, \dots, G\}$ gives
\[
    \mathbb{E}\|\widehat{\mathbf{m}}\|_{2}^{2} = \sum_{g=1}^{G} m_{g}^{2} + \frac{1}{B} \sum_{g=1}^{G} \mathrm{Var}(h\Psi_{g}).
\]
Recognizing that $\|\mathbf{m}\|_{2}^{2} = \sum_{g=1}^{G} m_{g}^{2}$ is the squared norm of the true coefficient vector, we obtain
\[
    \mathbb{E}\|\widehat{\mathbf{m}}\|_{2}^{2} = \|\mathbf{m}\|_{2}^{2} + \frac{1}{B}\sum_{g=1}^{G}\mathrm{Var}(h\Psi_{g}).
\]
\end{proof}
	\begin{theorem}[Uniform concentration over $g$]
		\label{prop:conc_agnostic} If $\{h(\mathbf{Z})\Psi_{g}(\mathbf{Z})\}_{g\le
		G}$ are sub-Gaussian with common proxy $\sigma^{2}$, then for any
		$\epsilon,\delta\in(0,1)$, if $B\gtrsim \sigma^{2}\epsilon^{-2}\log(G/\delta
		)$,
		\[
			\mathbb{P}\Big(\max_{1\le g\le G}\big|\widehat{m}_{g}-\mathbb{E}[h\Psi
			_{g}]\big|\le \epsilon\Big) \ge 1-\delta.
		\]
	\end{theorem}
\begin{proof}
Since the samples $\mathbf{Z}_b$ are independent and identically distributed (i.i.d.), the sequence of random variables $\{X_{g}(\mathbf{Z}_b)\}_{b=1}^B$ are mutually independent, each being a sub-Gaussian random variable with proxy variance $\sigma^2$. 

By the properties of sub-Gaussian random variables, the empirical mean $\widehat{m}_{g}$ of independent sub-Gaussian variables is also sub-Gaussian. Specifically, the proxy variance of the mean $\widehat{m}_g$ scales by a factor of $1/B$ relative to the single-sample proxy variance, given by $\sigma_{\text{mean}}^2 = \frac{\sigma^2}{B}.$
Applying the Hoeffding-type tail bound for centered sub-Gaussian random variables, the probability that the empirical mean deviates from the true mean $m_g$ by more than $\epsilon$ satisfies
\[
    \mathbb{P}\Big(\big|\widehat{m}_{g} - m_g\big| > \epsilon\Big) \le 2 \exp\left( - \frac{\epsilon^2}{2\sigma_{\text{mean}}^2} \right)=2 \exp\left( - \frac{B \epsilon^2}{2\sigma^2} \right).
\]
This inequality establishes an upper bound on the probability that the estimation error for a specific feature basis $g$ exceeds the margin $\epsilon$.

Then, we aim to bound the maximum deviation across all $g \in \{1, \dots, G\}$ by $\epsilon$. The event where the maximum deviation exceeds $\epsilon$ is mathematically equivalent to the event that at least one component $g$ has a deviation exceeding $\epsilon$.

Applying Boole's inequality (the union bound), i.e., $\mathbb{P}(\cup A_i) \le \sum \mathbb{P}(A_i)$, we can upper bound the total failure probability
\[
    \mathbb{P}\Big(\max_{1\le g\le G}\big|\widehat{m}_{g} - m_g\big| > \epsilon\Big) = \mathbb{P}\left( \bigcup_{g=1}^{G} \Big\{ \big|\widehat{m}_{g} - m_g\big| > \epsilon \Big\} \right)\le \sum_{g=1}^{G} \mathbb{P}\Big(\big|\widehat{m}_{g} - m_g\big| > \epsilon\Big).
\]
Because the failure probability bound is identical for each $g$, substituting and summing over the $G$ components gives
\[
    \mathbb{P}\Big(\max_{1\le g\le G}\big|\widehat{m}_{g} - m_g\big| > \epsilon\Big) \le \sum_{g=1}^{G} 2 \exp\left( - \frac{B \epsilon^2}{2\sigma^2} \right) = 2G \exp\left( - \frac{B \epsilon^2}{2\sigma^2} \right).
\]

To guarantee that the probability of the maximum error exceeding $\epsilon$ (the failure probability) is controlled below our predefined tolerance level $\delta$, we enforce the condition
\[
    2G \exp\left( - \frac{B \epsilon^2}{2\sigma^2} \right) \le \delta.
\]
Dividing by $2G$ and solving for $B$, we obtain
\[
    B \ge \frac{2\sigma^2}{\epsilon^2} \log\left(\frac{2G}{\delta}\right).
\]
We can decompose the logarithmic term as $\log(2G/\delta) = \log(G/\delta) + \log(2)$. In an asymptotic sense (denoted by the Big-O or $\gtrsim$ notation), the absolute constant factor $2$ can be absorbed. Therefore, provided that the batch size $B$ satisfies:
\[
    B \gtrsim \frac{\sigma^2}{\epsilon^2} \log\left(\frac{G}{\delta}\right)
\]
(where $\gtrsim$ hides absolute constants, such as $2$ in this derivation), the probability that the maximum error exceeds $\epsilon$ is at most $\delta$. 

Taking the complement of this event (the success event) yields the final uniform bound
\[
    \mathbb{P}\Big(\max_{1\le g\le G}\big|\widehat{m}_{g} - \mathbb{E}[h\Psi_{g}]\big| \le \epsilon\Big) \ge 1 - \delta.
\]
\end{proof}
\begin{addblock}
\begin{remark}[When the assumption fails]
The sub-Gaussian (in particular, uniformly bounded) assumption on $h\Psi_g$ is natural for polynomial dependence on bounded or Gaussian inputs, but it is violated by log-normal coefficient fields, whose exponential dependence on $\mathbf{Z}$ produces heavy tails and places substantial energy in high-degree modes. This is consistent with the reduced variance accuracy observed on the log-normal Darcy benchmark, and it can be detected before training with the subspace energy ratio~\eqref{eq:ser}.
\end{remark}
\end{addblock}
	\subsection{Feature Gram/Orthogonality conditioning and its concentration}
	Let $F_{\mathrm{s}}\in\mathbb{R}^{B\times M}$ be the spatial feature matrix with
	rows $F_{\mathrm{s}}(b,:)=\big(\Phi_{i}(\mathbf{x}_{b})\big)_{i=1}^{M}$ sampled
	from $\mu_{x}$. Define the empirical Gram
	$\widehat{G}_{\Phi}=\frac{1}{B}F_{\mathrm{s}}^{\top}F_{\mathrm{s}}$ and
	population Gram
	$G_{\Phi}=\mathbb{E}_{\mathbf{x}\sim\mu_x}[\Phi(\mathbf{x})\Phi(\mathbf{x})^{\top}
	],$
	where $\Phi( \mathbf{x})=(\Phi_{i}(\mathbf{x}))_{i=1}^{M}$.

	\begin{theorem}[Near-orthogonality $\Rightarrow$ bounded condition
	number]
		\label{prop:cond_agnostic} If $\|\widehat{G}_{\Phi}-I\|_{2}\le \delta<1$,
		then the singular values of $F_{\mathrm{s}}/\sqrt{B}$ lie in
		$[\sqrt{1-\delta},\sqrt{1+\delta}]$ and
		$\kappa(F_{\mathrm{s}})\le \sqrt{\frac{1+\delta}{1-\delta}}$.
	\end{theorem}
\begin{proof}
Define the matrix $A = \frac{1}{\sqrt{B}}F_{\mathrm{s}}$. By the properties of the singular value decomposition (SVD), the squared singular values of $A$, denoted by $\sigma_j^2(A)$, coincide with the eigenvalues of the symmetric matrix $A^\top A$, denoted by $\lambda_j(A^\top A)$. Observing the expansion of $A^\top A$,
\[
    A^\top A = \left(\frac{1}{\sqrt{B}}F_{\mathrm{s}}\right)^\top \left(\frac{1}{\sqrt{B}}F_{\mathrm{s}}\right) = \frac{1}{B}F_{\mathrm{s}}^\top F_{\mathrm{s}}.
\]
According to the given definition, this is precisely the empirical Gram matrix $\widehat{G}_{\Phi}$. Consequently, we establish
\[
    \lambda_j(\widehat{G}_{\Phi}) = \sigma_j^2\left(\frac{1}{\sqrt{B}}F_{\mathrm{s}}\right).
\]

We are given the condition $\|\widehat{G}_{\Phi}-I\|_{2} \le \delta < 1$. For any real symmetric matrix $M$, its spectral norm (or 2-norm) is equal to the maximum absolute value of its eigenvalues, which implies
\[
    \|M\|_{2} = \max_j |\lambda_j(M)|.
\]
Substituting $M = \widehat{G}_{\Phi}-I$ yields that for all $j$,
$    |\lambda_j(\widehat{G}_{\Phi}-I)| \le \delta.
$
Since the identity matrix $I$ has all eigenvalues equal to 1 and shares the same eigenspace with $\widehat{G}_{\Phi}$, the eigenvalues of $\widehat{G}_{\Phi}-I$ can be expressed simply as $\lambda_j(\widehat{G}_{\Phi}) - 1$. Substituting this into the inequality gives
\[
    |\lambda_j(\widehat{G}_{\Phi}) - 1| \le \delta.
\]
Expanding the absolute value, we obtain the upper and lower bounds for all eigenvalues of the empirical Gram matrix
\[
    1 - \delta \le \lambda_j(\widehat{G}_{\Phi}) \le 1 + \delta,
\]
which means
\[
    1 - \delta \le \sigma_j^2\left(\frac{1}{\sqrt{B}}F_{\mathrm{s}}\right) \le 1 + \delta.
\]
Since singular values are inherently non-negative ($\sigma_j \ge 0$) and the condition $\delta < 1$ ensures that $1-\delta > 0$, we get
\[
    \sqrt{1 - \delta} \le \sigma_j\left(\frac{1}{\sqrt{B}}F_{\mathrm{s}}\right) \le \sqrt{1 + \delta}.
\]
This concludes the proof of the first part: all singular values of $F_{\mathrm{s}}/\sqrt{B}$ are bounded within the interval $[\sqrt{1-\delta}, \sqrt{1+\delta}]$.

The spectral condition number (2-norm condition number) of a matrix $X$ is defined as the ratio of its maximum singular value to its minimum singular value
\[
    \kappa(X) = \frac{\sigma_{\max}(X)}{\sigma_{\min}(X)}.
\]
Importantly, scaling a matrix by a non-zero scalar does not alter its condition number, as both the maximum and minimum singular values are scaled by the exact same proportion. Thus,
\[
    \kappa(F_{\mathrm{s}}) = \kappa\left(\frac{1}{\sqrt{B}}F_{\mathrm{s}}\right).
\]
Based on the singular value bounds, we have:
\begin{itemize}[nosep]
    \item Upper bound for the maximum singular value: $\sigma_{\max}\left(\frac{1}{\sqrt{B}}F_{\mathrm{s}}\right) \le \sqrt{1 + \delta}$
    \item Lower bound for the minimum singular value: $\sigma_{\min}\left(\frac{1}{\sqrt{B}}F_{\mathrm{s}}\right) \ge \sqrt{1 - \delta}$
\end{itemize}
Substituting these bounds into the definition of the condition number yields the desired upper bound
\[
    \kappa(F_{\mathrm{s}}) \le \frac{\sqrt{1 + \delta}}{\sqrt{1 - \delta}} = \sqrt{\frac{1+\delta}{1-\delta}}.
\]
\end{proof}
	\begin{theorem}[Concentration of the sample Gram]
		\label{prop:gram_conc} Assume $\|\Phi(\mathbf{x})\|_{2}\le L$ almost surely.
		Then for any $\delta\in(0,1)$, with probability $\ge 1-\delta$.
		\[
				\big\|\widehat{G}_{\Phi}-G_{\Phi}\big\|_{2} \lesssim L^{2}\Big(\sqrt{\tfrac{\log(M/\delta)}{B}}+\tfrac{\log(M/\delta)}{B}\Big) 
		\]
		In particular, if $G_{\Phi}=I$ (population-orthonormalized features), then
		$\widehat{G}_{\Phi}$ is near-identity with the above rate.
	\end{theorem}
\begin{proof}
{Let $\Phi_{b}:=\Phi(\mathbf{x}_{b})\in\mathbb{R}^{M}$ be the spatial feature vector at the $b$-th collocation point, and define the centered, scaled rank-one summands
\[
    X_{b} \;:=\; \tfrac{1}{B}\big(\Phi_{b}\Phi_{b}^{\top} - G_{\Phi}\big), \qquad b=1,\dots,B.
\]
By construction $\widehat{G}_{\Phi}-G_{\Phi}=\sum_{b=1}^{B}X_{b}$ and $\mathbb{E}\,[X_{b}]=0$, so the deviation $\|\widehat{G}_{\Phi}-G_{\Phi}\|_{2}$ is a sum of $B$ independent, mean-zero, symmetric matrices. The result therefore follows from the matrix Bernstein inequality once we control (i) a uniform almost-sure norm bound $R:=\max_{b}\|X_{b}\|_{2}$ and (ii) a matrix-variance parameter $\sigma^{2}:=\big\|\sum_{b}\mathbb{E}[X_{b}^{2}]\big\|_{2}$. We compute these two ingredients next.}

By the triangle inequality for matrix norms, we obtain
\[
    \|X_{b}\|_{2} \le \frac{1}{B} \big( \|\Phi_{b}\Phi_{b}^{\top}\|_{2} + \|G_{\Phi}\|_{2} \big).
\]
We are given the condition $\|\Phi(\mathbf{x})\|_{2} \le L$ almost surely (a.s.). Since $\Phi_{b}\Phi_{b}^{\top}$ is a rank-one matrix, its unique non-zero eigenvalue (which coincides with its spectral norm) is given by its trace
\[
    \|\Phi_{b}\Phi_{b}^{\top}\|_{2} = \mathrm{Tr}(\Phi_{b}\Phi_{b}^{\top}) = \|\Phi_{b}\|_{2}^{2} \le L^{2}.
\]
By Jensen's inequality and the convexity of matrix norms, the spectral norm of the population covariance matrix is also bounded by
\[
    \|G_{\Phi}\|_{2} = \|\mathbb{E}[\Phi_{b}\Phi_{b}^{\top}]\|_{2} \le \mathbb{E}\big[\|\Phi_{b}\Phi_{b}^{\top}\|_{2}\big] \le L^{2}.
\]
Substituting these two terms back yields a uniform upper bound on the norm
\[
    \|X_{b}\|_{2} \le \frac{1}{B} (L^{2} + L^{2}) = \frac{2L^{2}}{B} := R.
\]

Consider the squared term $X_{b}^{2} = \frac{1}{B^{2}}\big(\Phi_{b}\Phi_{b}^{\top} - G_{\Phi}\big)^{2}$. In terms of the positive semi-definite (PSD) partial ordering ($\preceq$), the second moment of a centered random matrix is always bounded above by the second moment of the uncentered matrix (i.e., $\mathbb{E}[(Y-\mathbb{E}Y)^2] \preceq \mathbb{E}[Y^2]$). Then, we have
\[
    \mathbb{E}[X_{b}^{2}] \preceq \frac{1}{B^{2}}\mathbb{E}\big[(\Phi_{b}\Phi_{b}^{\top})^{2}\big].
\]
Expanding the squared term, we have
\[
    (\Phi_{b}\Phi_{b}^{\top})^{2} = (\Phi_{b}\Phi_{b}^{\top})(\Phi_{b}\Phi_{b}^{\top}) = \Phi_{b}(\Phi_{b}^{\top}\Phi_{b})\Phi_{b}^{\top} = \|\Phi_{b}\|_{2}^{2} (\Phi_{b}\Phi_{b}^{\top}).
\]
Since $\|\Phi_{b}\|_{2}^{2} \le L^{2}$, it follows that
\[
    (\Phi_{b}\Phi_{b}^{\top})^{2} \preceq L^{2}(\Phi_{b}\Phi_{b}^{\top}).
\]
Taking the expectation and summing over all $b=1,\dots,B$, we bound the total variance matrix as
\[
    \sum_{b=1}^{B} \mathbb{E}[X_{b}^{2}] \preceq \sum_{b=1}^{B} \frac{L^{2}}{B^{2}} \mathbb{E}[\Phi_{b}\Phi_{b}^{\top}] = \frac{L^{2}}{B} G_{\Phi}.
\]
Evaluating the spectral norm yields the variance parameter $\sigma^2$,
\[
    \sigma^{2} = \Big\|\sum_{b=1}^{B} \mathbb{E}[X_{b}^{2}]\Big\|_{2} \le \frac{L^{2}}{B} \|G_{\Phi}\|_{2} \le \frac{L^{4}}{B}.
\]

The Matrix Bernstein inequality states that for any $t \ge 0$, we have
\[
    \mathbb{P}\left( \left\| \sum_{b=1}^{B} X_{b} \right\|_{2} \ge t \right) \le 2M \exp\left( - \frac{t^{2}/2}{\sigma^{2} + Rt/3} \right).
\]
Setting the failure probability bound equal to the given confidence level $\delta$, we obtain
\[
    2M \exp\left( - \frac{t^{2}/2}{\sigma^{2} + Rt/3} \right) = \delta \implies \frac{t^{2}/2}{\sigma^{2} + Rt/3} = \log\left(\frac{2M}{\delta}\right).
\]
Solving this quadratic equation for $t$ yields the high-probability upper bound for $t$ given by
\[
    t \le \sqrt{2\sigma^{2}\log\left(\frac{2M}{\delta}\right)} + \frac{2}{3}R\log\left(\frac{2M}{\delta}\right).
\]

Substituting the bounds, specifically $\sigma^{2} \le \frac{L^{4}}{B}$ and $R = \frac{2L^{2}}{B}$, we obtain
\[
    t \le \sqrt{\frac{2L^{4}}{B}\log\left(\frac{2M}{\delta}\right)} + \frac{4L^{2}}{3B}\log\left(\frac{2M}{\delta}\right).
\]
Factoring out $L^{2}$ and subsuming the absolute constants (e.g., $\sqrt{2}$, $4/3$, and the constant offset in the logarithmic term, treating $\log(2M/\delta) \approx \log(M/\delta)$) into the asymptotic $\lesssim$ notation, we arrive at
\[
    t \lesssim L^{2} \left( \sqrt{\frac{\log(M/\delta)}{B}} + \frac{\log(M/\delta)}{B} \right).
\]
This implies that with probability at least $1 - \delta$, the deviation is bounded by
\[
    \|\widehat{G}_{\Phi} - G_{\Phi}\|_{2} \lesssim L^{2} \left( \sqrt{\frac{\log(M/\delta)}{B}} + \frac{\log(M/\delta)}{B} \right).
\]
Notably, when $G_{\Phi} = I$, this result immediately demonstrates that the empirical Gram matrix $\widehat{G}_{\Phi}$ approaches the identity matrix at the derived rate (satisfying the near-identity property), which asymptotically guarantees the approximate orthogonality of the features. 
\end{proof}
	\subsection{CP kernel: algebra, efficiency, and gauge}
	Consider the rank-$q$ CP parameterization
	$C_{ikg}=\sum_{r=1}^{q}A_{ir}B_{kr}D_{gr}$ and the batched forward
	$u=((F_{\mathrm{s}}A)\odot(F_{\mathrm{t}}B)\odot (F_{\mathrm{z}}D))\,w$.

	\begin{theorem}[Parameter and arithmetic complexity]
		\label{prop:cp_complexity} The dense tensor parameterization uses $O(MM_{t}
		G)$ parameters and $O(BMM_{t}G)$ arithmetic per batch, whereas the CP
		kernel uses $O((M{+}M_{t}{+}G)q)$ parameters and
		$O(B\,q\,(M{+}M_{t}{+}G))$ arithmetic.
	\end{theorem}
	{\begin{proof}
	\textit{Parameter count.} The dense parameterization stores every entry $C_{ikg}$ of the 3-way coupling tensor independently, so it uses $MM_{t}G$ scalar parameters, i.e.\ $O(MM_{t}G)$. The CP kernel instead stores three factor matrices of shapes $M{\times}q$, $M_{t}{\times}q$, $G{\times}q$ and a weight vector $w\in\mathbb{R}^{q}$, totalling $q(M+M_{t}+G)+q = O(q(M+M_{t}+G))$ scalar parameters.

	\textit{Arithmetic per batch.} The dense forward $u_{b}=\sum_{i,k,g}C_{ikg}\Phi_{i}(\mathbf{x}_{b})T_{k}(t_{b})\Psi_{g}(\mathbf{Z}_{b})$ touches every entry of $C$ for every sample, costing $\Theta(MM_{t}G)$ FLOPs per sample and $O(BMM_{t}G)$ over a batch of size $B$. The CP forward $u=((F_{\mathrm{s}}A)\odot(F_{\mathrm{t}}B)\odot(F_{\mathrm{z}}D))w$ first multiplies the three feature matrices $F_{\mathrm{s}}\in\mathbb{R}^{B{\times}M}$, $F_{\mathrm{t}}\in\mathbb{R}^{B{\times}M_{t}}$, $F_{\mathrm{z}}\in\mathbb{R}^{B{\times}G}$ by the corresponding factors at cost $O(BqM)+O(BqM_{t})+O(BqG)=O(Bq(M+M_{t}+G))$, then performs an $O(Bq)$ Hadamard fusion and an $O(Bq)$ inner product with $w$, for an overall cost $O(Bq(M+M_{t}+G))$. Hence the CP kernel reduces both parameter count and per-batch arithmetic from multiplicative $O(MM_{t}G)$ to additive $O(q(M+M_{t}+G))$, an asymptotic improvement whenever $q\ll MM_{t}G/(M+M_{t}+G)$.
	\end{proof}}

	\begin{theorem}[Gauge invariance and normalization]
		\label{prop:cp_gauge} For any invertible diagonal $\Lambda_{1},\Lambda_{2}
		,\Lambda_{3}\in\mathbb{R}^{q\times q}$, the mapping
		\[
			(A,B,D,w)\ \mapsto\ \big(A\Lambda_{1},\ B\Lambda_{2},\ D\Lambda_{3}\,
			\Lambda_{1}^{-1}\Lambda_{2}^{-1}\Lambda_{3}^{-1}w\big)
		\]
		leaves $u$ unchanged. Column-wise normalization of $(A,B,D)$ fixes a gauge
		(up to column permutations/signs) and improves Lipschitz conditioning of
		the forward map with respect to each factor.
	\end{theorem}
\begin{proof}
Let $\Lambda_{1}, \Lambda_{2}, \Lambda_{3} \in \mathbb{R}^{q \times q}$ be invertible diagonal matrices. We define the transformed parameters as follows
\begin{align}
    \tilde{A} &= A\Lambda_{1}, \\
    \tilde{B} &= B\Lambda_{2}, \\
    \tilde{D} &= D\Lambda_{3}, \\
    \tilde{w} &= \Lambda_{1}^{-1}\Lambda_{2}^{-1}\Lambda_{3}^{-1}w.
\end{align}
The computation for the forward propagation is formulated as
\[
    \tilde{u} = \big((F_{\mathrm{s}}\tilde{A}) \odot (F_{\mathrm{t}}\tilde{B}) \odot (F_{\mathrm{z}}\tilde{D})\big) \tilde{w}.
\]

First, we consider the transformation of the feature projection matrices. Since $F_{\mathrm{s}}$ is post-multiplied by $\tilde{A}$ and $\Lambda_{1}$ is a diagonal matrix, this operation is equivalent to scaling the columns of $F_{\mathrm{s}}A$, which means
\[
    F_{\mathrm{s}}\tilde{A} = F_{\mathrm{s}}(A\Lambda_{1}) = (F_{\mathrm{s}}A)\Lambda_{1}.
\]
Similarly, we have
\begin{align}
    F_{\mathrm{t}}\tilde{B} &= (F_{\mathrm{t}}B)\Lambda_{2}, \\
    F_{\mathrm{z}}\tilde{D} &= (F_{\mathrm{z}}D)\Lambda_{3}.
\end{align}

Next, we examine the Hadamard product (denoted by $\odot$). For arbitrary matrices $X, Y$ and diagonal matrices $\Lambda_X, \Lambda_Y$, their Hadamard product satisfies
\[
    (X\Lambda_X) \odot (Y\Lambda_Y) = (X \odot Y)(\Lambda_X\Lambda_Y).
\]
This holds because post-multiplying by a diagonal matrix essentially scales the columns of the target matrix, and the scaling factors for the $r$-th column are directly multiplied together. Expanding the Hadamard product of the three transformed matrices, we obtain
\[
    \big((F_{\mathrm{s}}A)\Lambda_{1}\big) \odot \big((F_{\mathrm{t}}B)\Lambda_{2}\big) \odot \big((F_{\mathrm{z}}D)\Lambda_{3}\big) = \big((F_{\mathrm{s}}A) \odot (F_{\mathrm{t}}B) \odot (F_{\mathrm{z}}D)\big) (\Lambda_{1}\Lambda_{2}\Lambda_{3}).
\]

Finally, we multiply this result by the transformed weight vector $\tilde{w}$ and we get
\[
    \tilde{u} = \big((F_{\mathrm{s}}A) \odot (F_{\mathrm{t}}B) \odot (F_{\mathrm{z}}D)\big) (\Lambda_{1}\Lambda_{2}\Lambda_{3}) \big(\Lambda_{1}^{-1}\Lambda_{2}^{-1}\Lambda_{3}^{-1}w\big).
\]
Since multiplication between diagonal matrices is commutative, and $\Lambda_i \Lambda_i^{-1} = I$ (the identity matrix), all the diagonal matrices cancel out, yielding
\[
    (\Lambda_{1}\Lambda_{2}\Lambda_{3}) (\Lambda_{1}^{-1}\Lambda_{2}^{-1}\Lambda_{3}^{-1}) = I.
\]
Therefore, 
\[
    \tilde{u} = \big((F_{\mathrm{s}}A) \odot (F_{\mathrm{t}}B) \odot (F_{\mathrm{z}}D)\big) w = u.
\]
This concludes the proof that the forward propagation mapping is strictly gauge invariant under the aforementioned transformations.

As demonstrated in the preceding proof, for any non-zero constant $\lambda$, we can scale the $r$-th column of the factor matrix $A$ by $\lambda$ while simultaneously dividing the $r$-th element of the weight $w$ by $\lambda$ without altering the network output. This symmetry induces infinitely many equivalent global optima in the parameter space (i.e., continuous scaling valleys). 

To eliminate this continuous ambiguity, we apply column normalization to $A, B$, and $D$ (typically by constraining the $L_2$ norm of each column to 1). Enforcing the constraint $\|A_{:, r}\|_2 = 1$ strictly limits the permissible diagonal scaling factors to $|\lambda| = 1$. This implies that the diagonal entries of $\Lambda_{1}, \Lambda_{2}$, and $\Lambda_{3}$ are restricted to $\pm 1$. Consequently, the continuous scaling ambiguity collapses into discrete sign permutations. Combined with the inherent column permutation symmetry of the rank-1 components in CP decomposition, we successfully fix the gauge up to permutations and signs.

In neural network optimization, the smoothness of the loss function with respect to a specific parameter (e.g., matrix $A$) is governed by the Lipschitz constant of its partial derivative. In the multilinear CP format, the partial derivative of $u$ with respect to $A$ heavily depends on the magnitudes of $B, D$, and $w$. Roughly speaking, the gradient norm is proportional to
\[
    \left\| \frac{\partial u}{\partial A} \right\| \propto \|B\| \cdot \|D\| \cdot \|w\|.
\]

Without normalization constraints, the optimizer might drift into ill-conditioned regions where certain columns of $B$ and $D$ become pathologically large (e.g., $10^6$) while $w$ becomes vanishingly small (e.g., $10^{-12}$) to compensate. Although this leaves $u$ invariant, it causes the gradient with respect to $A$ to explode, driving the Lipschitz constant toward infinity and severely destabilizing the optimization process (Gradient Explosion). 

By enforcing column normalization on $A, B$, and $D$, their norms are rigidly bounded near 1. The network is thus compelled to absorb all global magnitude information into the unique linear weight vector $w$. This provides well-bounded control over the Jacobians of the factor matrices, significantly improving the Lipschitz conditioning of the forward mapping with respect to each factor. As a result, the optimization process (such as gradient descent) becomes substantially more stable and efficient.
\end{proof}
        \subsection{Closed-form derivatives for Gaussian atoms and stability under
	warps}
	For
	$\varphi(\mathbf{x})=\exp\!\big(-\tfrac12(\mathbf{x}-\mu)^{\top}\Sigma^{-1}(
	\mathbf{x}-\mu)\big)$
	with $\Sigma\succ 0$,
	\[
		\begin{aligned}
			\nabla\varphi(\mathbf{x}) & =-\Sigma^{-1}(\mathbf{x}-\mu)\,\varphi(\mathbf{x}),                                                          \\
			\Delta\varphi(\mathbf{x}) & =\big((\mathbf{x}-\mu)^{\top}\Sigma^{-2}(\mathbf{x}-\mu)-\mathrm{tr}\,\Sigma^{-1}\big)\,\varphi(\mathbf{x}),
		\end{aligned}
	\]
	and for a $\mathcal{C}^{1}$ residual warp $\mathbf{W}(\mathbf{x})=\mathbf{x}+
	\mathbf{w}_{\theta}(\mathbf{x})$, derivatives of
	$\varphi\!\circ\!\mathbf{W}$ follow by the chain rule with Jacobian
	$J_{\mathbf{W}}(\mathbf{x})$. If $\|J_{\mathbf{W}}-I\|_{2}\le \varepsilon<1$,
	the mapping $\Phi_{i}\mapsto \Phi_{i}\circ \mathbf{W}$ is bi-Lipschitz on
	$L^{2}(\Omega)$ with constants depending only on $\varepsilon$ and $\Omega$.

	{The three statements above are direct calculus consequences and are recorded here for completeness.

	\textit{(i) Gradient and Laplacian of a Gaussian atom.} Write $\varphi(\mathbf{x})=\exp(-\tfrac12 q(\mathbf{x}))$ with $q(\mathbf{x})=(\mathbf{x}-\mu)^{\top}\Sigma^{-1}(\mathbf{x}-\mu)$. Since $\Sigma$ is symmetric, $\nabla q(\mathbf{x})=2\Sigma^{-1}(\mathbf{x}-\mu)$, and the chain rule gives
	\[
	    \nabla\varphi(\mathbf{x}) = -\tfrac12\,\nabla q(\mathbf{x})\,\varphi(\mathbf{x}) = -\Sigma^{-1}(\mathbf{x}-\mu)\,\varphi(\mathbf{x}).
	\]
	Taking divergences and using the product rule $\nabla\!\cdot\!(f\mathbf{v})=\mathbf{v}^{\top}\nabla f + f\,\nabla\!\cdot\!\mathbf{v}$ with $\mathbf{v}=\Sigma^{-1}(\mathbf{x}-\mu)$ (so $\nabla\!\cdot\!\mathbf{v}=\mathrm{tr}\,\Sigma^{-1}$),
	\[
	    \Delta\varphi(\mathbf{x}) = \nabla\!\cdot\!(\nabla\varphi)
	    = \big((\mathbf{x}-\mu)^{\top}\Sigma^{-2}(\mathbf{x}-\mu) - \mathrm{tr}\,\Sigma^{-1}\big)\varphi(\mathbf{x}),
	\]
	matching the displayed formula.

	\textit{(ii) Composition with a warp.} For any $\mathcal{C}^{1}$ map $\mathbf{W}$, the multivariate chain rule gives $\nabla(\varphi\!\circ\!\mathbf{W})(\mathbf{x})=J_{\mathbf{W}}(\mathbf{x})^{\top}(\nabla\varphi)(\mathbf{W}(\mathbf{x}))$, and the Laplacian follows from $\Delta(\varphi\!\circ\!\mathbf{W})=\mathrm{tr}\!\big(J_{\mathbf{W}}^{\top}(\mathrm{Hess}\,\varphi)(\mathbf{W})\,J_{\mathbf{W}}\big)+(\nabla\varphi)(\mathbf{W})^{\top}\,\Delta\mathbf{W}$. For a residual warp $\mathbf{W}(\mathbf{x})=\mathbf{x}+\mathbf{w}_{\theta}(\mathbf{x})$, $J_{\mathbf{W}}=I+J_{\mathbf{w}_{\theta}}$ and all quantities are computable in closed form once $\mathbf{w}_{\theta}$ and its first/second derivatives are available from the warp MLP.

	\textit{(iii) Bi-Lipschitz stability under small Jacobian perturbation.} Suppose $\|J_{\mathbf{W}}(\mathbf{x})-I\|_{2}\le\varepsilon<1$ uniformly on $\Omega$. By the Neumann series, $J_{\mathbf{W}}$ is invertible everywhere with $\|J_{\mathbf{W}}^{-1}\|_{2}\le 1/(1-\varepsilon)$ and $\|J_{\mathbf{W}}\|_{2}\le 1+\varepsilon$. The mean-value theorem on the convex domain $\Omega$ then yields the two-sided bound
	\[
	    (1-\varepsilon)\,\|\mathbf{x}-\mathbf{x}'\|_{2}
	    \;\le\; \|\mathbf{W}(\mathbf{x})-\mathbf{W}(\mathbf{x}')\|_{2}
	    \;\le\; (1+\varepsilon)\,\|\mathbf{x}-\mathbf{x}'\|_{2},
	\]
	so $\mathbf{W}$ is bi-Lipschitz with constants $(1\pm\varepsilon)$, and a standard change-of-variables argument shows that $\Phi_{i}\mapsto \Phi_{i}\!\circ\!\mathbf{W}$ is a bi-Lipschitz map on $L^{2}(\Omega)$ with constants depending only on $\varepsilon$ (through the Jacobian determinant $\det J_{\mathbf{W}}\in[(1-\varepsilon)^{d},(1+\varepsilon)^{d}]$) and $\Omega$. In particular, near-orthogonality of the unwarped dictionary is preserved up to a factor controlled by $\varepsilon$, which is why the warp does not destabilize the Gram conditioning.}

	\section{Manufactured Solution Definitions}
	\label{sec:manufactured-solutions}

	All four benchmark problems are defined on the spatio-temporal domain
	$\Omega = [-1,1]^{2}$, $\mathcal{T}=[0,1]$, with $R$ i.i.d.\ standard-normal
	random variables $\mathbf{Z}=(Z_{1},\dots,Z_{R})\sim\mathcal{N}(0,I_{R})$.
	These four equations are standard benchmarks in the high-dimensional PDE
	literature~\cite{fan2026physics}.
	\paragraph{Karhunen--Lo\`eve coefficient.}
	All problems share a log-normal random field:
	\[
		k(\mathbf{x},\mathbf{Z}) = \exp\!\Bigl(\mu_{k}+ \sigma
		\sum_{i=1}^{R}\sqrt{\lambda_{i}}\, Z_{i}\, \phi_{i}(\mathbf{x})\Bigr),
	\]
	where $\{\lambda_{i},\phi_{i}\}$ are the eigenvalue--eigenfunction pairs of
	an exponential covariance kernel
	$C(\mathbf{x},\mathbf{x}')=\exp(-\|\mathbf{x}-\mathbf{x}'\|_{1}/\ell)$ with
	correlation length $\ell=1$, $\mu_{k}=0$, and $\sigma=1$.




	\paragraph{1.\ Stochastic Diffusion.}
	\[
		\frac{\partial u}{\partial t}- \nabla\!\cdot
		\!\bigl(k(\mathbf{x},\mathbf{Z})\,\nabla u\bigr) = s(\mathbf{x},t,\mathbf{Z}
		),
	\]
	with manufactured exact solution
	\[
		u^{*}(\mathbf{x},t,\mathbf{Z}) = \sin(\pi x)\,\sin(\pi y)\,e^{-t}\,\Bigl(
		1 + 0.1\textstyle\sum_{i=1}^{R}Z_{i}\Bigr),
	\]
	where the source $s$ is obtained by substituting $u^{*}$ into the PDE.
	Dirichlet boundary conditions $u\big|_{\partial\Omega}=0$ and initial
	condition $u(0,\mathbf{x},\mathbf{Z})=u^{*}(0,\mathbf{x},\mathbf{Z})$ are
	imposed.

	\paragraph{2.\ Stochastic Allen--Cahn.}
	\[
		\frac{\partial u}{\partial t}- \nabla\!\cdot\!\bigl(k\,\nabla u\bigr) - (
		u - u^{3}) = s,
	\]
	with the \emph{same} manufactured solution as diffusion. The cubic reaction term
	$u-u^{3}$ introduces additional nonlinearity.

	\paragraph{3.\ Stochastic Burgers.}
	\[
		\frac{\partial u}{\partial t}+ u\!\left(\frac{\partial u}{\partial x}+\frac{\partial
		u}{\partial y}\right) - \nu\,\Delta u = s,
	\]
	with kinematic viscosity $\nu = 0.01$ and the same manufactured solution as
	diffusion. The convective nonlinearity $u\,\nabla u$ makes this the most challenging
	benchmark for standard PINNs.

	\paragraph{4.\ Stochastic Darcy (divergence form).}
	\[
		\frac{\partial u}{\partial t}- \nabla\!\cdot\!\bigl(k\,\nabla u\bigr) = s
		,
	\]
	with a different manufactured solution that couples the solution to the
	stochastic field more tightly:
	\[
		\displaystyle u^{*}(\mathbf{x},t,\mathbf{Z}
		) = \sin(\pi x)\sin(\pi y)\,e^{-t}\exp\!\Bigl(\gamma\!\sum_{i=1}^{R}\!\sqrt{\lambda_{i}}
		\,Z_{i}\,\phi_{i}(\mathbf{x})\Bigr),
	\]
	where $\gamma=0.3$. Unlike the linear $Z$-dependence in the first three
	problems, the Darcy solution depends exponentially on $\mathbf{Z}$ through the
	KL expansion, requiring higher gPC orders for accurate recovery.

	\paragraph{5.\ Non-manufactured Stochastic Poisson (forward simulation).}
	We solve a stochastic Poisson equation with a random source term:
	\[
		- \Delta u = f(x,y,\mathbf{Z}), \quad \text{with }u=0 \text{ on }\partial
		\Omega,
	\]
	where
	$f(x,y,\mathbf{Z}) = 1 + \sigma \sum_{i=1}^{R}\sqrt{\lambda_{i}}\, Z_{i}\, \phi
	_{i}(\mathbf{x})$
	is a KL-expanded random source field ($\sigma{=}0.5$, $R{=}5$). The
	reference solution is computed numerically via finite differences on a
	$64{\times}64$ grid over $10{,}000$ Monte Carlo samples.

	\paragraph{6.\ Non-manufactured Stochastic Darcy (forward simulation).}
	We also solve a stochastic Darcy flow with a log-normal permeability field:
	\[
		- \nabla\!\cdot\!\bigl(k(x,y,\mathbf{Z})\,\nabla u\bigr) = 1.0, \quad \text{with
		}u=0 \text{ on }\partial\Omega,
	\]
	where
	$k(x,y,\mathbf{Z}) = \exp\!\bigl(\sigma \sum_{i=1}^{R}\sqrt{\lambda_{i}}\, Z_{i}
	\, \phi_{i}(\mathbf{x})\bigr)$
	with $\sigma{=}1.0$ and $R{=}5$. The physical response $u(x,y,\mathbf{Z})$ is
	purely driven by the stochastic permeability field $k$ and the constant source
	term. The reference solution is computed via the same finite-difference
	Monte Carlo procedure ($N{=}10{,}000$, $64{\times}64$ grid).

	\section{Details of Baseline Methods}
	\label{sec:baseline-details}
	In this section, we record the exact baseline protocols used in the main
	experiments. The primary comparisons include six neural or classical baselines
	(Vanilla PINN, Vanilla PINN+proj, PI-DeepONet, MC-Dropout PINN, Deep
	Ensembles, stochastic collocation, \add{ SPINN+Z, functional CP-PINN, and Neural Chaos),} and PC$^{2}$ is another 
	physics-informed polynomial-chaos baseline for the selected forward
	benchmarks.

	{\paragraph{Common training configuration (all neural baselines).}
	Unless otherwise noted, every neural baseline shares the following implementation choices, which match those of S$^{2}$-PINN to ensure a fair comparison:
	\begin{itemize}[itemsep=2pt, leftmargin=*]
	    \item \textit{Optimizer.} Adam with $(\beta_1,\beta_2)=(0.9,0.999)$, $\epsilon=10^{-8}$, no learning-rate scheduler\add{ (except for the tuned PI-DeepONet, see below)}, and no gradient clipping. The base learning rate is $10^{-3}$ for diffusion / Allen--Cahn / Burgers and $5\!\times\!10^{-4}$ for Darcy.
	    \item \textit{Initialization.} All \texttt{nn.Linear} layers use Xavier-normal weights and zero biases.
	    \item \textit{Loss weights.} For diffusion / Allen--Cahn / Burgers, $w_{\mathrm{pde}}=w_{\mathrm{bc}}=w_{\mathrm{ic}}=1.0$; for Darcy, $w_{\mathrm{bc}}=w_{\mathrm{ic}}=10.0$ to compensate for the steeper boundary layers. The projection weight $w_{\mathrm{proj}}$ is $1.0$ for ``$+$proj'' variants and $0$ otherwise (see per-baseline notes below).
	    \item \textit{Training horizon.} 20{,}000 Adam steps for diffusion / Allen--Cahn / Burgers and 30{,}000 for Darcy. 
	    \item \textit{Sampling strategy.} All collocation, boundary, and initial-condition batches are resampled \emph{every} step (\texttt{resample\_every}\,$=\!1$). Collocation points are drawn uniformly at random in $\mathcal{T}\!\times\!\Omega$; boundary points draw $t\sim\mathcal{U}(\mathcal{T})$, the free spatial coordinate $\sim\mathcal{U}([-1,1])$, and the boundary edge index $\sim\mathrm{Uniform}\{1,2,3,4\}$; initial-condition points fix $t=t_{\min}$ with uniform $(x,y)$. Random inputs are drawn from $\mathcal{N}(0,I_R)$ for Gaussian configurations and from $\mathcal{U}([-1,1]^R)$ for uniform configurations.
	\end{itemize}}

	\paragraph{Vanilla PINN.}
	The standard physics-informed neural network represents the solution $u(t,\mathbf{x}
	,\mathbf{Z})$ as a fully connected Multi-Layer Perceptron (MLP) taking the concatenated
	space, time, and stochastic variables as inputs. We use a 4-layer MLP with
	256 hidden units per layer and hyperbolic tangent ($\tanh$) activations. {Concretely, the architecture is $(3{+}R)\!\to\!256\!\to\!256\!\to\!256\!\to\!256\!\to\!1$ with $\tanh$ on every hidden layer and a linear output; for the main stochastic runs $R{=}5$, giving an input dimension of $8$.} The
	loss function is defined as 
    \begin{equation}
        \mathcal{L}_{\text{PINN}}= w_{\text{pde}}
	\mathbb{E}[r^{2}] + w_{\text{bc}}\mathbb{E}_{\partial\Omega}[(u-u^{*})^{2}
	] + w_{\text{ic}}\mathbb{E}[(u-u^{*})^{2}|_{t=0}].
    \end{equation}
The network relies
	purely on dense sampling and gradient descent to learn the complex high-dimensional
	stochastic dependence without any explicit polynomial basis.

	\paragraph{Vanilla PINN + proj.}
	This baseline uses the exact same MLP architecture as the Vanilla PINN but augments
	the standard loss function with the gPC Galerkin projection term introduced in
	our method:
	$\mathcal{L}_{\text{PINN+proj}}= \mathcal{L}_{\text{PINN}}+ w_{\text{proj}}\sum
	_{g=1}^{G}\big(\mathbb{E}[r\Psi_{g}(\mathbf{Z})]\big)^{2}${, with $w_{\mathrm{proj}}{=}1.0$ matching the S$^{2}$-PINN setting}. This isolates the
	effect of the physics objective from the separable tensor architecture.

	\paragraph{PI-DeepONet.}
	Physics-Informed DeepONet~\cite{goswami2023physics} utilizes a dual-network
	architecture consisting of a Branch net and a Trunk net. The Branch net takes
	the stochastic parameters $\mathbf{Z}$ and outputs $q$ latent coefficients.
	The Trunk net processes the spatiotemporal coordinates $(t,\mathbf{x})$ and
	outputs $q$ basis functions. The final solution is computed as the inner product
	of these two latent vectors:
	$u(t,\mathbf{x},\mathbf{Z}) = \sum_{k=1}^{q}b_{k}(\mathbf{Z}) \tau_{k}(t,\mathbf{x}
	) {\,+\,b_{0}}${, where $b_{0}$ is a single learnable scalar bias added to the inner product. Concretely, both branches use $3$ hidden layers of width $128$ with $\tanh$ activations and a linear projection to a $q{=}128$-dimensional latent: Branch is $R\!\to\!128\!\to\!128\!\to\!128\!\to\!128$ (so input dim $5$ for $R{=}5$) and Trunk is $3\!\to\!128\!\to\!128\!\to\!128\!\to\!128$ (taking $(t,x,y)$). Weights are Xavier-normal initialised with zero biases}. It is trained using the standard
	$\mathcal{L}_{\text{PINN}}$ loss. \add{Because the default configuration diverged on Allen--Cahn, the reported results use a tuned setting: learning rate $10^{-3}$, and on Allen--Cahn and Darcy a \texttt{ReduceLROnPlateau} scheduler with factor $0.5$, patience $3$, and minimum learning rate $10^{-6}$.}
	\paragraph{Derivative-informed operator training.}
	Rather than constraining the residual, derivative-informed objectives supervise
	the input--output sensitivities of a neural operator, which matters when the
	surrogate is differentiated inside inverse problems or PDE-constrained
	optimization. Yang et al.~\cite{yang2026derivativeinformedtrainingneuraloperators}
	enforce this consistency entirely on-the-fly: a sketched tangent consistency
	loss derives the sensitivity constraint directly from the governing equation,
	so no offline tangent labels are required and the operator architecture is
	unchanged. We do not include it as a direct baseline because it targets
	tangent-level supervision of a \emph{deterministic} operator and carries no
	explicit stochastic basis, so it does not produce calibrated moments for random
	PDEs.
	\paragraph{MC-Dropout PINN.}
	To provide a strong probabilistic baseline for uncertainty quantification,
	we augment the Vanilla PINN with Dropout layers (with drop probability $p=0.1$)
	after every hidden layer. {Concretely, a dropout layer is inserted after the activation of every hidden layer (but not after the linear output), so the architecture is $\big[\text{Linear}\to\tanh\to\mathrm{Dropout}(p)\big]^{4}\to\text{Linear}$ with the same Xavier-normal initialization as the Vanilla PINN.} During training, dropout acts as a regularizer. During
	inference, dropout remains active, and we sample 50 forward passes for each input
	collocation point. This yields a predictive distribution that estimates the
	epistemic uncertainty of the network, which is then combined with the aleatoric
	variation over $\mathbf{Z}$ to compute coverage and calibration metrics. {The training loss is $\mathcal{L}_{\text{PINN}}$.}

	\paragraph{Deep Ensembles.}
	Following Lakshminarayanan et al.~\cite{DBLP:conf/nips/Lakshminarayanan17}, we
	train an ensemble of $K=5$ independent Vanilla PINNs, each initialized with a
	different random seed. {Each member uses the same Vanilla-PINN architecture (4 hidden layers $\times$ 256, $\tanh$) and the same 20{,}000-step training schedule, but draws its own collocation samples and weight initialization from the seed. Members are trained completely independently and then aggregated post-hoc: the prediction is the per-point mean and the epistemic uncertainty is the per-point standard deviation across the $K$ members. Uncertainty-quantification metrics are computed with $1024$ stochastic samples each for moment estimation, interval coverage, and held-out test points (\texttt{n\_z\_moment}\,$=\!$\,\texttt{n\_z\_interval}\,$=\!$\,\texttt{n\_z\_test}\,$=\!1024$).} Deep ensembles are widely considered a gold-standard
	baseline for predictive uncertainty. {Unlike the other no-projection baselines, the Deep-Ensembles launch script does not override $w_{\mathrm{proj}}$, so each member inherits $w_{\mathrm{proj}}{=}1.0$ from the base configuration. This is intentional, since the strongest possible Deep-Ensembles baseline includes the same projection regularizer used by S$^{2}$-PINN.}

	\paragraph{Stochastic Collocation (SC).}
	SC is a classical non-neural spectral method that leverages tensorized {Gauss quadrature} nodes in the stochastic space.
    For a quadrature order $n_{q}$, we define a grid of $n_{q}^{R}$ nodes in $\mathbb{R}^{R}$. We then train a
	completely independent, deterministic spatial PINN $u_{i}(t,\mathbf{x})$ for
	each node $\mathbf{Z}_{i}$. The statistical moments of the random field are
	reconstructed post-training using the corresponding quadrature weights{: $\bar{u}(t,\mathbf{x})=\sum_{i}W_{i}\,u_{i}(t,\mathbf{x})$ and $\mathrm{Var}(u)(t,\mathbf{x})=\max\!\big(0,\,\sum_{i}W_{i}\,u_{i}(t,\mathbf{x})^{2}-\bar{u}(t,\mathbf{x})^{2}\big)$, where $\{W_{i}\}$ are the tensor-product quadrature weights and the clamp at zero guards against round-off-induced negative variance}. While
	highly accurate for smooth problems, the number of required PINNs scales
	exponentially with $R$, leading to severe computational bottlenecks (e.g.,
	$n_{q}=3$ for $R =5$ requires 243 independent training runs).
	{The reported tractable SC baseline uses $n_{q}{=}2$ per stochastic dimension, which gives $2^{5}{=}32$ tensor-product nodes for $R{=}5$. Each node is a Vanilla-PINN-style spatial network trained for $5{,}000$ Adam steps on diffusion / Allen--Cahn / Burgers and $8{,}000$ steps on Darcy. 
    Reference Monte-Carlo statistics for SC's relative-$L^{2}$ metrics are computed with $4{,}096$ stochastic samples and $2{,}048$ spatial probes ($n_{\mathrm{mc,ref}}{=}4096$, $n_{\mathrm{probes}}{=}2048$).}
    {\paragraph{Physics-informed polynomial chaos (PC$^{2}$).}
We implement the PC$^{2}$ baseline of Nov\'ak et al.~\cite{novak2024physics}
using UQpy's physics-informed PCE stack. Dirichlet data are used as boundary
regression data, while PDE residual constraints are enforced at virtual
collocation points through the KKT system in \texttt{ConstrainedPCE}. To avoid
an unfair capacity mismatch, we do not use UQpy's native total-degree basis over
all variables as the main comparison; instead, we use an anisotropic product
basis with stochastic total degree fixed at $p_Z{=}3$ and deterministic degree
$p_{\mathrm{det}}$ swept over $\{3,4,5\}$ for diffusion and $\{4,6,8\}$ for
Poisson/Darcy. Model selection uses only held-out PDE/boundary residuals. This
keeps the stochastic gPC capacity aligned with S$^{2}$-PINN while giving PC$^{2}$
a reasonable deterministic polynomial space.}

	\begin{addblock}
	\paragraph{SPINN+Z.}
	A separable PINN~\cite{cho2024separable} in which the random variables $Z_1,\dots,Z_R$ are treated as additional separable input coordinates, each processed by its own coordinate network, with the per-coordinate features combined by the SPINN outer-product/summation rule. It is trained with $\mathcal{L}_{\text{PINN}}$ and contains no explicit stochastic basis or projection term.

	\paragraph{Functional CP-PINN.}
	The functional tensor decomposition PINN of Vemuri et al.~\cite{vemuri2024functional}, in which each coordinate (space, time, and random variables) is assigned a separate sub-network and the outputs are coupled by a CP outer product. It is trained with $\mathcal{L}_{\text{PINN}}$.

	\paragraph{Neural Chaos.}
	Neural Chaos~\cite{bahmani2025neural} replaces the fixed polynomial chaos basis with sequentially learned neural basis functions. As a data-driven operator, it is trained on reference solution realizations and therefore acts as a supervised oracle rather than a purely physics-informed competitor; a physics residual is only evaluated post hoc. Since diffusion, Allen--Cahn, and Burgers share the same manufactured solution field, a single shared checkpoint is used for these three equations, and only the equation-specific post-hoc residual differs.

    	\paragraph{Architectural PINN variants.}
	Two recent PINN architectures improve the base network rather than the
	stochastic treatment. compleX-PINN~\cite{si2026complex} replaces the activation
	with a learnable form derived from Cauchy's integral theorem, obtaining high
	accuracy with a single hidden layer and reporting gains on high-dimensional
	problems. IDPINN~\cite{si2025initialization} decomposes the domain into
	subdomains with dedicated networks and initializes training from a
	small-data pretrained model, improving interface smoothness and convergence
	speed. We do not include either as a baseline because both are
	\emph{deterministic} architecture or optimization devices: they carry no
	explicit stochastic basis and therefore produce no calibrated moments or
	variance estimates for random PDEs. Their improvements are orthogonal to ours
	in principle and could be combined with the separable gPC representation
	introduced here.
	\end{addblock}



    \section{Model Architectures, Parameter Counts, and Training Configuration}
	\label{sec:model-details}
	\Cref{tab:model-params} summarises the architecture and learnable
	parameter count for the main neural baselines and stochastic collocation.
	PC$^{2}$ is excluded from this parameter table because its capacity is better
	described by polynomial basis size rather than neural weights. S$^{2}$-PINN
	uses roughly $30\times$ fewer neural parameters than a single Vanilla PINN
	because the CP kernel couples three compact, purpose-built feature extractors
	instead of a monolithic MLP.
    
	\begin{table}[htbp]
		\centering
		\caption{Architecture details and parameter counts for all compared
		methods on the $R{=}5$ stochastic diffusion benchmark. $M$: spatial
		atoms; $M_{t}$: temporal features; $G$: gPC basis size; $q$: CP rank; $h$:
		hidden width; $L$: number of hidden layers.}
		\label{tab:model-params} \resizebox{\columnwidth}{!}{%
		\begin{tabular}{llcr}
			\toprule Method                & Architecture                                                                                                                                                                                & Training steps & \#Params    \\
			\midrule \textbf{S$^{2}$-PINN} & \makecell[l]{SpatialDict2D ($M{=}256$, warp MLP $2{\to}64{\to}2$) \\ + TimeBasis ($K{=}8 \Rightarrow M_{t}{=}17$) \\ + GPCBasis ($R{=}5, p{=}3 \Rightarrow G{=}56$) \\ + CPCore ($q{=}16$)} & 20\,000        & 7\,018      \\
			\midrule Vanilla PINN          & MLP $8{\to}256{\to}256{\to}256{\to}256{\to}1$, tanh                                                                                                                                         & 20\,000        & 199\,937    \\
			Vanilla PINN+proj              & same + gPC projection loss                                                                                                                                                                  & 20\,000        & 199\,937    \\
			MC-Dropout PINN                & same + dropout $p{=}0.1$                                                                                                                                                                    & 20\,000        & 199\,937    \\
			PI-DeepONet                    & \makecell[l]{Branch: $5{\to}128{\to}128{\to}128{\to}128$ \\ Trunk: $3{\to}128{\to}128{\to}128{\to}128$}                                                                                     & 20\,000        & 101\,633    \\
			SC ($n_{q}{=}2$, 32 nodes)     & 32 $\times$ Vanilla PINN                                                                                                                                                                    & 5\,000 each    & 6\,397\,504 \\
			\bottomrule
		\end{tabular}%
		}
	\end{table}
 All experiments are based on NVIDIA T4. The S$^{2}$-PINN
	parameter count decomposes as
	\begin{itemize}
		\setlength{\itemsep}{0pt}

		\item \textit{Spatial dictionary}: 256 centres ($\mu \in \mathbb{R}^{256\times
			2}$), 256 log-bandwidths, warp MLP ($2{\times}64{+}64{+}64{\times}2{+}
			2 = 194$
			params) $\Rightarrow$ 962.

		\item \textit{Temporal basis}: 8 learnable log-frequencies $\Rightarrow$
			8.

		\item \textit{gPC basis}: fixed Hermite polynomials (no learnable params).

		\item \textit{CP core}: $A\!\in\!\mathbb{R}^{256\times 16},\, B\!\in\!\mathbb{R}
			^{17\times 16},\, D\!\in\!\mathbb{R}^{56\times 16},\, w\!\in\!\mathbb{R}
			^{16}$ $\Rightarrow$ 5\,280.

		\item \textit{Warp MLP}: $2{\times}64{+}64{+}64{\times}2{+}2 = 194$ $\Rightarrow$
			included above. Biases account for the remainder up to 7\,018.
	\end{itemize}

	{\paragraph{Training algorithm and default hyperparameters.}
	The pseudocode~\Cref{alg:training} summarizes the training loop.}
	\begin{algorithm}
		[H]
		\caption{Training loop}
		\label{alg:training}
		\begin{algorithmic}
			[1] \State Initialize dictionary/warp parameters, CP factors
			$(A,B,D,w)$ \For{$s=1$ to $S$} \If{$s \bmod S_{\text{rs}}=1$} \State
			Resample PDE / BC / IC / Gram batches \EndIf \State Build features
			$\mathbf{F}_{\text{space}},\mathbf{F}_{\text{time}},\mathbf{F}_{\text{stoc}}$
			\State Compute $u$, $\partial_{t}u$, and $\Delta u$ (analytic time,
			autodiff space) \State Evaluate residuals $r$ and all loss terms in \eqref{eq:loss}
			\State Adam step
			\State Column-normalize CP factors; update moving stats for Gram if
			used \EndFor
		\end{algorithmic}
	\end{algorithm}

	\paragraph{Default Hyperparameters.}
	For the four benchmark problems in Tables~\ref{tab:mean-accuracy}--\ref{tab:var-accuracy},
	we use
	\[
    M=256,\quad M_{t}=17,\quad (R,p)=(5,3)\Rightarrow
		G=\tbinom{8}{3}=56,\quad q=16,
	\]
	loss weights
	\[
		\begin{aligned}
			w_{\mathrm{pde}} =w_{\mathrm{proj}}=w_{\mathrm{bc}}=w_{\mathrm{ic}}=
			w_{\mathrm{orth}} =1.0,\, \lambda_{\mathrm{fac}}=0.1,
		\end{aligned}
	\]
	gradient accumulation $\kappa=1$, resampling period $S_{\text{rs}}=1$, and collocation
	batch sizes of $32{,}768$ for PDE, boundary, and initial conditions.

	\section{Ablation Study Details}
	\label{sec:ablation-details}

	We conducted controlled ablations on the stochastic diffusion benchmark to
	isolate the contribution of each hyperparameter. All experiments report errors
	at final time $t{=}1.0$. Tables~\ref{tab:ablation_baseline_arch}--\ref{tab:ablation_cp_rank}
	present the full results for baseline architecture, spatial atoms, gPC order, and CP rank, respectively. 

	\begin{table}[htbp]
		\centering
		\caption{Baseline Architecture Ablation (hidden size × layers)}
		\label{tab:ablation_baseline_arch} \resizebox{\columnwidth}{!}{%
		\begin{tabular}{lccc}
			\toprule Configuration               & $\mathcal{E}_{\bar{u}}$                   & $\mathcal{E}_{u_1}$                       & $\mathcal{E}_{u_2}$                       \\
			\midrule \textit{default (baseline)} & $(4.23 \pm 1.03) \times 10^{-3}$          & $(3.97 \pm 1.39) \times 10^{-3}$          & $(4.25 \pm 1.07) \times 10^{-3}$          \\
			$h=128, L=4 $                          & $(4.41 \pm 0.63) \times 10^{-2}$          & $(4.10 \pm 0.34) \times 10^{-2}$          & $(3.86 \pm 0.64) \times 10^{-2}$          \\
			$h=128, L=6$                           & \textbf{$(2.33 \pm 0.44) \times 10^{-2}$} & \textbf{$(2.12 \pm 0.44) \times 10^{-2}$} & \textbf{$(2.26 \pm 0.38) \times 10^{-2}$} \\
			$h=256, L=4$                           & $(3.26 \pm 1.39) \times 10^{-2}$          & $(3.09 \pm 1.09) \times 10^{-2}$          & $(3.04 \pm 1.01) \times 10^{-2}$          \\
			$h=256, L=6$                           & $(6.10 \pm 2.69) \times 10^{-2}$          & $(5.81 \pm 2.56) \times 10^{-2}$          & $(5.68 \pm 2.58) \times 10^{-2}$          \\
			\bottomrule
		\end{tabular}%
		}
	\end{table}

	\begin{table}[htbp]
		\centering
		\caption{Spatial Atoms (M) Ablation}
		\label{tab:ablation_spatial_atoms} \resizebox{\columnwidth}{!}{%
		\begin{tabular}{lccc}
			\toprule Configuration               & $\mathcal{E}_{\bar{u}}$                   & $\mathcal{E}_{u_1}$                       & $\mathcal{E}_{u_2}$                       \\
			\midrule \textit{default (baseline)} & $(4.23 \pm 1.03) \times 10^{-3}$          & $(3.97 \pm 1.39) \times 10^{-3}$          & $(4.25 \pm 1.07) \times 10^{-3}$          \\
			$M=64$                                 & $(9.79 \pm 1.95) \times 10^{-3}$          & $(9.82 \pm 1.89) \times 10^{-3}$          & $(9.77 \pm 1.90) \times 10^{-3}$          \\
			$M=128$                                & $(6.11 \pm 1.37) \times 10^{-3}$          & $(6.16 \pm 1.30) \times 10^{-3}$          & $(6.07 \pm 1.30) \times 10^{-3}$          \\
			$M=512$                                & \textbf{$(1.91 \pm 0.49) \times 10^{-3}$} & \textbf{$(1.89 \pm 0.39) \times 10^{-3}$} & \textbf{$(1.93 \pm 0.62) \times 10^{-3}$} \\
			\bottomrule
		\end{tabular}%
		}
	\end{table}

	\begin{table}[H]
		\centering
		\caption{gPC Order (p) Ablation}
		\label{tab:ablation_gpc_order} \resizebox{\columnwidth}{!}{%
		\begin{tabular}{lccc}
			\toprule Configuration               & $\mathcal{E}_{\bar{u}}$                   & $\mathcal{E}_{u_1}$                       & $\mathcal{E}_{u_2}$                       \\
			\midrule \textit{default (baseline)} & $(4.23 \pm 1.03) \times 10^{-3}$          & $(3.97 \pm 1.39) \times 10^{-3}$          & $(4.25 \pm 1.07) \times 10^{-3}$          \\
			p=1                                  & $(3.85 \pm 1.75) \times 10^{-3}$          & $(4.30 \pm 1.91) \times 10^{-3}$          & $(4.44 \pm 1.68) \times 10^{-3}$          \\
			p=2                                  & \textbf{$(3.50 \pm 1.06) \times 10^{-3}$} & \textbf{$(3.54 \pm 1.03) \times 10^{-3}$} & \textbf{$(3.51 \pm 1.12) \times 10^{-3}$} \\
			p=4                                  & $(2.56 \pm 3.43) \times 10^{-2}$          & $(2.35 \pm 3.08) \times 10^{-2}$          & $(2.55 \pm 3.41) \times 10^{-2}$          \\
			\bottomrule
		\end{tabular}%
		}
	\end{table}

	\begin{table}[H]
		\centering
		\caption{CP Rank ($q$) Ablation}
		\label{tab:ablation_cp_rank} \resizebox{\columnwidth}{!}{%
		\begin{tabular}{lccc}
			\toprule Configuration               & $\mathcal{E}_{\bar{u}}$                   & $\mathcal{E}_{u_1}$                       & $\mathcal{E}_{u_2}$                       \\\midrule 
			$q=4$& $(4.74 \pm 1.84) \times 10^{-3}$          & $(4.66 \pm 1.79) \times 10^{-3}$          & $(4.79 \pm 1.89) \times 10^{-3}$          \\
			$q=8$& \textbf{$(2.80 \pm 1.04) \times 10^{-3}$} & \textbf{$(2.73 \pm 1.09) \times 10^{-3}$} & \textbf{$(2.85 \pm 0.93) \times 10^{-3}$} \\
             $q=16$ (\textit{default}) & $(4.23 \pm 1.03) \times 10^{-3}$          & $(3.97 \pm 1.39) \times 10^{-3}$          & $(4.25 \pm 1.07) \times 10^{-3}$          \\
			$q=32$& $(3.24 \pm 0.91) \times 10^{-3}$          & $(3.43 \pm 0.42) \times 10^{-3}$          & $(3.15 \pm 0.83) \times 10^{-3}$          \\
			\bottomrule
		\end{tabular}%
		}
	\end{table}

	\section{Scaling Study: Stochastic Dimension and CP Rank}
	\label{sec:scaling-sweep}

	To characterize how S$^{2}$-PINN scales with stochastic dimensionality $R$ and
	CP rank $q$, we sweep $R\in\{3,5,8,10\}$ and $q\in\{8,16,32\}$ on the
	stochastic diffusion benchmark with $p{=}3$ fixed. 
    \Cref{tab:scaling-sweep} shows the full results
	including wall-clock time and peak GPU memory. The $R=15$ setting is included as a larger stress test and is evaluated for $q\in\{16,32\}$.
    
	{The table supports three conclusions. First, mean and variance
	errors remain in the $10^{-3}$ to $10^{-2}$ range through the diffusion stress
	test, including $R{=}15$ where the gPC basis has $G{=}816$ modes. Second, larger
	CP rank is not uniformly better on this benchmark: $q{=}8$ is competitive for
	$R\le 10$, while $q{=}32$ improves the reported $R{=}15$ run. Third, wall-clock
	time increases with stochastic dimension, but peak GPU memory remains close to
	$1$\,GB, consistent with the low-rank CP coupling avoiding a dense
	$M M_t G$ tensor.}
	\add{For reference, $R{=}15$ with $p{=}3$ already yields $G{=}\binom{18}{3}{=}816$ stochastic modes, comparable to the number of retained modes in KL-based stochastic modeling studies of heterogeneous media. To probe extreme scaling, we additionally run $R{=}50$ ($G{=}23{,}426$): the mean and variance errors remain at the $10^{-2}$ level, while peak memory grows to about $18$\,GB, reflecting the polynomial growth of $G$ with $R$ (Appendix~\ref{sec:cost-analysis}).}

	\begin{table}[htbp]
		\centering
		\caption{Scaling sweep on stochastic diffusion: mean/variance relative $L
		^{2}$ error, wall-clock time, and peak GPU memory
		across stochastic dimension $R$ and CP rank $q$ ($p{=}3$, $M{=}256$, 20K
		steps.) \add{$^{\ast}$The $R{=}50$ run ($G{=}23{,}426$) uses learning rate $10^{-4}$ and an NVIDIA L40S GPU because of its ${\approx}18$\,GB peak memory; its time is therefore not directly comparable with the T4 runs.}}
		\label{tab:scaling-sweep} \scriptsize \resizebox{\columnwidth}{!}{%
		\begin{tabular}{cc|cccc}
			\toprule $R$ & $q$ & Mean rel.\ $L^{2}$                        & Var.\ rel.\ $L^{2}$                       & Time (min)      & Mem (GB)      \\
			\midrule 3   & 8   & $(5.88 \pm 1.83) \times 10^{-3}$          & $(8.00 \pm 2.37) \times 10^{-3}$          & $84.6 \pm 0.9$  & $0.99$        \\
			3            & 16  & $(5.81 \pm 0.89) \times 10^{-3}$          & $(7.32 \pm 1.38) \times 10^{-3}$          & $84.9 \pm 0.6$  & $0.99$        \\
			3            & 32  & $(5.37 \pm 0.66) \times 10^{-3}$          & $(9.61 \pm 2.98) \times 10^{-3}$          & $84.3 \pm 0.5$  & $0.99$        \\
			\midrule 5   & 8   & \textbf{$(4.50 \pm 0.31) \times 10^{-3}$} & $(6.96 \pm 2.23) \times 10^{-3}$          & $93.6 \pm 0.4$  & $0.99$        \\
			5            & 16  & $(6.43 \pm 1.71) \times 10^{-3}$          & $(1.26 \pm 0.45) \times 10^{-2}$          & $93.6 \pm 0.6$  & $0.99$        \\
			5            & 32  & $(4.58 \pm 0.94) \times 10^{-3}$          & $(7.94 \pm 1.29) \times 10^{-3}$          & $94.0 \pm 0.2$  & $0.99$        \\
			\midrule 8   & 8   & $(5.86 \pm 0.73) \times 10^{-3}$          & $(8.14 \pm 1.46) \times 10^{-3}$          & $146.1 \pm 1.1$ & $0.99$        \\
			8            & 16  & $(5.81 \pm 2.25) \times 10^{-3}$          & $(8.50 \pm 2.17) \times 10^{-3}$          & $146.7 \pm 0.4$ & $0.99$        \\
			8            & 32  & $(6.41 \pm 1.72) \times 10^{-3}$          & $(8.69 \pm 1.02) \times 10^{-3}$          & $146.1 \pm 1.0$ & $1.00$        \\
			\midrule 10  & 8   & $(4.27 \pm 0.62) \times 10^{-3}$          & \textbf{$(5.57 \pm 0.98) \times 10^{-3}$} & $233.6 \pm 1.7$ & $0.99$        \\
			10           & 16  & $(5.49 \pm 2.30) \times 10^{-3}$          & $(7.49 \pm 3.44) \times 10^{-3}$          & $233.8 \pm 1.8$ & $0.99$        \\
			10           & 32  & $(5.16 \pm 1.78) \times 10^{-3}$          & $(6.10 \pm 1.81) \times 10^{-3}$          & $234.5 \pm 1.0$ & $0.99$        \\
			\midrule 15  & 16  & $(7.30 \pm 1.55) \times 10^{-3}$          & $(9.46 \pm 2.84) \times 10^{-3}$          & {$677.3\pm 1.7$}   & {$1.06$} \\
			15           & 32  & $(4.18 \pm 0.52) \times 10^{-3}$          & $(5.40 \pm 0.39) \times 10^{-3}$          &  {$679.5\pm 2.3$}   & {$1.06$}  \\
			\midrule \add{50$^{\ast}$} & \add{--} & \add{$(1.09 \pm 0.14) \times 10^{-2}$} & \add{$(1.24 \pm 0.30) \times 10^{-2}$} & \add{$131.9 \pm 0.3$} & \add{$18.0$} \\
			\bottomrule
		\end{tabular}%
		}
	\end{table}

	\section{Collocation Density Sensitivity}
	\label{sec:colloc-sweep}

	We vary the number of collocation points $N_{c}\in\{4096, 8192, 16384, 32768,
	65536\}$ on the stochastic diffusion benchmark ($R{=}5$, $p{=}3$, $q{=}16$, 20K
	steps) to characterize sensitivity to spatial sampling
	density. \Cref{tab:colloc-sweep} reports the results.

	{Mean error decreases modestly from $6.34\!\times\!10^{-3}$ at
	$N_{c}{=}4096$ to $4.68\!\times\!10^{-3}$ at $N_{c}{=}65536$, a
	$1.4\times$ improvement over a $16\times$ increase in batch size. Variance
	error is also lowest at the largest batch, while calibration error is not
	monotone. We therefore use $N_{c}{=}32{,}768$ as a practical default rather
	than claiming that larger collocation batches are uniformly better.}

	\begin{table}[htbp]
		\centering
		\caption{Collocation density sweep on stochastic diffusion ($R{=}5$,
		$p{=}3$, $q{=}16$, 20K steps).}
		\label{tab:colloc-sweep} \scriptsize \resizebox{\columnwidth}{!}{%
		\begin{tabular}{c|ccc}
			\toprule $N_{c}$ & Mean rel.\ $L^{2}$                        & Var.\ rel.\ $L^{2}$                       & Cal.\ error                      \\
			\midrule 4096    & $(6.34 \pm 3.27) \times 10^{-3}$          & $(1.30 \pm 1.25) \times 10^{-2}$          & $(1.85 \pm 0.32) \times 10^{-2}$ \\
			8192             & $(5.95 \pm 1.58) \times 10^{-3}$          & $(9.70 \pm 3.23) \times 10^{-3}$          & $(1.47 \pm 0.53) \times 10^{-2}$ \\
			16384            & $(5.72 \pm 0.79) \times 10^{-3}$          & $(9.54 \pm 1.75) \times 10^{-3}$          & $(1.66 \pm 0.72) \times 10^{-2}$ \\
			32768            & $(6.43 \pm 1.71) \times 10^{-3}$          & $(1.26 \pm 0.45) \times 10^{-2}$          & $(1.27 \pm 0.35) \times 10^{-2}$ \\
			65536            & \textbf{$(4.68 \pm 0.56) \times 10^{-3}$} & \textbf{$(8.49 \pm 5.49) \times 10^{-3}$} & $(1.94 \pm 0.52) \times 10^{-2}$ \\
			\bottomrule
		\end{tabular}%
		}
	\end{table}

	\begin{addblock}
	\section{Computational Cost with Respect to \texorpdfstring{$q$}{q}, \texorpdfstring{$M$}{M}, and \texorpdfstring{$R$}{R}}
	\label{sec:cost-analysis}

	\paragraph{CP core.}
	For a batch of size $B$, the forward map is $u_b=\sum_{k=1}^{q}w_k(\mathbf{F}_{\text{space}}A)_{bk}(\mathbf{F}_{\text{time}}B_t)_{bk}(\mathbf{F}_{\text{stoc}}D)_{bk}$ with $A\in\mathbb{R}^{M\times q}$, $B_t\in\mathbb{R}^{M_t\times q}$, and $D\in\mathbb{R}^{G\times q}$. The CP core therefore has exactly $P_{\mathrm{CP}}=q(M+M_t+G+1)$ parameters, and the three matrix products cost $O(BMq)$, $O(BM_tq)$, and $O(BGq)$, i.e., $T_{\mathrm{CP}}=O\big(Bq(M+M_t+G)\big)$. A dense coupling tensor would require $P_{\mathrm{dense}}=MM_tG$ parameters and $T_{\mathrm{dense}}=O(BMM_tG)$ operations, so the CP coupling is cheaper whenever $q\ll MM_tG/(M+M_t+G)$.

	\paragraph{Cost in $q$.}
	Each additional CP component adds $\Delta_q P_{\mathrm{CP}}=M+M_t+G+1$ parameters and $\Delta_q T_{\mathrm{CP}}=O\big(B(M+M_t+G)\big)$ operations; for fixed $M,M_t,G$ we have $P_{\mathrm{CP}}=\Theta(q)$ and $T_{\mathrm{CP}}=\Theta(Bq)$. The best rank-$q$ approximation error is non-increasing in $q$, because rank-$(q{-}1)$ tensors are contained in the rank-$q$ class. The accuracy of a finitely trained model, however, need not improve monotonically, because optimization is non-convex and subject to a fixed step budget, sampling noise, and regularization. This explains the non-monotone behaviour across $q\in\{8,16,32\}$ in \Cref{tab:ablation_cp_rank,tab:scaling-sweep}.

	\paragraph{Cost in $M$.}
	Each additional Gaussian atom adds $q$ entries to $A$ and three dictionary parameters (two centre coordinates and one width), i.e., $\Delta_M P\approx q+3$, and costs $\Delta_M T=O\big(B(q+1)\big)$ for feature evaluation and contraction. For fixed remaining sizes, $P=\Theta(Mq)$ and $T=\Theta(BMq)$. First- and second-order spatial derivatives only multiply the $BM$ term by a constant that depends on the derivative order of the PDE and do not change the linear dependence on $M$.

	\paragraph{Cost in $R$.}
	The number of total-degree gPC modes is $G(R,p)=\binom{R+p}{p}$; increasing $R\to R+1$ adds $\Delta_R G=\binom{R+p}{p-1}$ modes, hence $\Delta_R P_{\mathrm{CP}}=q\binom{R+p}{p-1}$ and $\Delta_R T_{\mathrm{CP}}=O\big(Bq\binom{R+p}{p-1}\big)$. For fixed $p$, $G(R,p)=\Theta(R^p/p!)$, so that
	\[
	P_{\mathrm{CP}}=O\Big(qM+qM_t+q\tfrac{R^p}{p!}\Big),\qquad T_{\mathrm{CP}}=O\Big(Bq\big(M+M_t+\tfrac{R^p}{p!}\big)\Big).
	\]
	Thus, for a fixed degree $p$, the CP core grows polynomially (of degree $p$) in $R$ rather than exponentially; the curse of dimensionality is mitigated but not eliminated.

	\paragraph{gPC basis evaluation.}
	Evaluating all $G$ tensorized modes costs $T_{\mathrm{gPC,impl}}=O(BRp+BRG)$, where $BRp$ accounts for the univariate polynomials of each random variable and $BRG$ for assembling the tensorized modes. Since every total-degree multi-index with $|\boldsymbol\alpha|\le p$ has at most $p$ non-zero entries, a sparse multi-index evaluation reduces this to $T_{\mathrm{gPC,sparse}}=O(BRp+BpG)$.

	\paragraph{Projected loss and total training cost.}
	The projected residual adds $T_{\mathrm{proj}}=O(BG)$ for the residual--basis products and reductions. Writing $C_{\mathrm{PDE}}$ for the constant factor of the PDE residual and its automatic differentiation, one training step costs
	\[
	T_{\mathrm{step}}=O\Big(C_{\mathrm{PDE}}\,B\big[M+M_t+Rp+RG+q(M+M_t+G)\big]+BG\Big),
	\]
	where $RG$ can be replaced by $pG$ with sparse gPC evaluation, and $S$ training steps cost $T_{\mathrm{train}}=S\,T_{\mathrm{step}}$. The activation memory is approximately $O\big(B(M+M_t+G+q)\big)$, and the memory for parameters and optimizer states is $O\big(q(M+M_t+G)+M\big)$.

	\paragraph{Accuracy versus nominal capacity.}
	The retained index set $\Lambda$ (determined by $R$ and $p$) and the CP rank $q$ play different roles: $G=\binom{R+p}{p}$ is the number of retained gPC modes, whereas $q$ is the separation rank. For an orthonormal expansion $u=\sum_{\boldsymbol\alpha}u_{\boldsymbol\alpha}\Psi_{\boldsymbol\alpha}$, the best-approximation error of $\Lambda$ is $\|u-\mathsf{P}_\Lambda u\|_{L^2}^2=\sum_{\boldsymbol\alpha\notin\Lambda}|u_{\boldsymbol\alpha}|^2$. Enlarging $R$ or $p$ strictly reduces this truncation error only when the newly included coefficients carry non-negligible stochastic energy, and a larger space also increases estimation and optimization difficulty under a fixed sampling and training budget. This is visible in \Cref{tab:ablation_gpc_order,tab:scaling-sweep}: $q{=}8$ is competitive for $R\le 10$, whereas $q{=}32$ helps at $R{=}15$, and $p{=}2$ outperforms $p{=}4$ in the trained model. Larger $R$, $p$, and $q$ increase nominal capacity, but by themselves do not guarantee better finite-budget accuracy.

	\section{Stochastic Navier--Stokes Benchmark}
	\label{sec:navier-stokes}
	We consider the 2D incompressible Navier--Stokes equations in vorticity form,
	\[
	\partial_t\omega+\mathbf{u}\cdot\nabla\omega=\nu(\boldsymbol\xi)\Delta\omega+f,\qquad \nabla\cdot\mathbf{u}=0,\qquad \omega=\nabla\times\mathbf{u},
	\]
	on $\Omega=[0,2\pi]^2$, $t\in[0,5]$, with $\boldsymbol\xi\sim\mathcal{U}([-1,1]^5)$, where $\mathbf{u}=(u_1,u_2)$ is the velocity, $\omega$ the vorticity, and the pressure has been eliminated. The five random variables control the two phase shifts, the viscosity, and the amplitudes of two vortex modes:
	\[
	X=x+0.25\xi_1,\quad Y=y+0.25\xi_2,\quad \nu(\boldsymbol\xi)=0.05(1+0.2\xi_3)\in[0.04,0.06],
	\]
	\[
	A(\boldsymbol\xi)=1+0.1\xi_4\in[0.9,1.1],\quad B(\boldsymbol\xi)=0.5(1+0.2\xi_5)\in[0.4,0.6].
	\]
	With $a=Ae^{-2\nu t}$ and $b=Be^{-5\nu t}$, the stream function is $\psi=a\sin X\sin Y+b\sin(2X)\sin Y$, and the exact velocity and vorticity are
	\[
	u_1=\partial_Y\psi=a\sin X\cos Y+b\sin(2X)\cos Y,\quad u_2=-\partial_X\psi=-a\cos X\sin Y-2b\cos(2X)\sin Y,
	\]
	\[
	\omega=\partial_X u_2-\partial_Y u_1=2a\sin X\sin Y+5b\sin(2X)\sin Y.
	\]
	This construction satisfies incompressibility exactly; the linear time-derivative and diffusion terms cancel, while the two modes produce a non-zero nonlinear interaction. The resulting analytical forcing is $f(t,x,y,\boldsymbol\xi)=-6ab\sin^3X\sin Y\cos Y$. Results are reported in \Cref{tab:ns-moments,tab:ns-calibration}. Neural Chaos is again a supervised oracle.

	\begin{table}[htbp]
		\centering
		\caption{Stochastic Navier--Stokes: relative $L^2$ errors of the velocity mean and variance.}
		\label{tab:ns-moments}\scriptsize\resizebox{\columnwidth}{!}{%
		\begin{tabular}{lcccc}
			\toprule
			Method & $u_1$ mean & $u_2$ mean & $u_1$ variance & $u_2$ variance \\
			\midrule
			S$^{2}$-PINN (proj) & $\mathbf{(2.05\pm0.34)\times10^{-3}}$ & $\mathbf{(2.40\pm0.21)\times10^{-3}}$ & $\mathbf{(3.13\pm0.49)\times10^{-2}}$ & $\mathbf{(2.58\pm0.71)\times10^{-2}}$ \\
			S$^{2}$-PINN (no-proj) & $(1.09\pm0.36)\times10^{-2}$ & $(1.15\pm0.27)\times10^{-2}$ & $(2.10\pm0.10)\times10^{-1}$ & $(1.27\pm0.03)\times10^{-1}$ \\
			Vanilla PINN & $(3.44\pm0.43)\times10^{-2}$ & $(3.17\pm1.22)\times10^{-2}$ & $(2.15\pm0.21)\times10^{-1}$ & $(1.89\pm0.17)\times10^{-1}$ \\
			SPINN+Z & $(1.70\pm0.62)\times10^{-2}$ & $(1.79\pm0.34)\times10^{-2}$ & $(2.63\pm0.81)\times10^{-1}$ & $(2.05\pm0.82)\times10^{-1}$ \\
			Functional CP-PINN & $(3.48\pm0.75)\times10^{-2}$ & $(2.83\pm0.74)\times10^{-2}$ & $(1.39\pm0.39)\times10^{-1}$ & $(9.81\pm3.49)\times10^{-2}$ \\
			Neural Chaos (supervised oracle) & $(1.98\pm1.10)\times10^{-2}$ & $(1.93\pm0.64)\times10^{-2}$ & $(7.91\pm1.09)\times10^{-2}$ & $(1.40\pm0.38)\times10^{-1}$ \\
			\bottomrule
		\end{tabular}}
	\end{table}

	\begin{table}[htbp]
		\centering
		\caption{Stochastic Navier--Stokes: velocity calibration. Coverage values are fractions; Cov90 and Cov95 are better when closer to $0.900$ and $0.950$, and the calibration error is better when smaller.}
		\label{tab:ns-calibration}\scriptsize
		\begin{tabular}{llccc}
			\toprule
			Method & Component & Cov90 & Cov95 & Cal.~Err. \\
			\midrule
			S$^{2}$-PINN (proj) & $u_1$ & $\mathbf{0.886\pm0.002}$ & $\mathbf{0.938\pm0.002}$ & $\mathbf{0.011\pm0.003}$ \\
			 & $u_2$ & $\mathbf{0.892\pm0.001}$ & $\mathbf{0.944\pm0.001}$ & $\mathbf{0.007\pm0.001}$ \\
			S$^{2}$-PINN (no-proj) & $u_1$ & $0.771\pm0.026$ & $0.837\pm0.024$ & $0.119\pm0.025$ \\
			 & $u_2$ & $0.854\pm0.016$ & $0.909\pm0.015$ & $0.041\pm0.015$ \\
			Vanilla PINN & $u_1$ & $0.753\pm0.015$ & $0.811\pm0.015$ & $0.128\pm0.014$ \\
			 & $u_2$ & $0.793\pm0.011$ & $0.852\pm0.011$ & $0.095\pm0.011$ \\
			SPINN+Z & $u_1$ & $0.754\pm0.008$ & $0.810\pm0.008$ & $0.129\pm0.007$ \\
			 & $u_2$ & $0.803\pm0.020$ & $0.857\pm0.018$ & $0.085\pm0.018$ \\
			Functional CP-PINN & $u_1$ & $0.828\pm0.023$ & $0.881\pm0.021$ & $0.061\pm0.022$ \\
			 & $u_2$ & $0.841\pm0.015$ & $0.894\pm0.013$ & $0.051\pm0.014$ \\
			Neural Chaos (supervised oracle) & $u_1$ & $\mathbf{0.886\pm0.013}$ & $0.935\pm0.011$ & $\mathbf{0.012\pm0.009}$ \\
			 & $u_2$ & $0.887\pm0.003$ & $0.937\pm0.003$ & $0.010\pm0.002$ \\
			\bottomrule
		\end{tabular}
	\end{table}


	\section{Additional Baseline Results}
	\label{sec:extra-baselines}
	\Cref{tab:calibration-extra} complements \Cref{tab:calibration} with the calibration of PI-DeepONet (all metrics) and of SPINN+Z, functional CP-PINN, and Neural Chaos (calibration error only), and \Cref{tab:runtime} reports training time and peak GPU memory.
	\begin{table}[htbp]
		\centering
		\caption{Calibration of additional baselines (same protocol as \Cref{tab:calibration}). $^{\ddagger}$Supervised oracle with one shared checkpoint for diffusion, Allen--Cahn, and Burgers; see \Cref{tab:mean-accuracy}.}
		\label{tab:calibration-extra} \scriptsize \resizebox{\columnwidth}{!}{%
		\begin{tabular}{lcccccc}
			\hline
			Method / PDE & Cov@50\% & Cov@80\% & Cov@90\% & Cov@95\% & Sharpness$_{90}$ & Cal.~Err. \\
			\hline
\multicolumn{7}{l}{{PI-DeepONet}} \\
\quad {Diffusion} & {$0.414\pm0.048$} & {$0.678\pm0.064$} & {$0.770\pm0.066$} & {$0.819\pm0.062$} & {$0.197\pm0.002$} & {$0.117\pm0.060$} \\
\quad {Allen--Cahn} & {$0.413\pm0.046$} & {$0.676\pm0.062$} & {$0.767\pm0.063$} & {$0.815\pm0.060$} & {$0.199\pm0.001$} & {$0.120\pm0.058$} \\
\quad {Burgers} & {$0.461\pm0.024$} & {$0.741\pm0.031$} & {$0.834\pm0.030$} & {$0.878\pm0.028$} & {$0.202\pm0.005$} & {$0.059\pm0.028$} \\
\quad {Darcy} & {$0.448\pm0.053$} & {$0.723\pm0.060$} & {$0.827\pm0.052$} & {$0.887\pm0.044$} & {$0.095\pm0.001$} & {$0.067\pm0.050$} \\
\multicolumn{7}{l}{{SPINN+Z (calibration error only)}} \\
\quad {Diffusion} & {--} & {--} & {--} & {--} & {--} & {$0.013\pm0.003$} \\
\quad {Allen--Cahn} & {--} & {--} & {--} & {--} & {--} & {$0.011\pm0.002$} \\
\quad {Burgers} & {--} & {--} & {--} & {--} & {--} & {$0.013\pm0.001$} \\
\quad {Darcy} & {--} & {--} & {--} & {--} & {--} & {$0.008\pm0.001$} \\
\multicolumn{7}{l}{{Functional CP-PINN (calibration error only)}} \\
\quad {Diffusion} & {--} & {--} & {--} & {--} & {--} & {$0.153\pm0.054$} \\
\quad {Allen--Cahn} & {--} & {--} & {--} & {--} & {--} & {$0.173\pm0.064$} \\
\quad {Burgers} & {--} & {--} & {--} & {--} & {--} & {$0.159\pm0.038$} \\
\quad {Darcy} & {--} & {--} & {--} & {--} & {--} & {$0.050\pm0.025$} \\
\multicolumn{7}{l}{{Neural Chaos$^{\ddagger}$ (calibration error only)}} \\
\quad {Diffusion} & {--} & {--} & {--} & {--} & {--} & {$0.054\pm0.006$} \\
\quad {Allen--Cahn} & {--} & {--} & {--} & {--} & {--} & {--} \\
\quad {Burgers} & {--} & {--} & {--} & {--} & {--} & {--} \\
\quad {Darcy} & {--} & {--} & {--} & {--} & {--} & {$0.093\pm0.070$} \\
			\hline
		\end{tabular}%
		}
	\end{table}
	\begin{table}[htbp]
		\centering
		\caption{Training time and peak GPU memory on stochastic diffusion ($R{=}5$).}
		\label{tab:runtime}\small\resizebox{\columnwidth}{!}{%
		\begin{tabular}{lcccccc}
			\toprule
			 & S$^{2}$-PINN & Vanilla PINN & PI-DeepONet & SPINN+Z & Func.\ CP-PINN & Neural Chaos \\
			\midrule
			Time (min) & ${\approx}90$ & ${\approx}80$ & ${\approx}120$ & ${\approx}100$ & ${\approx}100$ & ${\approx}120$ \\
			Peak memory (GB) & $1.0$ & $4.0$ & $5.0$ & $0.3$ & $0.3$ & $1.5$ \\
			\bottomrule
		\end{tabular}}
	\end{table}
	\end{addblock}

	\section{Visualizations of Stochastic Fields}
	\label{sec:visualizations}

	To visually compare spatial stochastic structure across PDE families, we present
	moment heatmaps (mean/variance) at $t{=}1.0$ for diffusion, Allen--Cahn,
	Burgers, and the two non-manufactured benchmarks (Poisson MC and Darcy MC).
	Each panel includes both 2D fields and 3D surfaces for reference, model prediction,
	and absolute error.

	\begin{figure}[htbp]
		\centering
		\includegraphics[width=\linewidth]{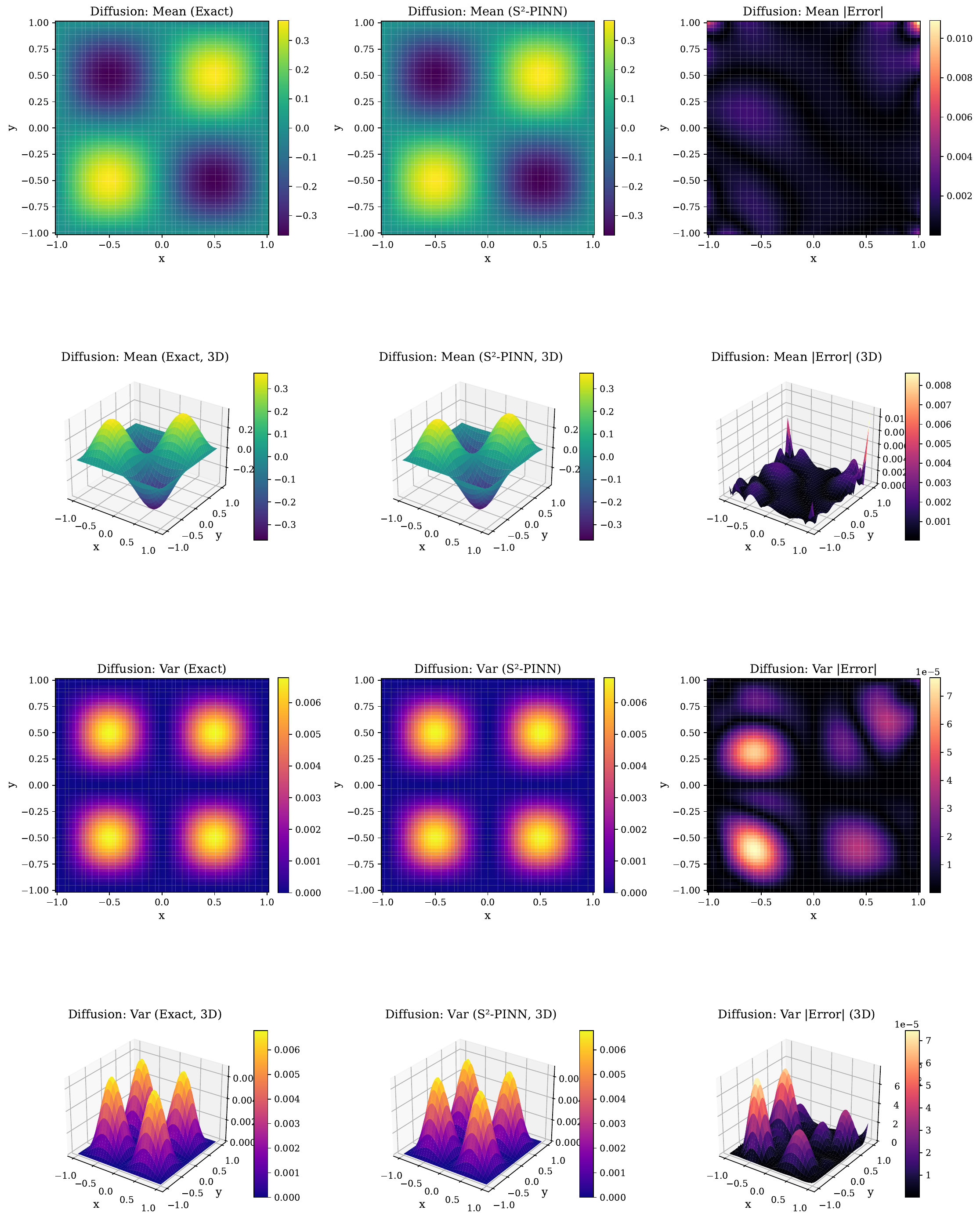}
		\caption{Diffusion heatmaps (2D and 3D) at $t{=}1.0$: reference mean/variance,
		model-predicted mean/variance, and absolute errors.}
		\label{fig:heatmaps-diffusion}
	\end{figure}

	\begin{figure}[htbp]
		\centering
		\includegraphics[width=\linewidth]{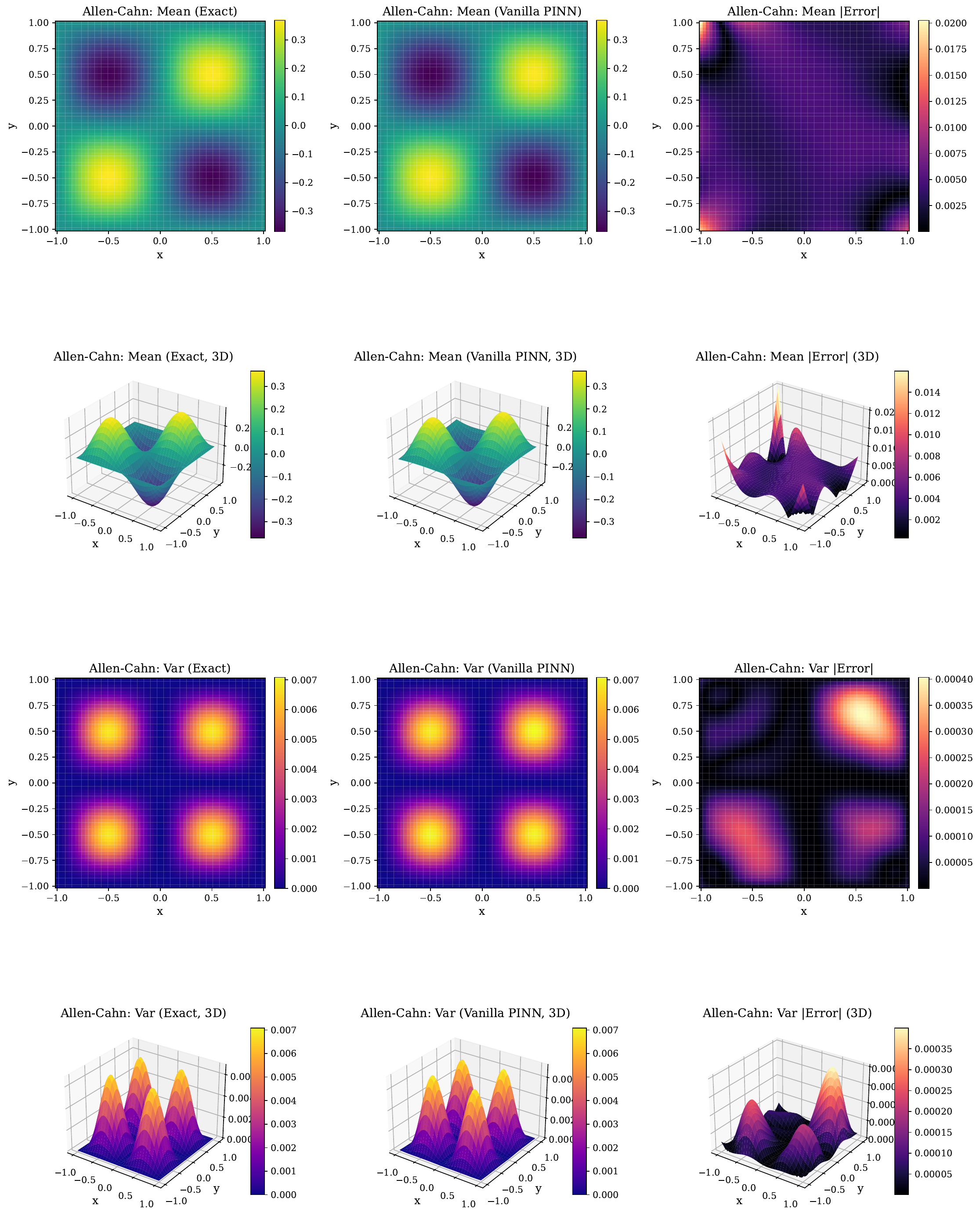}
		\caption{Allen--Cahn heatmaps (2D and 3D) at $t{=}1.0$: reference mean/variance,
		model-predicted mean/variance, and absolute errors.}
		\label{fig:heatmaps-allen-cahn}
	\end{figure}

	\begin{figure}[htbp]
		\centering
		\includegraphics[width=\linewidth]{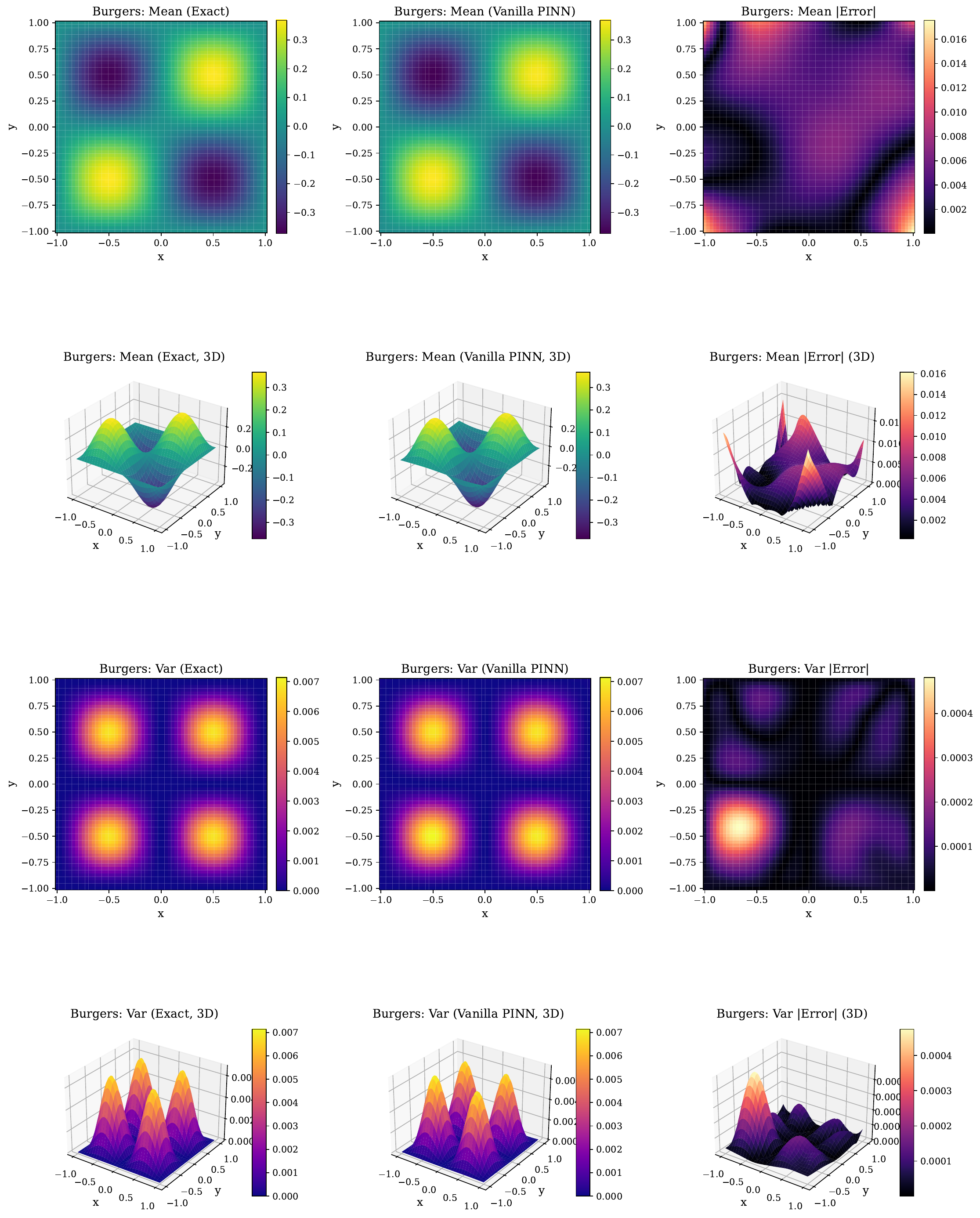}
		\caption{Burgers heatmaps (2D and 3D) at $t{=}1.0$: reference mean/variance,
		model-predicted mean/variance, and absolute errors.}
		\label{fig:heatmaps-burgers}
	\end{figure}

	\begin{figure}[htbp]
		\centering
		\includegraphics[width=\linewidth]{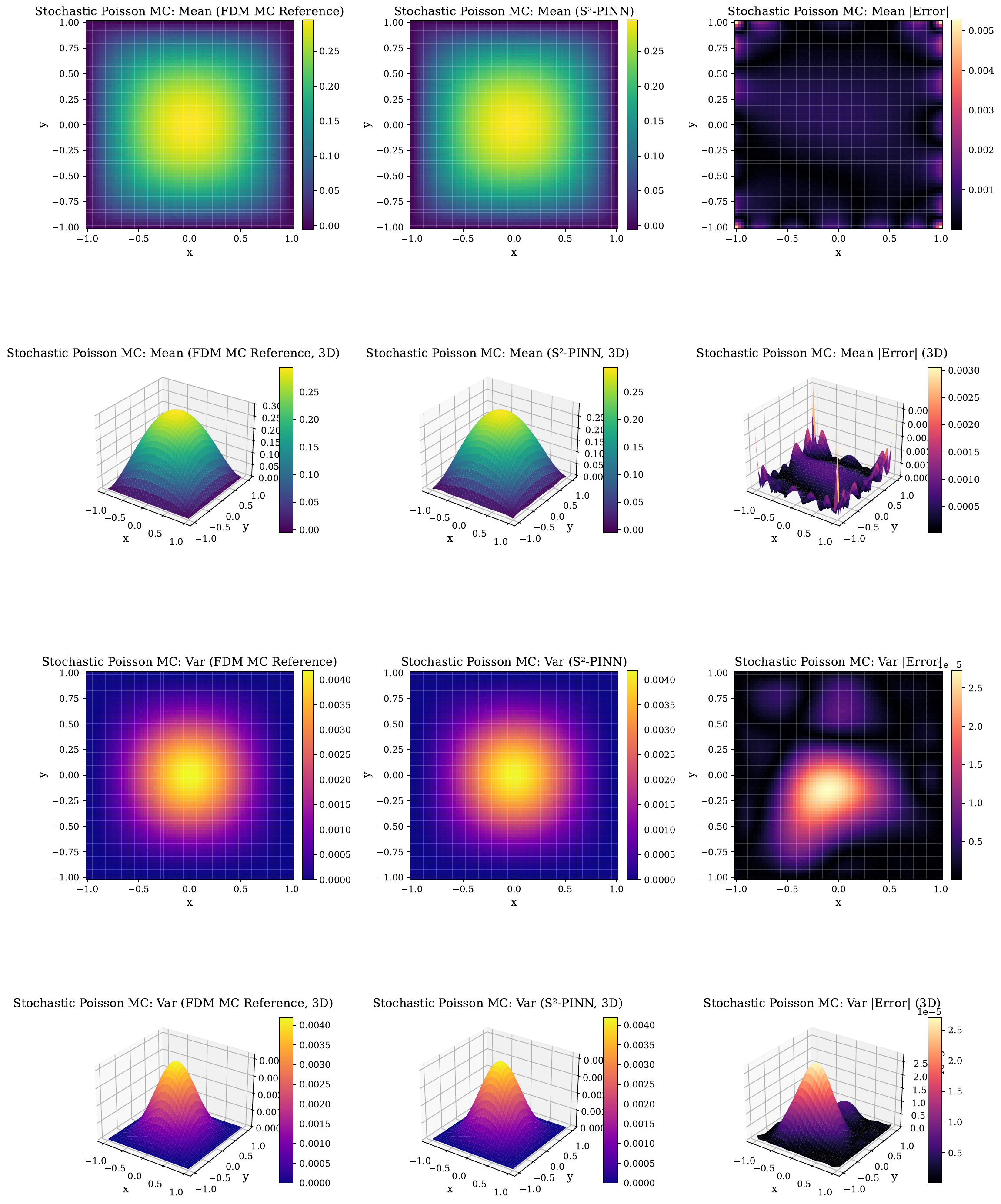}
		\caption{Non-manufactured Poisson heatmaps (2D and 3D) at $t{=}1.0$: FDM
		Monte Carlo reference mean/variance ($10{,}000$ samples), S$^{2}$-PINN prediction,
		and absolute errors.}
		\label{fig:heatmaps-poisson}
	\end{figure}

	\begin{figure}[htbp]
		\centering
		\includegraphics[width=\linewidth]{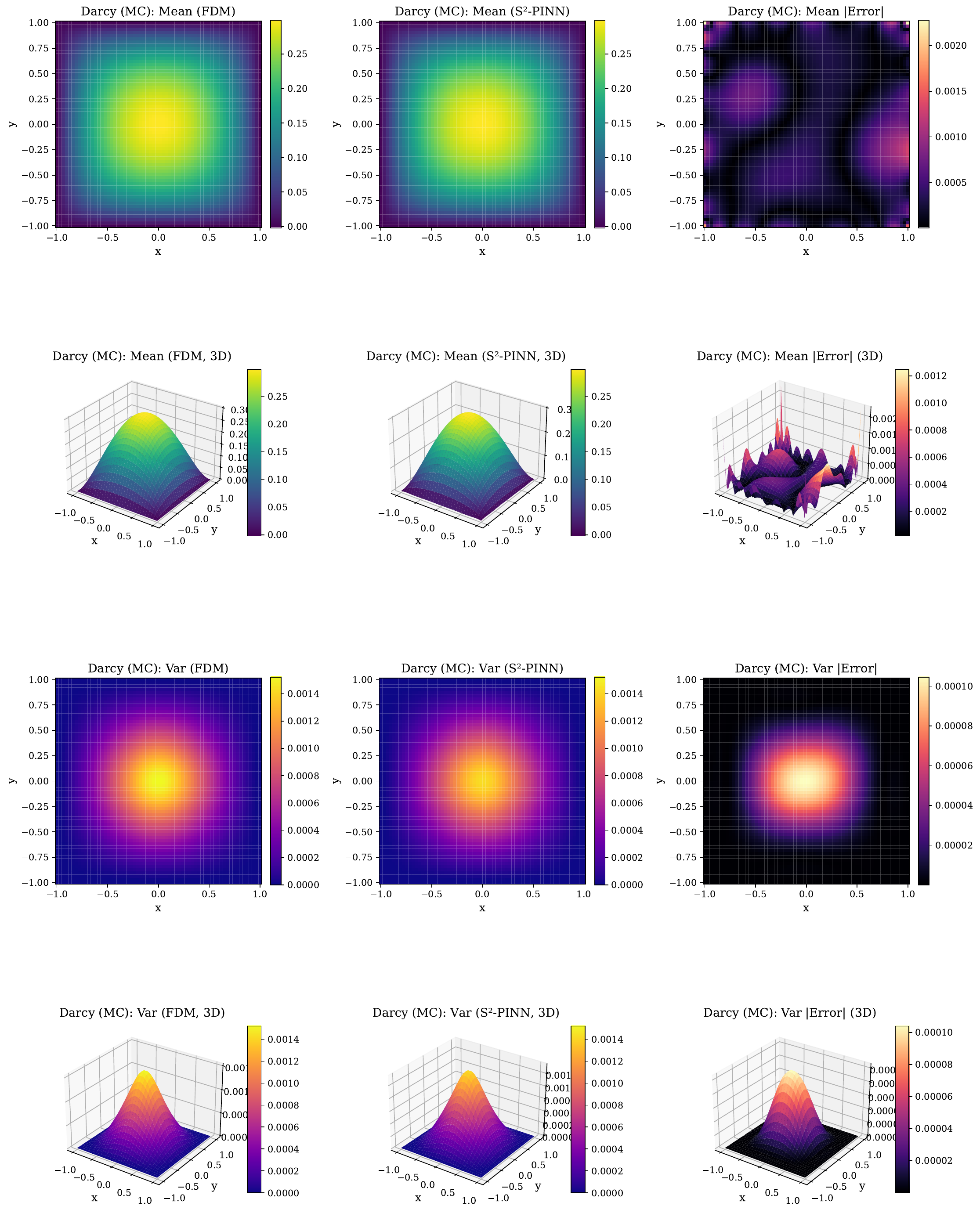}
		\caption{Non-manufactured Darcy heatmaps (2D and 3D) at $t{=}1.0$: FDM Monte
		Carlo reference mean/variance ($10{,}000$ samples), S$^{2}$-PINN prediction,
		and absolute errors.}
		\label{fig:heatmaps-darcy}
	\end{figure}


	\newpage
	\input{checklist.tex}
\end{document}

%% file: checklist.tex
\section*{NeurIPS Paper Checklist}

The checklist is designed to encourage best practices for responsible machine learning research, addressing issues of reproducibility, transparency, research ethics, and societal impact. Do not remove the checklist: {\bf The papers not including the checklist will be desk rejected.} The checklist should follow the references and follow the (optional) supplemental material.  The checklist does NOT count towards the page
limit. 

Please read the checklist guidelines carefully for information on how to answer these questions. For each question in the checklist:
\begin{itemize}
    \item You should answer \answerYes{}, \answerNo{}, or \answerNA{}.
    \item \answerNA{} means either that the question is Not Applicable for that particular paper or the relevant information is Not Available.
    \item Please provide a short (1--2 sentence) justification right after your answer (even for \answerNA). 
\end{itemize}

{\bf The checklist answers are an integral part of your paper submission.} They are visible to the reviewers, area chairs, senior area chairs, and ethics reviewers. You will also be asked to include it (after eventual revisions) with the final version of your paper, and its final version will be published with the paper.

The reviewers of your paper will be asked to use the checklist as one of the factors in their evaluation. While \answerYes{} is generally preferable to \answerNo{}, it is perfectly acceptable to answer \answerNo{} provided a proper justification is given (e.g., error bars are not reported because it would be too computationally expensive'' or ``we were unable to find the license for the dataset we used''). In general, answering \answerNo{} or \answerNA{} is not grounds for rejection. While the questions are phrased in a binary way, we acknowledge that the true answer is often more nuanced, so please just use your best judgment and write a justification to elaborate. All supporting evidence can appear either in the main paper or the supplemental material, provided in appendix. If you answer \answerYes{} to a question, in the justification please point to the section(s) where related material for the question can be found.

IMPORTANT, please:
\begin{itemize}
    \item {\bf Delete this instruction block, but keep the section heading ``NeurIPS Paper Checklist"},
    \item  {\bf Keep the checklist subsection headings, questions/answers and guidelines below.}
    \item {\bf Do not modify the questions and only use the provided macros for your answers}.
\end{itemize}


\begin{enumerate}

\item {\bf Claims}
    \item[] Question: Do the main claims made in the abstract and introduction accurately reflect the paper's contributions and scope?
    \item[] Answer: \answerYes{}
    \item[] Justification: The abstract and introduction accurately describe the proposed contributions about theoretical and experimental results.

\item {\bf Limitations}
    \item[] Question: Does the paper discuss the limitations of the work performed by the authors?
    \item[] Answer: \answerYes{}
    \item[] Justification: We discuss limitations of our work in Section \ref{sec:conclusion}

\item {\bf Theory assumptions and proofs}
    \item[] Question: For each theoretical result, does the paper provide the full set of assumptions and a complete (and correct) proof?
    \item[] Answer: \answerYes{}
    \item[] Justification: We provide our proof for stability and consistency in Appendix~\ref{sec:theory}

    \item {\bf Experimental result reproducibility}
    \item[] Question: Does the paper fully disclose all the information needed to reproduce the main experimental results of the paper to the extent that it affects the main claims and/or conclusions of the paper (regardless of whether the code and data are provided or not)?
    \item[] Answer: \answerYes{}
    \item[] Justification: Yes, we include all experimental details necessary to reproduce. Pseudo code for the algorithm is given in Algorithm~\ref{alg:training} and default parameters are listed in Appendix~\ref{sec:model-details}.

\item {\bf Open access to data and code}
    \item[] Question: Does the paper provide open access to the data and code, with sufficient instructions to faithfully reproduce the main experimental results, as described in supplemental material?
    \item[] Answer: \answerYes{}
    \item[] Justification: You can visit this \href{https://anonymous.4open.science/status/neurips2026-0424}{anonymous link} to get the code.

\item {\bf Experimental setting/details}
    \item[] Question: Does the paper specify all the training and test details (e.g., data splits, hyperparameters, how they were chosen, type of optimizer) necessary to understand the results?
    \item[] Answer:  \answerYes{}
    \item[] Justification: Default parameters are listed in Appendix~\ref{sec:model-details}.

\item {\bf Experiment statistical significance}
    \item[] Question: Does the paper report error bars suitably and correctly defined or other appropriate information about the statistical significance of the experiments?
    \item[] Answer:  \answerYes{}
    \item[] Justification: All the results  are reported as an average of 5 random seeds (42,123,456,789.1024) along with mean and standard deviation.

\item {\bf Experiments compute resources}
    \item[] Question: For each experiment, does the paper provide sufficient information on the computer resources (type of compute workers, memory, time of execution) needed to reproduce the experiments?
    \item[] Answer: \answerYes{}
    \item[] Justification: All the experiments are run on NVIDIA T4 and it is reported in~\Cref{sec:model-details}.
    
\item {\bf Code of ethics}
    \item[] Question: Does the research conducted in the paper conform, in every respect, with the NeurIPS Code of Ethics \url{https://neurips.cc/public/EthicsGuidelines}?
    \item[] Answer: \answerYes{}
    \item[] Justification: Yes, the research conducted complies with the NeurIPS ethical guidelines in all respects.

\item {\bf Broader impacts}
    \item[] Question: Does the paper discuss both potential positive societal impacts and negative societal impacts of the work performed?
    \item[] Answer: \answerYes{},
    \item[] Justification: They are reported in Section~\ref{sec:conclusion}.
    
\item {\bf Safeguards}
    \item[] Question: Does the paper describe safeguards that have been put in place for responsible release of data or models that have a high risk for misuse (e.g., pre-trained language models, image generators, or scraped datasets)?
    \item[] Answer: \answerNA{}
    \item[] Justification: The paper doesn’t include any pre-trained models or datasets.

\item {\bf Licenses for existing assets}
    \item[] Question: Are the creators or original owners of assets (e.g., code, data, models), used in the paper, properly credited and are the license and terms of use explicitly mentioned and properly respected?
    \item[] Answer: \answerYes{}
    \item[] Justification: Yes, all assets have been cited appropriately.

\item {\bf New assets}
    \item[] Question: Are new assets introduced in the paper well documented and is the documentation provided alongside the assets?
    \item[] Answer: \answerNA{}
    \item[] Justification: This research does not release any new assets.

\item {\bf Crowdsourcing and research with human subjects}
    \item[] Question: For crowdsourcing experiments and research with human subjects, does the paper include the full text of instructions given to participants and screenshots, if applicable, as well as details about compensation (if any)? 
    \item[] Answer: \answerNA{}
    \item[] Justification: The paper does not involve crowdsourcing nor research with human subjects.

\item {\bf Institutional review board (IRB) approvals or equivalent for research with human subjects}
    \item[] Question: Does the paper describe potential risks incurred by study participants, whether such risks were disclosed to the subjects, and whether Institutional Review Board (IRB) approvals (or an equivalent approval/review based on the requirements of your country or institution) were obtained?
    \item[] Answer: \answerNA{}
    \item[] Justification: This research does not involve any human subjects.

\item {\bf Declaration of LLM usage}
    \item[] Question: Does the paper describe the usage of LLMs if it is an important, original, or non-standard component of the core methods in this research? Note that if the LLM is used only for writing, editing, or formatting purposes and does \emph{not} impact the core methodology, scientific rigor, or originality of the research, declaration is not required.
    \item[] Answer: \answerNo{}
    \item[] Justification: This research presented does not involve LLMs as any important, original, or non-standard components.

\end{enumerate}

%% file: origin.bib
@article{benner2016block,
  title={Block-diagonal preconditioning for optimal control problems constrained by {PDE}s with uncertain inputs},
  author={Benner, P. and Onwunta, A. and Stoll, M.},
  journal={SIAM Journal on Matrix Analysis and Applications},
  volume={37},
  number={2},
  pages={491--518},
  year={2016},
  publisher={SIAM}
}

@article{olivier2020uqpy,
  title={UQpy: A general purpose Python package and development environment for uncertainty quantification},
  author={Olivier, Audrey and Giovanis, Dimitris G and Aakash, BS and Chauhan, Mohit and Vandanapu, Lohit and Shields, Michael D},
  journal={Journal of Computational Science},
  volume={47},
  pages={101204},
  year={2020},
  publisher={Elsevier}
}

@article{benner2015low,
  title={Low-rank solution of unsteady diffusion equations with stochastic coefficients},
  author={Benner, Peter and Onwunta, Akwum and Stoll, Martin},
  journal={SIAM/ASA Journal on Uncertainty Quantification},
  volume={3},
  number={1},
  pages={622--649},
  year={2015},
  publisher={SIAM}
}

@article{BENNER201626,
title = {Low-rank solvers for unsteady Stokes–Brinkman optimal control problem with random data},
journal = {Computer Methods in Applied Mechanics and Engineering},
volume = {304},
pages = {26-54},
year = {2016},
issn = {0045-7825},
doi = {https://doi.org/10.1016/j.cma.2016.02.004},
url = {https://www.sciencedirect.com/science/article/pii/S0045782516300263},
author = {Peter Benner and Sergey Dolgov and Akwum Onwunta and Martin Stoll}
}

@article{Antil2023,
  title={TTRISK: Tensor Train Decomposition Algorithm for Risk-Averse Optimization},
  author={Antil, Harbir and Dolgov, Sergey and Onwunta, Akwum},
  journal={NumericalLinear Algebra with Applications},
  volume={3},
  number={3},
  pages={e2481},
  year={2023},
  url={https://doi.org/10.1002/nla.2481.12}
 }

@article{Antil2025,
  title={Smoothed Moreau-Yosida Tensor-Train Approximation of State-Constrained Optimization Problems Under Uncertainty},
  author={Antil, Harbir and Dolgov, Sergey and Onwunta, Akwum},
  journal={NumericalLinear Algebra with Applications},
  pages={32:e70028},
  year={2025},
  url={https://doi.org/10.1002/nla.70028}
 }

@article{Antil2025b,
  title={Tensor train solution to uncertain optimization problems with shared sparsity penalty},
  author={Antil, Harbir and Dolgov, Sergey and Onwunta, Akwum},
  journal={Optimization and Engineering},
   year={2025},
  url={https://doi.org/10.1007/s11081-025-10036-7}
 }


%% file: pinn.bib
@inproceedings{DBLP:conf/nips/Lakshminarayanan17,
  author       = {Balaji Lakshminarayanan and
                  Alexander Pritzel and
                  Charles Blundell},
  editor       = {Isabelle Guyon and
                  Ulrike von Luxburg and
                  Samy Bengio and
                  Hanna M. Wallach and
                  Rob Fergus and
                  S. V. N. Vishwanathan and
                  Roman Garnett},
  title        = {Simple and Scalable Predictive Uncertainty Estimation using Deep Ensembles},
  booktitle    = {Advances in Neural Information Processing Systems 30: Annual Conference
                  on Neural Information Processing Systems 2017, December 4-9, 2017,
                  Long Beach, CA, {USA}},
  pages        = {6402--6413},
  year         = {2017},
  url          = {https://proceedings.neurips.cc/paper/2017/hash/9ef2ed4b7fd2c810847ffa5fa85bce38-Abstract.html},
  bibsource    = {dblp computer science bibliography, https://dblp.org}
}

@incollection{goswami2023physics,
  title={Physics-informed deep neural operator networks},
  author={Goswami, Somdatta and Bora, Aniruddha and Yu, Yue and Karniadakis, George Em},
  booktitle={Machine learning in modeling and simulation: methods and applications},
  pages={219--254},
  year={2023},
  publisher={Springer}
}

@inproceedings{cho2024separable,
 author = {Cho, Junwoo and Nam, Seungtae and Yang, Hyunmo and Yun, Seok-Bae and Hong, Youngjoon and Park, Eunbyung},
 booktitle = {Advances in Neural Information Processing Systems},
 pages = {23761--23788},
 title = {Separable Physics-Informed Neural Networks},
 volume = {36},
 year = {2023}
}

@article{eskin2024separable,
  title={Separable physics-informed neural networks for the solution of elasticity problems},
  author={Es'kin, Vasiliy A and Davydov, Danil V and Gur'eva, Julia V and Malkhanov, Alexey O and Smorkalov, Mikhail E},
  journal={arXiv preprint arXiv:2401.13486},
  year={2024}
}

@article{zanotta2025quantum,
  title={Quantum Orthogonal Separable Physics-Informed Neural Networks},
  author={Zanotta, Pietro and Budinski, Ljubomir and Aytekin, Caglar and Lahtinen, Valtteri},
  journal={arXiv preprint arXiv:2511.12613},
  year={2025}
}

@article{raissi2019physics,
  title={Physics-informed neural networks: A deep learning framework for solving forward and inverse problems involving nonlinear partial differential equations},
  author={Raissi, Maziar and Perdikaris, Paris and Karniadakis, George E},
  journal={Journal of Computational Physics},
  volume={378},
  pages={686--707},
  year={2019},
  publisher={Elsevier}
}

@article{karniadakis2021physics,
  title={Physics-informed machine learning},
  author={Karniadakis, George E and Kevrekidis, Ioannis G and Lu, Lu and Perdikaris, Paris and Wang, Sifan and Yang, Liu},
  journal={Nature Reviews Physics},
  volume={3},
  number={6},
  pages={422--440},
  year={2021},
  publisher={Nature Publishing Group}
}

@article{zhang2019quantifying,
  title={Quantifying total uncertainty in physics-informed neural networks for solving forward and inverse stochastic problems},
  author={Zhang, Dongkun and Lu, Lu and Guo, Ling and Karniadakis, George E},
  journal={Journal of Computational Physics},
  volume={397},
  pages={108850},
  year={2019},
  publisher={Elsevier}
}

@article{yang2021bpinns,
  title={{B-PINNs}: {B}ayesian physics-informed neural networks for forward and inverse {PDE} problems with noisy data},
  author={Yang, Liu and Meng, Xuhui and Karniadakis, George E},
  journal={Journal of Computational Physics},
  volume={425},
  pages={109913},
  year={2021},
  publisher={Elsevier}
}

@article{psaros2023uncertainty,
  title={Uncertainty quantification in scientific machine learning: Methods, metrics, and comparisons},
  author={Psaros, Apostolos F and Meng, Xuhui and Zou, Zongren and Guo, Ling and Karniadakis, George E},
  journal={Journal of Computational Physics},
  volume={477},
  pages={111902},
  year={2023},
  publisher={Elsevier}
}

@article{novak2024physics,
  title={Physics-informed polynomial chaos expansions},
  author={Nov{\'a}k, Luk{\'a}{\v{s}} and Sharma, Himanshu and Shields, Michael D},
  journal={Journal of Computational Physics},
  volume={506},
  pages={112926},
  year={2024},
  publisher={Elsevier}
}

@article{shi2026expanding,
  title={Expanding the Chaos: Neural Operator for Stochastic (Partial) Differential Equations},
  author={Shi, Dai and Lin, Lequan and Han, Andi and Thompson, Luke and Hern{\'a}ndez-Lobato, Jos{\'e} Miguel and Wang, Zhiyong and Gao, Junbin},
  journal={arXiv preprint arXiv:2601.01021},
  year={2026}
}

@article{lu2021learning,
  title={Learning nonlinear operators via {DeepONet} based on the universal approximation theorem of operators},
  author={Lu, Lu and Jin, Pengzhan and Pang, Guofei and Zhang, Zhongqiang and Karniadakis, George E},
  journal={Nature Machine Intelligence},
  volume={3},
  number={3},
  pages={218--229},
  year={2021},
  publisher={Nature Publishing Group}
}

@inproceedings{li2021fourier,
  title={Fourier Neural Operator for Parametric Partial Differential Equations},
  author={Li, Zongyi and Kovachki, Nikola Borislavov and Azizzadenesheli, Kamyar and Bhattacharya, Kaushik and Stuart, Andrew and Anandkumar, Anima and others},
  year={2020},
  booktitle={International Conference on Learning Representations}
}

@article{kolda2009tensor,
  title={Tensor Decompositions and Applications},
  author={Kolda, Tamara G and Bader, Brett W},
  journal={SIAM Review},
  volume={51},
  number={3},
  pages={455--500},
  year={2009},
  publisher={SIAM}
}

@article{oseledets2011tensor,
  title={Tensor-train decomposition},
  author={Oseledets, Ivan V},
  journal={SIAM Journal on Scientific Computing},
  volume={33},
  number={5},
  pages={2295--2317},
  year={2011},
  publisher={SIAM}
}

@article{bachmayr2023low,
  title={Low-rank tensor methods for partial differential equations},
  author={Bachmayr, Markus},
  journal={Acta Numerica},
  volume={32},
  pages={1--121},
  year={2023},
  publisher={Cambridge University Press}
}

@article{bachmayr2016tensor,
  title={Tensor networks and hierarchical tensors for the solution of high-dimensional partial differential equations},
  author={Bachmayr, Markus and Schneider, Reinhold and Uschmajew, Andr{\'e}},
  journal={Foundations of Computational Mathematics},
  volume={16},
  pages={1423--1472},
  year={2016},
  publisher={Springer}
}

@inproceedings{vemuri2024functional,
  title={Functional tensor decompositions for physics-informed neural networks},
  author={Vemuri, Sai Karthikeya and B{\"u}chner, Tim and Niebling, Julia and Denzler, Joachim},
  booktitle={International Conference on Pattern Recognition},
  pages={32--46},
  year={2024},
  organization={Springer}
}

@inproceedings{guo2025interpolating,
  title={Interpolating Neural Network-Tensor Decomposition (INN-TD): a scalable and interpretable approach for large-scale physics-based problems},
  author={Guo, Jiachen and Xie, Xiaoyu and Park, Chanwook and Zhang, Hantao and Politis, Matthew J and Domel, Gino and Liu, Wing Kam},
  booktitle={International Conference on Machine Learning},
  pages={21138--21162},
  year={2025},
  organization={PMLR}
}

@article{liu2022ttpinn,
  title={{TT-PINN}: A tensor-compressed neural {PDE} solver for edge computing},
  author={Liu, Ziyue and Yu, Xinling and Zhang, Zheng},
  journal={arXiv preprint arXiv:2207.01751},
  year={2022}
}

@article{xiu2002wiener,
  title={The {W}iener--{A}skey polynomial chaos for stochastic differential equations},
  author={Xiu, Dongbin and Karniadakis, George E},
  journal={SIAM Journal on Scientific Computing},
  volume={24},
  number={2},
  pages={619--644},
  year={2002},
  publisher={SIAM}
}

@article{babuvska2007stochastic,
  title={A stochastic collocation method for elliptic partial differential equations with random input data},
  author={Babu{\v{s}}ka, Ivo and Nobile, Fabio and Tempone, Ra{\'u}l},
  journal={SIAM Journal on Numerical Analysis},
  volume={45},
  number={3},
  pages={1005--1034},
  year={2007},
  publisher={SIAM}
}

@article{nouy2010proper,
  title={Proper generalized decompositions and separated representations for the numerical solution of high dimensional stochastic problems},
  author={Nouy, Anthony},
  journal={Archives of Computational Methods in Engineering},
  volume={17},
  number={4},
  pages={403--434},
  year={2010},
  publisher={Springer}
}

@article{chinesta2010recent,
  title={Recent advances and new challenges in the use of the proper generalized decomposition for solving multidimensional models},
  author={Chinesta, Francisco and Ammar, Amine and Cueto, Elias},
  journal={Archives of Computational Methods in Engineering},
  volume={17},
  number={4},
  pages={327--350},
  year={2010},
  publisher={Springer}
}

@article{doostan2013nonintrusive,
  title={Non-intrusive low-rank separated approximation of high-dimensional stochastic models},
  author={Doostan, Alireza and Validi, AbdolHamid and Iaccarino, Gianluca},
  journal={Computer Methods in Applied Mechanics and Engineering},
  volume={263},
  pages={42--55},
  year={2013},
  publisher={Elsevier}
}

@article{bigoni2016spectral,
  title={Spectral tensor-train decomposition},
  author={Bigoni, Daniele and Engsig-Karup, Allan P and Marzouk, Youssef M},
  journal={SIAM Journal on Scientific Computing},
  volume={38},
  number={4},
  pages={A2405--A2439},
  year={2016},
  publisher={SIAM}
}

@book{ghanem1991stochastic,
  title={Stochastic Finite Elements: A Spectral Approach},
  author={Ghanem, Roger G and Spanos, Pol D},
  year={1991},
  publisher={Springer}
}

@article{si2026complex,
  title={Complex physics-informed neural network},
  author={Si, Chenhao and Yan, Ming and Li, Xin and Xia, Zhihong},
  journal={Journal of Computational Physics},
  volume={553},
  pages={114713},
  year={2026},
  publisher={Elsevier}
}

@article{si2025initialization,
  title={Initialization-enhanced physics-informed neural network with domain decomposition ({IDPINN})},
  author={Si, Chenhao and Yan, Ming},
  journal={Journal of Computational Physics},
  volume={530},
  pages={113914},
  year={2025},
  publisher={Elsevier}
}

@article{bahmani2025neural,
title = {Neural chaos: A spectral stochastic neural operator},
journal = {Journal of Computational Physics},
volume = {539},
pages = {114233},
year = {2025},
author = {Bahador Bahmani and Ioannis G. Kevrekidis and Michael D. Shields},
}


%% file: ref.bib
@inproceedings{fan2026physics,
  title={Physics-Informed Inference Time Scaling for Solving High-Dimensional Partial Differential Equations},
  author={Fan, Zexi and Sun, Yan and Yang, Shihao and Lu, Yiping},
  booktitle={International Conference on Learning Representations},
  volume={2026},
  pages={127794--127844},
  year={2026}
}

@misc{yang2026derivativeinformedtrainingneuraloperators,
      title={Derivative-Informed Training of Neural Operators On-the-Fly via Sketched Tangent Consistency}, 
      author={Xinhan Yang and Lu Lu and Shancong Mou},
      year={2026},
      eprint={2609.32797},
      archivePrefix={arXiv},
      primaryClass={cs.LG},
      url={https://arxiv.org/abs/2609.32797}, 
}
